\documentclass[letterpaper,times]{IONconf-v2}
\usepackage[]{graphicx}    % We use this package in this document
\newcommand{\ignore}[1]{}  % {} empty inside = %% comment

\usepackage{amsmath,amssymb,amsfonts}
\usepackage{bm}
\usepackage{algorithm2e}
\usepackage{float}
\usepackage[labelformat=simple]{subcaption}

\usepackage{mathtools}
\usepackage{siunitx}
\usepackage{multirow}
\usepackage{wasysym}
\usepackage{csquotes}

\usepackage[hidelinks]{hyperref}

\newcommand{\R}{\mathbb{R}}

\newcommand{\matrank}{\mathrm{rank}}
\newcommand{\diag}{\mathrm{diag}}
\newcommand{\onesvec}{\mathbf{1}}
\newcommand{\zeromat}{\mathbf{0}}
\newcommand{\eye}{\mathbf{I}}
\newcommand{\hadamard}{\circ}

\newcommand{\satrange}{r}
\newcommand{\singlepoint}{\mathbf{x}}
\newcommand{\fault}{f}
\newcommand{\faultmag}{\bar{f}}
\newcommand{\noise}{w}
\newcommand{\sqrtterm}{s}

\newcommand{\rrefleftover}{\Delta}

\newcommand{\PointMat}{\mathbf{X}}

\newcommand{\EDM}{\mathbf{D}}
\newcommand{\EDMnoisy}{\Tilde{\mathbf{D}}}

\newcommand{\GCEDM}{\mathbf{G}}
\newcommand{\GCEDMnoisy}{\Tilde{\mathbf{G}}}

\newcommand{\centeringJ}{\mathbf{J}}

\newcommand{\FaultMat}{\mathbf{F}}

\newcommand{\SqrtMat}{\mathbf{S}}

\newcommand{\FaultMatSuper}{\FaultMat^{n, m}}
\newcommand{\SqrtMatSuper}{\SqrtMat^{n, d}}

\newcommand{\meanPoints}{\mu}

\newcommand{\teststatistic}{\gamma_{\text{test}}}
\newcommand{\teststatisticthreshold}{\bar{\gamma}_{\text{test}}}

\usepackage{amsthm}
\usepackage[english]{babel}
\newtheorem{proposition}{Proposition}[section]
\newtheorem{corollary}{Corollary}[proposition]

\newcolumntype{I}{!{\vrule width 1.5pt}}

\newcommand{\newtext}[1]{\textcolor{black}{#1}}
\newcommand{\newcaption}[1]{\caption{\textcolor{black}{#1}}}

\newcommand{\newnewtext}[1]{\textcolor{black}{#1}}
\newcommand{\newnewcaption}[1]{\caption{\textcolor{black}{#1}}}

\newenvironment{editsection}[1]
  {\begingroup\color{black}}
  {\endgroup}

\newenvironment{newnewsection}[1]
  {\begingroup\color{black}}
  {\endgroup}

\articletype{Regular papers}%

\received{}
\revised{}
\accepted{}
\doi{10.1109/xxx.xx.xxxx}
\journalname{NAVIGATION}
\journalvolume{-}
\journalnumber{-}

\title{\newtext{Satellite Autonomous Clock Fault Monitoring with Inter-Satellite Ranges Using Euclidean Distance Matrices}}

\author[1]{Keidai Iiyama}

\author[1]{Daniel Neamati}

\author[1]{Grace Gao}

\authormark{IIYAMA \textsc{et al}}

\address[1]{Department of Aeronautics and Astronautics, Stanford University, California, United States}

\corres{Grace Gao, Department of Aeronautics and Astronautics, Stanford University, California, United States \\ \email{gracegao@stanford.edu}}

\begin{document}

% The Abstract. The * indicates a section excluded from numbering.
\abstract[Abstract]{
\newnewtext{
To support reliable positioning, navigation, and timing (PNT) services for lunar satellite constellations, this paper proposes a framework for detecting satellite clock phase jumps using dual one-way inter-satellite range (ISR) measurements. 
The method models the constellation as a weighted graph and exploits rigidity properties to detect faults. 
Specifically, satellite clock faults are identified by monitoring the realizability of vertex-redundantly rigid subgraphs through the singular values of geometric-centered Euclidean distance matrices (GCEDMs).
We establish the graph-topological conditions required for fault detection and derive the statistical distribution of the GCEDM test statistic under both nominal and faulty conditions.
To address sparse link connectivity, we also introduce an ephemeris-augmented formulation that combines measured and estimated ranges and performs fault detection on fault-detectable subgraphs.
Simulation results using a terrestrial GPS constellation and a notional lunar constellation demonstrate that the proposed approach provides robust fault detection under both dense and sparse measurement conditions.
}

% To address the need for robust positioning, navigation, and timing services in lunar environments, this paper proposes a \newtext{novel onboard clock phase jump detection framework} for satellite constellations using \newtext{range measurements obtained from dual one-way inter-satellite links}. 
% Our approach leverages vertex redundantly rigid graphs to detect faults without relying on prior knowledge of satellite positions or clock biases,
% \newtext{providing flexibility for lunar satellite networks with diverse satellite types and operators.}
% We model satellite constellations as graphs\newtext{,} where satellites are vertices and inter-satellite links are edges. 
% The proposed algorithm \newtext{detects and identifies satellites with clock jumps by monitoring} the singular values of the geometric-centered Euclidean distance matrix (GCEDM) of \newtext{5-clique sub-graphs}. 
% The proposed method is validated through simulations of \newtext{a GPS constellation and a notional} constellation around the Moon, 
% demonstrating its effectiveness in various configurations. 
%This research contributes to the reliable operation of satellite constellations for future lunar exploration missions.
}


\maketitle

\section{Introduction}

% Lunar
To meet the growing need for robust positioning, navigation, and timing (PNT) services at the lunar surface and lunar orbits, 
NASA and its international partners are collaborating to develop LunaNet~\citep{Israel2020}, 
a network of networks providing data relay, PNT, detection, and science services.
In LunaNet, these services are provided through cooperation among interoperable systems that evolve over time to meet the growing needs for these services, efficiently establishing a reliable, sustainable, and scalable network. 

% Integrity is important
It is crucial to monitor LunaNet navigation satellites for the reliable operation of safety-critical missions, such as lunar landing and human missions on the lunar surface.
The quality of lunar navigation signals can be compromised by various system faults, such as clock runoffs (e.g.\newtext{,} phase and frequency jumps~\citep{weiss2010board}), unflagged maneuvers, 
failures in satellite payload signal generation components, and code-carrier incoherence~\citep{Walter2018}.
While the LunaNet Relay Service Documents (SRD) state that robustness of the \newtext{navigation} signal should be a key consideration for LunaNet~\citep{nasa2022srd}, 
\newnewtext{the} specific methodology to monitor faults on LunaNet satellites has not yet been solidified.

% GNSS for terrestrial
\begin{editsection}{blue}
The currently implemented integrity monitoring method for terrestrial GNSS can be categorized into three methods: (1) ground-based monitoring, (2) receiver autonomous integrity monitoring (RAIM), and (3) satellite-based integrity monitoring (SAIM).
The first method relies on a network of ground-based monitoring stations at known positions to track and collect the GNSS signals, which are then collected and processed at a central computing center to compute corrections and integrity information~\citep{MISRA_1993}.
The system fault alerts are provided as integrity information from GNSS augmentation systems, such as the satellite-based augmentation system (SBAS)~\citep{van2009gps}.
The second method, RAIM, uses redundant measurements from GNSS satellites to detect and exclude faulty satellites on the receiver side~\citep{MISRA_1993}.
The third method, SAIM, computes the integrity information onboard the navigation satellite to detect satellite-related faults such as clock anomalies and ephemeris errors~\citep{xu2011gnss}.

% Drawbacks of the current methods
While the first two methods are well established and widely adopted for integrity monitoring of terrestrial GNSS, their adoption may be more limited for lunar constellations.
The first method requires a network of ground-based monitoring stations, which would be expensive and impractical to deploy widely on the lunar surface, especially for near-term constellations.
The second method, RAIM, requires at least five satellites in view to detect (six satellites in view to exclude) a faulty satellite, which limits its availability in lunar constellations with a smaller number of navigation satellites than terrestrial GNSS.

% SAIM
The third method, SAIM, has several advantages compared with the other two methods:
(1) faster time to alert enabled by eliminating the latency associated with data collection, processing, and uplink on the ground
(2) the use of precise inter-satellite range (ISR) measurements (BeiDou satellites have demonstrated 10 cm accuracy~\citep{Zhang2024}) 
that are unaffected by the ionosphere and troposphere, 
and 
(3) the monitoring can be operated autonomously within the constellation and independently of ground monitoring.
The autonomous nature of SAIM is particularly valuable for early lunar deployments, when lunar surface stations will be scarce or absent.

% SAIM algorithms
SAIM has been implemented in the BeiDou GNSS system to detect clock phase and frequency jumps
by comparing the outputs of multiple onboard clocks~\citep{Cao2019, Chen2022},
at the cost of adding extra hardware.
Several other works have also proposed SAIM methods utilizing ISR measurements, which detect ephemeris anomalies by directly comparing the ISR measurement and the predicted range from the ephemeris
~\citep{wolf_onboard_2000, rodriguez-perez_inter-satellite_2011, xu2011gnss}.
However, \newnewtext{these methods} processes ISR measurements on a per-satellite basis, 
making their minimal detectable bias (MDB) dependent on ephemeris accuracy~\citep{xu2011gnss}.  
This dependence may prove challenging for lunar constellations, 
where orbit determination and time synchronization are more difficult than for terrestrial GNSS.  
\newnewtext{Another potential SAIM algorithm that can solve this problem} is to aggregate all ISR measurements centrally, form a linear system of residual equations, and apply data-snooping techniques~\citep{baarda1968testing, Imparato2018} to isolate faulty satellites. 
% Yet this approach still requires the satellite positions to compute the residuals, which may not be (easily) available onboard all LunaNet nodes, since the interoperable LunaNet network may include links to satellites dedicated to communication or science tasks rather than PNT~\citep{Giordano2023}.
\end{editsection}

% Map fault and ISR Bias
% bias we can 
In this paper, we propose a new algorithm to detect phase jumps in the satellite clocks
that uses \newtext{dual one-way ISR} measurements.
The proposed method detects faults by identifying biases injected into ISR measurements whenever a satellite's onboard clock experiences a phase jump
between two dual one-way links.
We prove that we are able to detect fault satellites (satellites with clocks that experienced phase jumps) by assessing the realizability of the 5-clique (a fully connected graph of 5 nodes) subgraphs, since these subgraphs remain rigid (graphs that cannot be susceptible to continuous flexing are called rigid graphs), after removing any vertex from the graph~\citep{alireza_motevallian_robustness_2015}.
The concept of the algorithm is illustrated in Figure~\ref{fig:concept}.
\begin{figure}[htb!]
    \centering
    \includegraphics[width=0.4\linewidth]{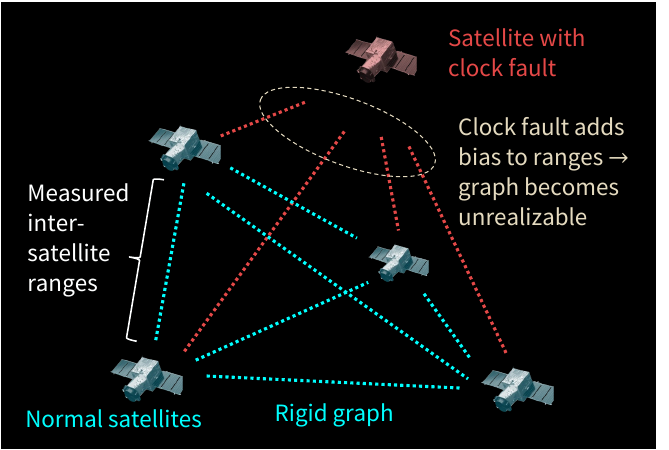}
    \newcaption{Concept of the proposed fault detection algorithm}
    \label{fig:concept}
\end{figure}

Our proposed method monitors the singular values of the geometric-centered \newtext{Euclidean distance matrix} (GCEDM)~\citep{dokmanic2015euclidean} constructed from the ISR measurements, 
to check if 5-clique subgraphs are realizable in 3-dimensional space. 
\newtext{In particular, we provide analytical expressions of the distribution of the squared 4th singular value under the presence of noise and biases in the range measurement, and design hypothesis tests to detect and identify the satellite clock with phase jumps.}
 
The contribution of this paper is summarized below. This work is based on our prior conference paper presented at the 2024 ION GNSS+ Conference~\citep{Iiyama2024Rigid}, 
\newtext{but includes key theoretical and simulation updates beyond our prior publication.}
\begin{editsection}{blue}
\begin{itemize}
    \item We establish the necessary graph topology, specifically, vertex redundant rigidity, for unambiguous fault detection from ISR measurements.
    \item We prove key rank properties of EDMs and GCEDMs, and derive the distribution of the squared 4th singular value under noise and bias.
    \item We develop a hypothesis test based on the GCEDM's 4th singular value to detect and identify satellite clock phase jumps.
    \item We validate our approach via simulations of a GPS constellation and a notional lunar constellation, analyzing the impact of hyperparameters and fault magnitudes, and comparing fault detection and false alarm rates against residual-based baseline methods.
\end{itemize}
\end{editsection}

\newtext{%Note that our
\newnewtext{Our} method is not intended to 
% completely 
replace decades of research in the existing fault detection methods. Rather, our method provides new theoretical and practical insights to view the fault detection problem from a different perspective.}

The paper is arranged as follows. 
\newtext{Section~\ref{sec:isl} introduces the formulation of the dual one-way ISR measurements.}
\newtext{Section~\ref{sec:residual_methods} describes two baseline clock jump detection methods, which use ISR measurement residuals computed from the ephemeris.}
\newtext{Section~\ref{sec:rigidity_fd} describes the proposed fault detection algorithm based on singular values of the GCEDM.}
\newtext{Section~\ref{sec:simulation} provides simulation results for the GPS constellation and a notional lunar constellation.}
The paper concludes with section~\ref{sec:conclusion}.

\begin{editsection}{blue}

\section{Measurements From Inter-satellite Links}
\label{sec:isl}
\subsection{Dual One-Way Inter-satellite Range Measurement}
We consider two satellites establishing a dual one-way inter-satellite link (ISL) as shown in Figure \ref{fig:dual_oneway} to obtain ISR measurements.
In this section, we formulate the ISR measurement model for the dual one-way link, based on \citet{Zhang2024}.

Let $i, j$ be the indices of the two satellites establishing the link, and $\singlepoint_i, \singlepoint_j \in \R^3$ the position vector of the two satellites.
The two observation equations for the dual one-way link can be expressed as follows:
\begin{align}
    \rho_{ij}(t_1) &= \|\singlepoint_j(t_1) - \singlepoint_i(t_1 - \Delta t_1) \| + c \cdot \left[ \tau_j(t_1) - \tau_i(t_1 - \Delta t_1) \right] 
    + c \cdot \left[\Delta t_j^{rx} + \Delta t_i^{tx} \right] + \omega_1 \\
    \rho_{ji}(t_2) &= \|\singlepoint_i(t_2) - \singlepoint_j(t_2 - \Delta t_2) \| + c \cdot \left[ \tau_i(t_2) - \tau_j(t_2 - \Delta t_2) \right] 
    + c \cdot \left[\Delta t_i^{rx} + \Delta t_j^{tx} \right] + \omega_2
\end{align}
where $\rho_{ij}, \rho_{ji}$ represents the measured pseudoranges for the two links where satellite $i$ and $j$ receives the signal at $t_1$ and $t_2$, respectively. 
$\Delta t_1$ and $\Delta t_2$ are the signal travel times; $\tau_i$ and $\tau_j$ are the clock offsets; $t^{rx}$ and $t^{tx}$ represent hardware delays of the receiving channel and the transmitting channels; $\omega_1$ and $\omega_2$ are the measurement noises.
We assume the phase center offset and relativistic effects are corrected, and not included in the equations.

To eliminate the common errors between the two links, the two measurements are transformed into the common time $t_0$ using the following equations:
\begin{align}
    \rho_{ij}(t_0) &= \rho_{ij}(t_1) + \Delta \rho_{ij} = \|\singlepoint_j(t_0) - \singlepoint_i(t_0) \| + c \cdot \left[ \tau_j(t_0) - \tau_i(t_0) \right] 
    + c \cdot \left[\Delta t_j^{rx} + \Delta t_i^{tx} \right] + \omega_{1, 0} \\
    \rho_{ji}(t_0) &= \rho_{ji}(t_2) + \Delta \rho_{ji} = \|\singlepoint_i(t_0) - \singlepoint_j(t_0) \| + c \cdot \left[ \tau_i(t_0) - \tau_j(t_0) \right] 
    + c \cdot \left[\Delta t_i^{rx} + \Delta t_j^{tx} \right] + \omega_{2, 0}
\end{align}
where $\Delta \rho_{ij}$ and $\Delta \rho_{ji}$ are the corrections of satellite position and clocks that can be expressed as:
\begin{align}
    \Delta \rho_{ij} &= \| \singlepoint_j(t_0) - \singlepoint_i(t_0)  \| - \|\singlepoint_j(t_1) - \singlepoint_j(t_1 - \Delta t_1)\| + c \cdot \left[ \tau_j(t_0) - \tau_i(t_0) - \tau_j(t_1) + \tau_i(t_1 - \Delta t_1) \right] \label{eq:corrected_rhoij} \\
    \Delta \rho_{ji} &= \| \singlepoint_i(t_0) - \singlepoint_j(t_0)  \| - \|\singlepoint_i(t_2) - \singlepoint_j(t_2 - \Delta t_2)\| + c \cdot \left[ \tau_i(t_0) - \tau_j(t_0) - \tau_i(t_1) + \tau_j(t_2 - \Delta t_2) \right]
    \label{eq:corrected_rhoji}
\end{align}
The corrections $\Delta \rho_{ij}$ and $\Delta \rho_{ji}$ are obtained using the predicted satellite ephemeris. 
Therefore, the accuracy of the corrections depends on the accuracy of the satellite velocity and the satellite clock drift of the ephemeris.
For Beidou satellites, the accuracy of the predicted satellite velocity and clock drift is better than 0.1 mm/s and 1 $\times 10^{-13}$ s/s, 
resulting in $\Delta \rho_{ij}, \Delta \rho_{ij} < 1.0$ cm when $t_2 - t_1 < 3$ s~\citep{Xie2022}.
For LunaNet satellites, the specification for the sum of velocity and clock drift error (converted to velocity by multiplying with lightspeed) in the ephemeris is 1.2 mm/s (for 3 $\sigma$), so the error of the correction terms can be in order of centimeters~\citep{nasa2022srd}.
Whenever further accuracy is required or velocity ephemeris information is not available, we can also utilize the inter-satellite range rate measurements between the satellites for correction~\citep{Iiyama2024contact}.

The sum and differences of equations \eqref{eq:corrected_rhoij} and \eqref{eq:corrected_rhoji} can be expressed as 
\begin{align}
    r_{ij}(t_0) &= \frac{\rho_{ij}(t_0) + \rho_{ji}(t_0)}{2} 
    = \| \singlepoint_j(t_0) - \singlepoint_i(t_0) \| + \bar{b}_{ij}^{range} + \frac{\omega_{1, 0} + \omega_{2, 0}}{2} \\
    \tau_{ij}(t_0) &= \frac{\rho_{ij}(t_0) - \rho_{ji}(t_0)}{2} 
    = c \cdot \left[ \tau_j(t_0) - \tau_i(t_0) \right] + \bar{b}_{ij}^{clock} + \frac{\omega_{1, 0} - \omega_{2, 0}}{2}
\end{align}
where $r_{ij}(t_0)$ and $\tau_{ij}(t_0)$ are the clock-free and geometry-free observations. $b_{ij}$ and $b_{ji}$ are the bias terms due to hardware delays as follows
\begin{align}
    \bar{b}_{ij}^{range} = c \cdot \frac{\Delta t_j^{rx} + \Delta t_i^{tx} + \Delta t_i^{rx} + \Delta t_j^{tx}}{2} \\
    \bar{b}_{ij}^{clock} = c \cdot \frac{\Delta t_j^{rx} + \Delta t_i^{tx} - \Delta t_i^{rx} - \Delta t_j^{tx}}{2}
\end{align}

The bias terms $b_{ij}^{range}$ and $b_{ij}^{clock}$ are usually constant over a short term and can therefore be estimated using filtering techniques. 
In this paper, we assume that these biases can be estimated within certain errors, and the error terms are included in the measurement noises, resulting in the following two equations.
\begin{align}
    r_{ij}(t_0) &=  \| \singlepoint_i - \singlepoint_j\| + \omega_{range}  \quad &\text{clock-free combination} \\
    \tau_{ij}(t_0) &= c \left[ \tau_j(t_0) - \tau_i(t_0) \right]  + \omega_{clock}  \quad &\text{geometry-free combination} 
\end{align}
For ISLs in Beidou, the observed $\omega_{range}$ and $\omega_{clock}$ can be approximated as a Gaussian distribution with standard deviations in the order of centimeters.

\begin{figure}[ht!]
\centering
\begin{subfigure}[b]{0.48\textwidth}
    \centering
    \includegraphics[height=65mm]{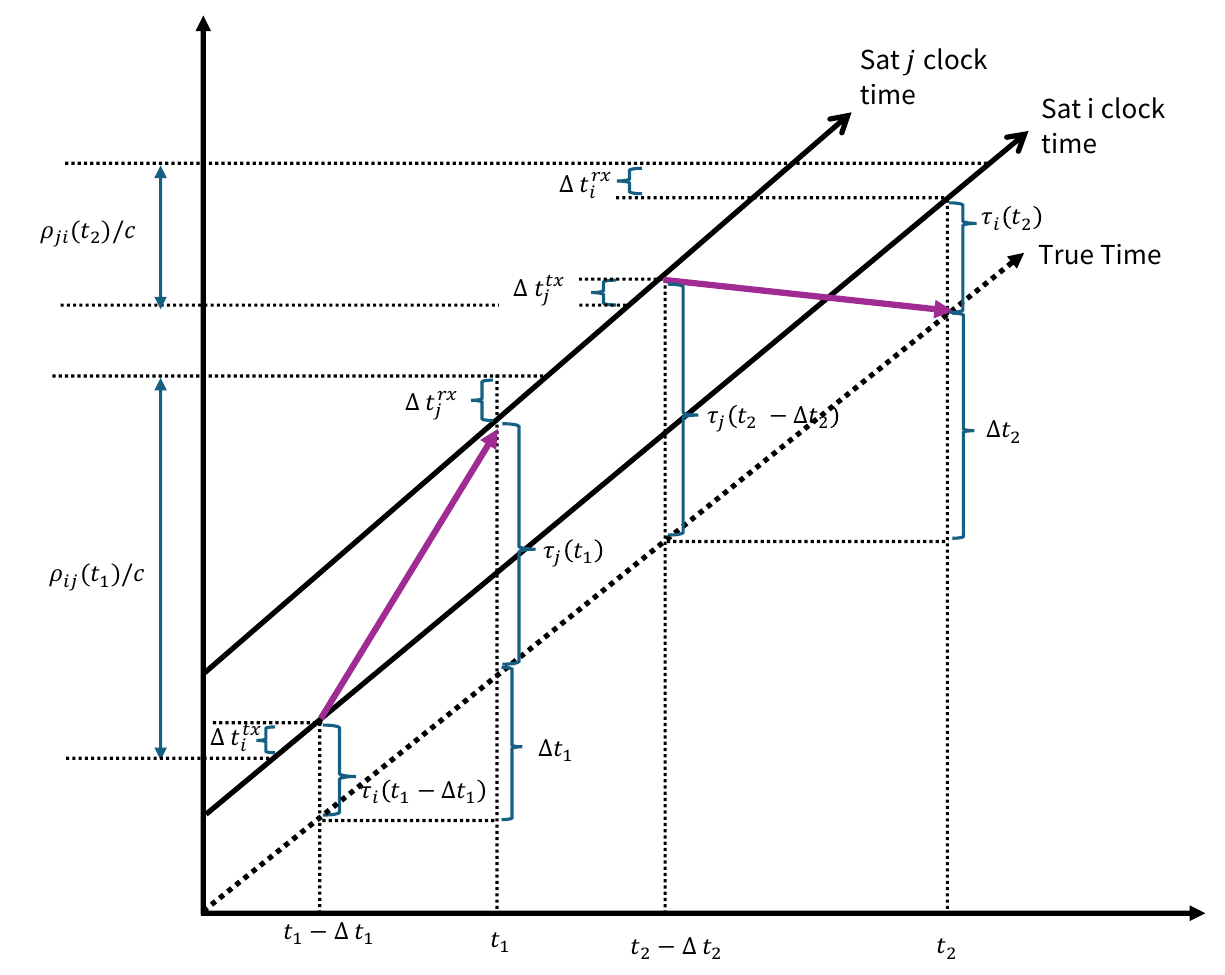}
    \subcaption{
    % \newcaption{
    Dual one-way link between the two satellites, when there is no phase jump at either satellite.
    Two pseudorange measurements $\rho_{ij}(t_1)$ and $\rho_{ji}(t_2)$ are obtained.
    }
    \label{fig:dual_oneway_normal}
\end{subfigure}
\hfill
\begin{subfigure}[b]{0.48\textwidth}
    \centering
    \includegraphics[height=65mm]{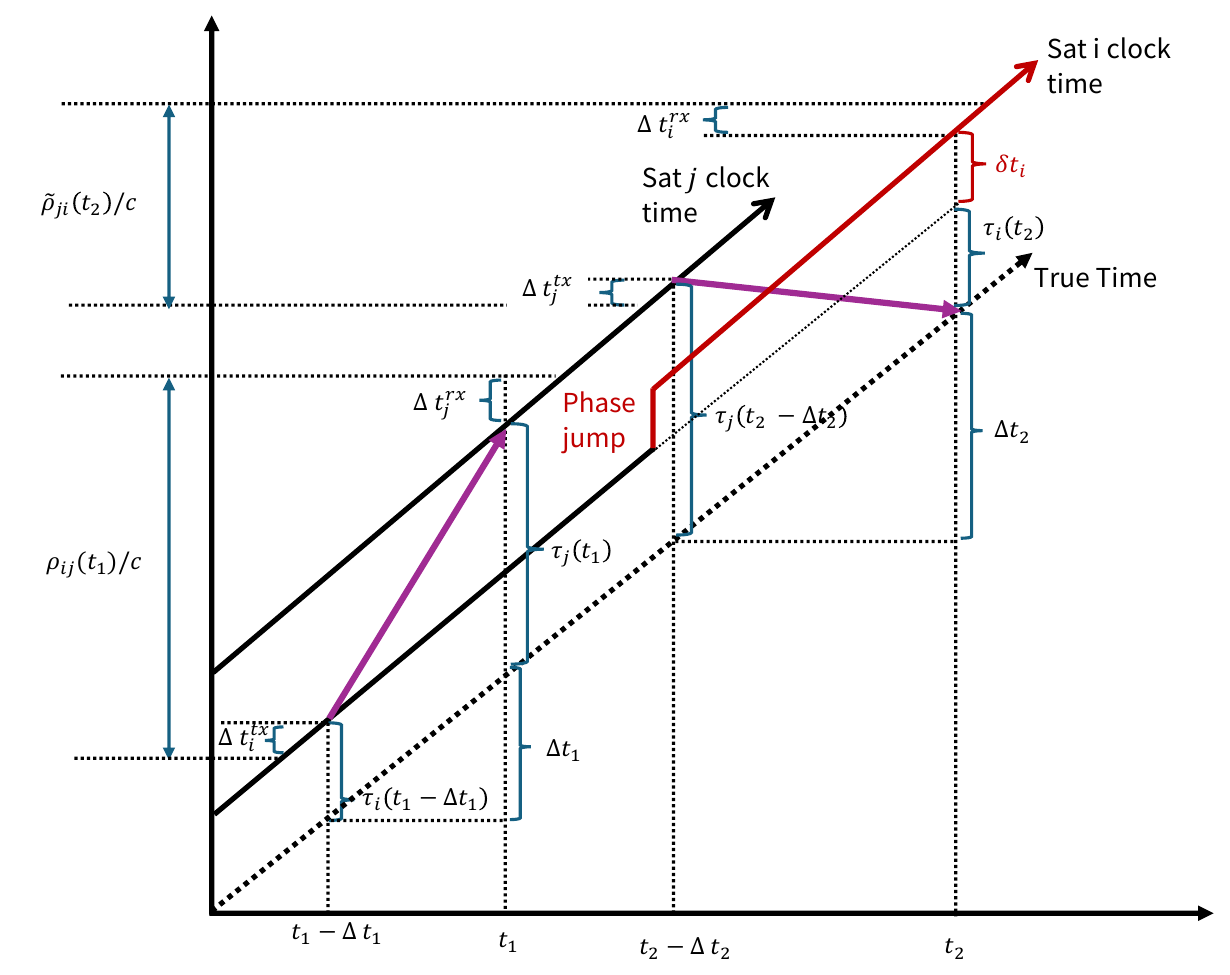}
    \subcaption{Dual one-way link between the two satellites when there is a phase jump at satellite $i$. 
    When the phase jump occurs at $t_1 - \Delta t_1 < t < t_2$, a bias $c \delta t_i$ is added to the 
    pseudorange measurement $\tilde{\rho}_{ji}(t_2)$.}
    \label{fig:dual_oneway_fault}
\end{subfigure}
\newcaption{Dual one-way inter-satellite link measurement between two satellites.
Satellite $i$ transmits the signal to satellite $j$, received at $t_1$, and satellite $j$ transmits the signal to satellite $i$, received at $t_2$.}
\label{fig:dual_oneway}
\end{figure}

\subsection{Satellite Clock Phase Jumps}
Atomic clocks on navigation satellites occasionally experience a sudden jump in the clock phase.
% , which is one of the common causes for GN
When satellite $i$'s clock experiences a phase jump of $\delta t_i$ (delay) between $t_1 - \Delta t_1$ and $t_2$,
as illustrated in Figure \ref{fig:dual_oneway_fault},
the corrupted pseudorange measurement $\tilde{\rho}_{ji}(t_2)$ can be re-written as follows
\begin{equation}
\begin{aligned}
    \tilde{\rho}_{ji}(t_2) 
    &= \|\singlepoint_i(t_2) - \singlepoint_j(t_2 - \Delta t_2) \| 
    + c \cdot \left[ \tau_i(t_2) - \tau_j(t_2 - \Delta t_2) \right] 
    + c \cdot \left[\Delta t_i^{rx} + \Delta t_j^{tx} \right] + c \cdot \delta t_i + \omega_2 \\
    &= \rho_{ji}(t_2) + c \cdot \delta t_i
\end{aligned}
\end{equation}
% The (estimated) 
\newnewtext{The calculated} correction term will be
\begin{equation}
\begin{aligned}
    \Delta \tilde{\rho}_{ji} &= \| \singlepoint_i(t_0) - \singlepoint_j(t_0)  \| - \|\singlepoint_i(t_2 + \delta t_i) - \singlepoint_j(t_2 + \delta t_i - \Delta t_2)\| + c \cdot \left[ \tau_i(t_0) - \tau_j(t_0) - \tau_i(t_1) + \tau_j(t_2 + \delta t_i - \Delta t_2) \right] \\
    &= \Delta \rho_{ij} + \delta \Delta \rho_{ji}
\end{aligned}
\end{equation}
where $\delta \Delta \rho_{ji}$ is the error in the correction term due to the error in the time stamping of the reception time $t_2$ to $t_2 + \delta t_i$, which can be represented as follows
\begin{equation}
\begin{aligned}
    \delta \Delta \rho_{ji} 
    &= \left( \|\singlepoint_i(t_2) - \singlepoint_j(t_2 - \Delta t_2)\| - \|\singlepoint_i(t_2 + \delta t_i) - \singlepoint_j(t_2 + \delta t_i - \Delta t_2)\| \right) + c \cdot \left[\tau_j(t_2 + \delta t_i - \Delta t_2) - \tau_j(t_2 - \Delta t_2) \right] \\
    &\approx \left[ -\dot{\rho_{ji}}(t_2) + c \cdot \dot{\tau_j}(t_2 - \Delta t_2) \right] \delta t_i
    \label{eq:corrected_rhoji_fault}
\end{aligned}
\end{equation}
where $\dot{\rho_{ji}}$ represents the range-rate of the link between $j$ and $i$, and $\dot{\tau_j}$ represents the clock drift of the onboard clock at satellite $j$.
Therefore, the corrected pseudorange at $t_0$ is 
\begin{equation}
    \begin{aligned}
        \tilde{\rho}_{ji}(t_0) = \tilde{\rho}_{ji}(t_2) + \Delta \tilde{\rho}_{ji}
        &= (\rho_{ji}(t_2) + c \delta t_i) + (\Delta \rho_{ij} + \delta \Delta \rho_{ji})\\
        &\approx (\rho_{ij}(t_2) + \Delta \rho_{ij}) + \delta t_i \cdot \left[c - \dot{\rho_{ji}}(t_2) + c \dot{\tau_j}(t_2 - \Delta t_2) \right] \\
        &\approx \rho_{ij}(t_0) + c \cdot \delta t_i \quad (\because \dot{\rho}_{ji} \ll c, \dot{\tau_j} \ll 1)
    \end{aligned}
\end{equation}

Similary, when satellite $j$'s clock experiences a phase jump of $\delta t_j$ (delay) between $t_1$ and $t_2 - \Delta t_2$, the $\rho_{ji}$ can be re-written as 
\begin{equation}
    \tilde{\rho}_{ji}(t_0) \approx \rho_{ij}(t_0) - c \cdot \delta t_j
\end{equation}

Therefore, the clock-free and geometry combination considering satellite clock jumps is:
\begin{align}
    r_{ij}(t_0) &=  \| \singlepoint_i(t_0) - \singlepoint_j(t_0)\| + f_i - f_j + \omega_{range}  \quad &\text{clock-free combination} \label{eq:range_t0} \\
    \tau_{ij}(t_0) &= c \left[ \tau_j(t_0) - \tau_i(t_0) \right] + f_i + f_j + \omega_{clock}  \quad &\text{geometry-free combination} \label{eq:tdiff_t0}\\
    f_i &= \begin{cases}
        \frac{c}{2} \cdot \delta t_i  & \text{clock jump $\delta t_i$ occurs at satellite $i$ for $t$ within $t_1 - \Delta t < t < t_2$}  \\
        0  & \text{satellite $i$'s clock is normal}
    \end{cases} 
    \label{eq:fault_i}
    \\
    f_j &= \begin{cases}
        \frac{c}{2} \cdot \delta t_j  & \text{clock jump $\delta t_j$ occurs at satellite $j$ for $t$ within $t_1 < t < t_2 - \Delta t$}  \\
        0  & \text{satellite $j$'s clock is normal}
    \end{cases}
    \label{eq:fault_j}
\end{align}

Note that if the phase jump $\delta t_i$ occurs at $t < t_1 - \Delta t_1$ or $t > t_2$, and $\delta t_j$ occurs at $t < t_1$ or $t > t_2 - \Delta t_2$, 
the errors cancel when $\tilde{\rho}_{ij}(t_0)$ and $\tilde{\rho}_{ji}(t_0)$ are added or subtracted. 
Consequently, these errors do not appear in $r_{ij}(t_0)$ or $\tau_{ij}(t_0)$. Therefore, to improve the likelihood of detecting clock jumps, we should either increase the frequency of the ISL measurements or extend the interval between $t_1$ and $t_2$.
However, note that increasing the interval $t_2 - t_1$ will lead to larger correction errors ($\Delta \rho_{ij}$ and $\Delta \rho_{ji}$), 
which in turn results in greater $\omega_{range}$ and $\omega_{clock}$.

% \begin{equation}
%     \satrange_{ij} = \begin{cases}
%         \| \singlepoint_i - \singlepoint_j\| + \fault_{ij} = \frac{1}{2} c \tau_{ij} &  (i \neq j)  \\
%         0   & (i = j)  
%     \end{cases}
% \end{equation}
% where  $\fault_{ij}$ is the bias of the range measurements, $c$ is the light speed, and $\tau_{ij}$ is the (measured) two-way travel time.
% The bias term is determined as follows
% \begin{align}
%     \fault_{ij} &= \fault_i + \fault_j,  \\
%     \fault_k &= \begin{cases}
%     \faultmag  & \text{satellite k is fault}  \\
%     0  & \text{satellite k is normal}
%     \end{cases}, \quad k \in \{i, j\}, \quad \faultmag \ \text{is a continuous random variable}
% \end{align}
% which means we assume that the bias is zero when no satellites are in fault status. 

\end{editsection}  % edit color =red
\begin{editsection}{blue}

\section{Baseline Methods: Clock Fault Detection Using Inter-satellite Range Residuals}
\label{sec:residual_methods}
In this section, we introduce two baseline methods for detecting clock jump faults \newnewtext{based on existing literature}.
The two methods \newnewtext{use} the range residuals, 
which are computed from the ISL range measurements and the ephemeris information.

%%%%%%%%%%%%%%%%%%%%%%%%%%%%%%%%%%%%%%%%%%%%
% Sum of Residuals
%%%%%%%%%%%%%%%%%%%%%%%%%%%%%%%%%%%%%%%%%%%%%%
\subsection{Fault Detection Using \newnewtext{Expected Range from Ephemeris}}
\label{sec:sum_of_residuals}
The first method is based on the ephemeris fault detection method proposed 
by \citet{rodriguez-perez_inter-satellite_2011} and \citet{xu2011gnss}.
While the original method was proposed for ephemeris fault detection, it could also
be used for clock phase jump detection.

Let $g_{ij}$ be the range residual of the $i$th ISL at time $t_0$, which is defined as
\begin{equation}
\begin{aligned}
    g_{ij} &= r_{ij}(t_0) - \| \hat{\mathbf{x}}_i(t_0) - \hat{\mathbf{x}}_j(t_0) \| \\
           &= \| \mathbf{x}_i(t_0) - \mathbf{x}_j(t_0) \| - \| \hat{\mathbf{x}}_i(t_0) - \hat{\mathbf{x}}_j(t_0) \| + f_{i} - f_{j} + \omega_{range}
\end{aligned}
\end{equation}
where $\hat{\mathbf{x}}_i(t_0)$ is the estimated position (from ephemeris) of the $i$th satellite at time $t_0$,

The test statistics for the $i$th satellite are defined as
\begin{equation}
    T_i = \sum_{j=1}^{l_i} \left(\frac{g_{ij}^2}{\sqrt{2 \sigma_r^2 + \sigma_m^2}} \right)^2
\end{equation}
where $l_i$ is the number of ISLs that involve the $i$th satellite, and $\sigma_r^2$ and $\sigma_m^2$ are the variances of the position estimates in the ephemeris
and the variance of ISL range measurement noise ($\omega_{range}$), respectively.

We define the null hypothesis $\mathcal{H}_{0, i}$ as the absence of faults in the $i$th satellite, 
and the alternative hypothesis $\mathcal{H}_{a, i}$ as the presence of a fault in the $i$th satellite.
The test statistic $T_i$ is distributed as a Chi-squared distribution with $l_i$ degrees of freedom under the null hypothesis $\mathcal{H}_0$.
The critical value for the test $\chi_{\alpha}^2(l_i, 0)$ for the false alarm rate $\alpha$ is given as
\begin{equation}
    \chi_{\alpha}^2(\newnewtext{l_i}, 0) = P\left[ \chi^2(\newnewtext{l_i}, 0) \leq \chi_{\alpha}^2(\newnewtext{l_i}, 0) \right] = \alpha
\end{equation}
where $\chi^2(r, \lambda)$ is the non-central Chi-squared distribution with $r$ degrees of freedom and the non-centrality parameter $\lambda$.
%
% The testing procedure is~\citep{xu2011gnss}
\newnewtext{Based on \citet{xu2011gnss}, the testing procedure is}

\begin{equation}
    \text{Accept} \ \mathcal{H}_0 \ \text{if} \ \text{max}_{j \in \{1, \ldots, n\}} \ T_i < \sqrt{\chi_{\alpha}^2(\newnewtext{l_i}, 0)}
\end{equation}
otherwise
\begin{equation}
    \text{Accept} \ \mathcal{H}_k \ \text{where} \ k = \text{argmax}_{j \in \{1, \ldots, n\}} \ \frac{T_j}{\sqrt{l_j}}
\end{equation}

The minimum detectable bias\newnewtext{, $MDB(k)$,} for $\mathcal{H}_k$ is computed as
\begin{equation}
    MDB(k) = \sqrt{\lambda(\alpha, \gamma, 1) \cdot (2 \sigma_r^2 + \sigma_m^2)}
    \label{eq:mdb_residual}
\end{equation}
where the non-centrality parameter $\lambda(\alpha, \gamma, \newnewtext{l_i})$ is computed as a function of false alarm rate $\alpha$ and detection power $\gamma$ (equal to the miss-detection rate of $1-\gamma$), by solving for the following equation
\begin{equation}
    P_{MD} = P \left[ \chi^2(1, \lambda) \leq \chi_{\alpha}^2(1, 0) \right] = 1 - \gamma
\end{equation}

Due to the correlation between $g_{i, j}$ and $g_{i, k} (i \neq k)$ via the common ephemeris error, the distribution of the test statistic under both hypotheses will not exactly follow a Chi-squared and non-central Chi-squared distribution, respectively. 
Based on this observation, \citet{xu2011gnss} proposed to compute a conservative threshold for the test statistic via numerical simulation 
by using the correlation between two links $g_{i, j}$ and $g_{i, k}$ as follows
\begin{equation}
    \rho(g_{ij}, g_{ik}) \leq \frac{\sigma_r^2}{2 \sigma_r^2 + \sigma_m^2}
\end{equation}
In this paper, we used Imhof's method~\citep{imhof1961computing} to compute the critical value for the test statistic. The detailed procedure is explained in Appendix~\ref{sec:imhof}.

The advantage of this \newnewtext{``ephemeris comparison''} approach is that the fault detection is fully distributed among the satellites.
\newnewtext{In addition, the method does not require a particular geometry of links to perform fault detection: if the ephemeris accuracy is sufficiently high, the fault can be detected by directly comparing the expected and measured range.}
The test is also able to detect the faults in the ephemeris information if they are present.
However, the disadvantage of the fault detection method is that the minimum detectable bias is dependent
on the ephemeris error $\sigma_r$, as shown in Equation \eqref{eq:mdb_residual}.

%%%%%%%%%%%%%%%%%%%%%%%%%%%%%%%%%%%%%%%%%%%%
% Baarda's $w$-test statistic
%%%%%%%%%%%%%%%%%%%%%%%%%%%%%%%%%%%%%%%%%%%%%%
\subsection{Data Snooping with Baarda's w-test Statistic}
\label{sec:baarda_test}
The test in Section \ref{sec:sum_of_residuals} does not use all the ISL measurements available in the constellation.
By testing the null hypothesis on a complete set of measurements from all ISLs, we can lower the minimum detectable bias 
when the measurement errors are smaller than the ephemeris errors.

Let $n$ be the number of satellites in the constellation, $\mathbf{X} \in \R^{3n}$ be the flattened vector of spacecraft positions at time $t_0$, $\mathbf{Y} \in \R^{m}$ be the vector of $m$ inter-satellite range measurements at time $t_0$.
We also define $b$ as the bias in the range measurement (equal to $c \delta t/2$ in Equations \eqref{eq:fault_i} and \eqref{eq:fault_j}) from one faulty satellite $f$ which experienced a clock phase jump $\delta t$.

The mapping between $\mathbf{X}$ and $\mathbf{Y}$ can be represented as the following nonlinear function $h$:
\begin{equation}
    \mathbf{Y} = h(\mathbf{X}) + \mathbf{c}_{f} b + \mathbf{\omega}, 
    \quad \mathbf{\omega} \sim \mathcal{N}(\mathbf{0}_m, \Sigma_{yy})
\end{equation}
where $\Sigma_{yy} = \sigma_m^2 \mathbf{I}_m$ is the covariance of the range measurements, and $\mathbf{c}_f$ is a known vector which maps the fault satellites to the corrupted range measurements, as follows
\begin{equation}
    {c}_{f}^{(i)} = \begin{cases}
        1  & \text{(satellite $f$ is \newnewtext{the} transmitter of link $i$) \& (clock phase jump at satellite $f$ occurred during link $i$)} \\
        -1 & \text{(satellite $f$ is \newnewtext{the} receiver of link $i$) \& (clock phase jump at satellite $f$ occurred during link $i$)}\\
        0  & \text{otherwise}
    \end{cases}
    \label{eq:cf_cases}
\end{equation}
where ${c}_{f}^{(i)}$ is the $i$th element of $\mathbf{c}_f$. 
By linearlizing the above equation around the ephemeris position $\bar{X}$, we obtain
\begin{equation}
    \mathbf{y} = \mathbf{Y} - h(\bar{X}) \approx \mathbf{H} (\mathbf{X} - \bar{\mathbf{X}}) + \mathbf{c}_f b + \mathbf{\omega} = \mathbf{Hx} + \mathbf{c}_f b + \mathbf{\omega}
\end{equation}
where $\mathbf{y} = \mathbf{Y} - h(\bar{X})$ and $\mathbf{x} = \mathbf{X} - \bar{\mathbf{X}}$, and $\mathbf{H} \in \R^{m \times 3n}$ is the Jacobian matrix.

For the test, we consider $n + 1$ hypotheses: The null hypothesis $\mathcal{H}_0$ where the model errors are absent, and the alternative hypotheses $\mathcal{H}_1, \ldots, \mathcal{H}_n$ which corresponds to the fault satellite being $f = 1, f = 2, \ldots, f = n$, respectively.
The Baarda's $w$-test statistic for testing $\mathcal{H}_0$ against single alternative $\mathcal{H}_k (k=1, \ldots, n)$ is given as~\citep{deJong2014}
\begin{equation}
    w_k = \frac{\mathbf{c}_k^{\top} \Sigma_{yy}^{-1} \newnewtext{\mathbf{P}_\mathbf{A}^{\perp}} y}{\sqrt{\mathbf{c}_k^{\top}  \Sigma_{yy}^{-1} \newnewtext{\mathbf{P}_\mathbf{A}^{\perp}} \mathbf{c}_k }}
\end{equation}
where the projector 
\newnewtext{$\mathbf{P}_\mathbf{A}^{\perp}$}
% $\mathbf{P_A}^{\perp}$ 
is defined as 
\begin{equation}
    \newnewtext{\mathbf{P}_\mathbf{A}^{\perp}} = \mathbf{I}_m - \mathbf{H} (\mathbf{H}^{\top} \mathbf{\Sigma_{yy}}^{-1} \mathbf{H})^{-1} \mathbf{H}^{\top} \mathbf{\Sigma_{yy}}^{-1}
    \label{eq:projector}
\end{equation}
and 
\begin{equation}
    {c}_{k}^{(i)} = \begin{cases}
        1  & \text{satellite $k$ is transmitter of link $i$} \\
        -1 & \text{satellite $k$ is receiver of link $i$} \\
        0  & \text{otherwise}
    \end{cases}
\end{equation}
Here, we assume that the clock phase jump at satellite $k$ corrupted all the ISLs that involve satellite $k$.
The testing procedure is summarized as follows~\citep{Imparato2018}.
\begin{equation}
    \text{Accept} \ \mathcal{H}_0 \ \text{if} \ \text{max}_{j \in \{1, \ldots, n\}} |w_j| < \sqrt{\chi_{\alpha}^2(1, 0)}
\end{equation}
otherwise
\begin{equation}
    \text{Accept} \ \mathcal{H}_k \ \text{if} \ |w_k| = \text{max}_{j \in \{1, \ldots, n\}} |w_j| \geq \sqrt{\chi_{\alpha}^2(1, 0)}
\end{equation}

The minimum detectable bias for $\mathcal{H}_k$ is computed as~\citep{deJong2014}
\begin{equation}
    MDB(k) = \sqrt{\frac{\lambda(\alpha, \gamma, 1)}{\mathbf{c}_k^{\top}  \Sigma_{yy}^{-1} \newnewtext{\mathbf{P}_\mathbf{A}^{\perp}} \mathbf{c}_k}}
\end{equation}

Iterating over $n$ alternate hypotheses using Baarda's $w$-test statistic (which is called data snooping) enables the detection of the existence of biases by considering the effect of the biases on the full set of observations. 
Each test is optimal for any bias magnitude $b$, whenever the assumptions of the test are true~\citep{Lehmann2017}. 
At the same time, it also has several limitations:
\begin{enumerate}
    \item In $c_k$, we assume that the biases are injected into all ISR measurements that contain a satellite with a clock phase jump. 
    However, this assumption is not always true since the clock jumps may have occurred outside the transmission and reception time for some links (the second condition in Equation \eqref{eq:cf_cases} is not satisfied), based on the satellite link schedule~\citep{Iiyama2024contact}.
    This will result in having smaller fault links than the assumed model ($\mathbf{c}_k$), and will result in decreased detection power.
    If we want to rigorously handle this uncertainty of $\mathbf{c}_k$, we would need to consider all possible combinations of the links as alternate hypothesis, but this would be computationally intractable.
    \item \newnewtext{As further discussed in Section~\ref{sec:topology}, the method fundamentally detects faults by checking the consistency among redundant measurements. 
    Therefore, a sufficient level of measurement redundancy is required. 
    When the measurement graph is sparse, and redundancy cannot be guaranteed, certain fault modes become unobservable, and the method may fail to detect them. 
    In these cases, it may be necessary to rely on direct comparison with ephemeris-based ranges, which do not depend on internal measurement consistency.
    }
\end{enumerate}
\newnewtext{
Despite these limitations, there remains significant room for methodological improvement to mitigate them. 
For example, instead of formulating $n$ hypotheses at the satellite (vertex) level, one could construct hypotheses at the edge (range) level, explicitly modeling which individual ISRs are corrupted. 
A subsequent inference layer could then estimate the faulty satellites from the set of detected faulty edges. 
Such a hierarchical formulation would better accommodate partial-link corruption.
We leave the development and evaluation of these extensions to future work.
}

\end{editsection}
\begin{editsection}{blue}
\section{Proposed Clock Fault Detection Algorithm Using Euclidean Distance Matrices}
\label{sec:rigidity_fd}
In this section, we propose a satellite clock fault detection algorithm based on rigid graph theory.
The method models the set of inter-satellite range measurements as a weighted graph and analyzes if the set of range measurements is realizable (embeddable) in 3D Euclidean space.
Instead of analyzing the realizability of the entire graph, the proposed method analyzes a set of subgraphs to detect faults.
The method directly processes the range measurements, and therefore does not require any prior information about the satellite positions or clocks.

The section is structured as follows. 
First, in section \ref{sec:topology}, we discuss the required topology to detect clock faults based on graph embeddability.
In particular, we prove that the graph needs to be vertex redundantly rigid (the minimum graph with such a property in 3D is a fully connected graph of 5 nodes) 
to detect biases in the range measurements from the graph's embeddability.
Next, in section \ref{sec:edm_property}, we show that the graph's embeddability can be determined by 
the properties of the geometric centered Euclidean distance matrix (GCEDM), without solving for the position of the nodes.
Finally, in section \ref{sec:fd_algorithm}, we propose a fault detection test based on the GCEDM distributions of
fully connected subgraphs in the graph.

\end{editsection}

\subsection{Required Graph Topology for Fault Detection}
\label{sec:topology}
\subsubsection{Definitions}
% In other words, we assume a constant bias $\faultmag$ gets inserted if either of the satellites in the edges is faulty.
% and 2$\faultmag$ get inserted if both satellites on the edges are fault satellites.

% Note that we need to solve for the light-time delay to compute the two-way range measurements between satellites~\citep{montenbruck2013satellite}, which requires some knowledge of the satellite orbit. For this paper, we assume this can be resolved using coarse information on the satellite orbits (e.g. Almanac) and the light time delay estimation error is included in $w_{ij}$.

\begin{editsection}{blue}
Consider a graph of $n$ nodes, which corresponds to satellites. 
Let $\singlepoint_i \in \R^d,  (i \in \{1, \ldots, n\})$ be the position of the $i$th node, and $\PointMat \in \R^{d \times n}$ be a matrix collecting these points. 
\newtext{In this paper, we are interested in the 3D case, $d=3$.}
In this section and the latter sections, we use the following notation for the range measurements between the two satellites $i$ and $j$, instead of equation \eqref{eq:range_t0}:
\begin{equation}
    \satrange_{ij} = \| \singlepoint_i - \singlepoint_j\| + \fault_{ij} + \noise_{ij} \\
\end{equation}
where $f_{ij} = \fault_i - \fault_j$ is the bias of the range measurements added by the clock phase jumps, and $\noise_{ij}$ is the measurement noise.
\end{editsection}

Based on the graph and observed ranges $\satrange_{ij}$, we define a weighted graph $G = \langle V, E, W \rangle$. 
A \textit{realization (embedding)} of $G$ is $d$-dimensional Euclidean distance space, $\R^d$, is defined as a function $\Phi$ that maps $V$ into $\R^d$, such that for each edge $e=(v_i, v_j) \in E$, $W(e) = r_{ij}$ \citep{jian_beyond_2010}. 
A weighted graph $G = \langle V, E, W \rangle$ is called \textit{d-embeddable} when there is a realization (embedding) $\Phi: V \rightarrow \R^{d} $ that maps the edges to points ($\singlepoint_1, \ldots, \singlepoint_n$) in $\R^d$ space.

\newtext{We define \newnewtext{a} ``fault satellite" as a satellite that experiences a clock phase jump, and corrupts at least one range measurement containing itself.}
\newtext{A weighted graph is \textit{fault disprovable} if and only if the embeddability of $G$ implies that it contains no fault satellites, and the non-embeddability of $G$ implies that it contains at least one fault satellite.}

A graph is called \textit{rigid} if it has no continuous deformation other than rotation, translation, and reflection while preserving the relative distance constraints between the vertices~\citep{laman_graphs_1970}. 
Otherwise, it is called a \textit{flexible} graph. 
A graph is called \textit{k-vertex redundantly rigid} if it remains rigid after removing any $(k-1)$ vertices. 
%Similarly, if it remains rigid after removing any $(k-1)$ edges, it is called \textit{k-edge redundantly rigid}.

\begin{figure}[ht]
\centering
\begin{minipage}[b]{0.33\columnwidth}
    \centering
    \includegraphics[width=0.9\columnwidth]{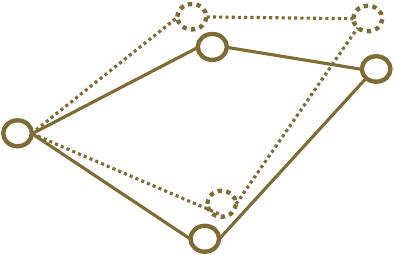}
    \subcaption{Flexible Graph}
    \label{fig:flexible_graph}
\end{minipage}
\begin{minipage}[b]{0.33\columnwidth}
    \centering
    \includegraphics[width=0.9\columnwidth]{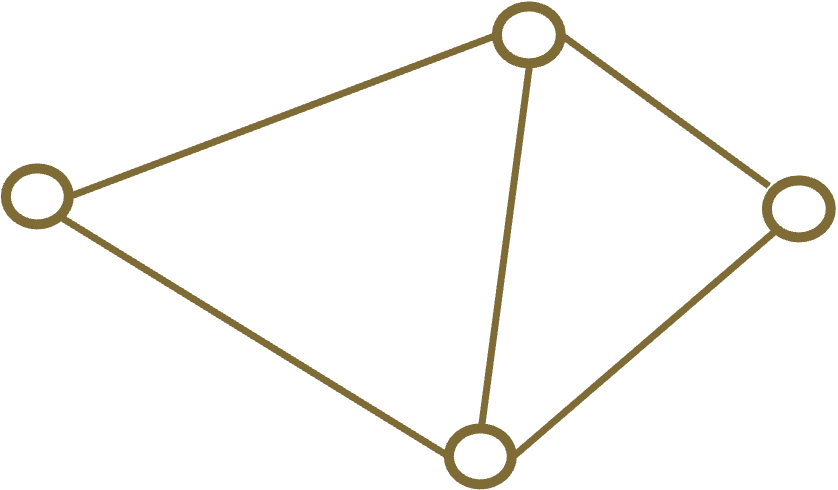}
    \subcaption{Rigid Graph}
    \label{fig:rigid_graphs}
\end{minipage}
\begin{minipage}[b]{0.33\columnwidth}
    \centering
    \includegraphics[width=0.9\columnwidth]{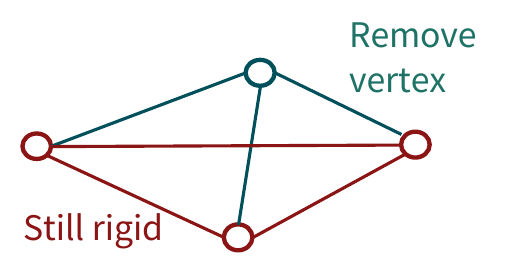}
    \subcaption{2-vertex redundantly rigid graph}
    \label{fig:red_rigid_graph}
\end{minipage}
\caption{Flexible, Rigid, and 2-vertex redundantly rigid graphs in $\R^2$}
\end{figure}

\subsubsection{Embeddability of Graphs with Fault Satellites}
In this section, we prove the sufficient condition for the graph to be fault disprovable in $\R^{d}$. 
\newtext{For the proofs of this subsection only, we ignore the noise added to the measurements.}

%%%%%%%%%%%%%%%%%%%%%
% Vertex Redundantly Rigid Case
%%%%%%%%%%%%%%%%%%%%%
\begin{proposition}
    \label{prop:fd_redudantly_rigd}
    Suppose there is \newtext{at most $k-1$} fault satellite\newnewtext{s} on a graph $G = \langle V, E, W \rangle $ that is $k$-vertex ($k \geq 2$) redundantly rigid in $\R^{d}$. 
    If $G$ is $d$-embeddable, then $G$ contains no fault vertex (satellite) with probability 1.
\end{proposition}

\begin{proof}
    See Appendix \ref{sec:app21}.
\end{proof}

\begin{proposition}
    \label{prop:fd_red_rigid2}
    Given a weighted graph $G$, if $G$ is $k$-vertex redundantly rigid ($k \geq 2)$, then $G$ is fault disprovable \newtext{at probability 1, if there are at most $k-1$ fault satellites}.
\end{proposition}

\begin{proof}
  For an arbitrary graph $G$, if $G$ is not embeddable, then $G$ contains \newtext{at least one} fault satellite.
  Therefore, with Proposition \ref{prop:fd_redudantly_rigd}, if $G$ is $k$-vertex redundantly rigid ($k \geq 2)$, then $G$ is fault disprovable \newtext{at probability 1, if there are at most $k-1$ fault satellites}.
\end{proof}

% %%%%%%%%%%%%%%%%%%%%%
% % Not Redundantly Rigid Case
% %%%%%%%%%%%%%%%%%%%%%
% \textbf{Necessary condition for fault disprovability}

% Next, we prove the necessary condition.
% \begin{proposition}
%     \label{prop:fd_not_redudantly_rigd}
%     Given a weighted graph $G$, if $G$ is fault disprovable, then $G$ is $k$-edge redundantly rigid ($k \geq 2)$.
% \end{proposition}

% \begin{proof}
%  See Appendix \ref{sec:app23}.
% \end{proof}

% \textbf{Summary}
% From proposition \ref{prop:fd_red_rigid2} and \ref{prop:fd_not_redudantly_rigd}, we obtain the following proposition.

% \begin{proposition}
%     Given a weighted graph $G$, $G$ is fault disprovable if $G$ is k-vertex redundantly rigid ($k \geq 2)$, and only if $G$ is k-edge redundantly rigid ($k \geq 2)$.
% \end{proposition}

\subsection{Properties of the Euclidean Distance Matrix}
\label{sec:edm_property}
In the previous section, we proved that by using vertex redundantly rigid graphs, we can detect faults by analyzing their embeddability.
Analyzing the embeddability of an arbitrary graph is known to be NP-hard ~\citep{Hendrickson1992}, but if the graph is fully connected, we can evaluate it by analyzing the ranks of the GCEDM. 
Note that fully connected graphs of more than 5 nodes are 2-vertex redundantly rigid in $\R^3$.

In this section, we introduce and prove several properties of the EDM and GCEDM constructed from the observed ranges between the satellites. 
We show how we can use the singular values of the GCEDM to evaluate if a (fully-connected) (sub)graph is embeddable in $\R^3$.

%%%%%%%%%%%%%%%%%%%%
%  Definitions
%%%%%%%%%%%%%%%%%%%%
\subsubsection{Geometric Centered Euclidean Distance Matrix}
% Consider a fully connected graph of $n$ nodes, which corresponds to satellites. 
% Let $\singlepoint_i \in \R^d,  (i \in \{1, \ldots, n\})$ be the position of the $i$th node, and $\PointMat \in \R^{d \times n}$ be a matrix collecting these points. 
% In this paper, we are interested in the 3D case, $d=3$.
% By establishing two-way inter-satellite links, we obtain the range measurement $r_{ij}$ between the two satellites $i$ and $j$, 
% \begin{equation}
%     \satrange_{ij} = \begin{cases}
%         \| \singlepoint_i - \singlepoint_j\| + \noise_{ij} + \fault_{ij} &  (i \neq j)  \\
%         0   & (i = j)  
%     \end{cases}
% \end{equation}
% Above, $w_{ij}$ is the range measurement noise, and $\fault_{ij}$ is the bias of the range measurements. 
% We assume the measurement noise $w_{ij}$ is sampled from a Gaussian distribution $\mathcal{N}(0, \sigma_w)$.
% The bias term is determined as follows
% \begin{align}
%     \fault_{ij} &= \fault_i + \fault_j,  \\
%     \fault_k &= \begin{cases}
%     \faultmag  & \text{satellite k is fault}  \\
%     0  & \text{satellite k is normal}
%     \end{cases}, \quad k \in \{i, j\}
% \end{align}
% which means we assume that the bias is zero when no satellites are in fault status. In other words, we assume a constant bias $\faultmag$ gets inserted if either of the satellites in the edges is faulty, and 2$\faultmag$ gets inserted if both satellites on the edges are fault satellites.

Using the set of observed ranges $\satrange_{ij}$, we construct an EDM, where its elements are equivalent to the square of the observed ranges ($\satrange_{ij}^2)$. 

Let  $\EDM^{n, d, m} \in \R^{n \times n}$ be an EDM without measurement noise ($\noise_{ij} = 0$) and $m$ nodes out of $n$ nodes being fault.
For example, a noiseless EDM with 6 satellites (in 3D space) with 2 fault satellites is denoted as $\EDM^{6, 3, 2}$.
Similarly, we define $\EDMnoisy^{n,d,m} \in \R^{n \times n}$ as EDM with measurement noise and $m$ nodes out of $n$ nodes having a fault.
A GCEDM $\GCEDM^{n,d,m}$ (or $\GCEDMnoisy^{n,d, m}$ for noisy EDM $\EDMnoisy^{n, d, m}$) is constructed from the EDM from the following operation
\begin{equation}
    \GCEDM^{n,d,m} = -\frac{1}{2} \centeringJ^n \EDM^{n,d,m} \centeringJ^n
\end{equation}
where $\centeringJ^n$ is the geometric centering matrix as follows.
\begin{equation}
    \centeringJ^{n} = \eye^n - \frac{1}{n} \onesvec {\onesvec}^{\top}
\end{equation}
where $\onesvec \in \R^n$ is the ones vector.
When no fault or noise is present $(m=0)$, the $\GCEDM^{n,d,0}$ will be positive semi-definite, and its rank will satisfy $\matrank(\GCEDM^{n,d,0}) = \matrank(\PointMat^{\top} \PointMat) \leq d$ \citep{dokmanic2015euclidean}.
%%%%%%%%%%%%%%%%%%%%
%  Rank Proofs
%%%%%%%%%%%%%%%%%%%%
\subsubsection{Rank of the Euclidean Distance Matrices with Fault Satellites} 
\newtext{\citet{Knowles2023}} observed that there are more than three nonzero singular values when the GCEDM is constructed from an EDM with fault satellites. An example is provided in Figure \ref{fig:eigval_logplot}.
However, mathematical proofs were not provided in their paper. 
In this section, we provide and prove several properties of the EDM and GCEDM that are corrupted with faults and measurement noises. 
% In the following proofs, we assume $\bar{f}$ is constant for all satellites, though this can be generalized to $\bar{f}$ being a random variable in future works.

\begin{figure}
    \centering
    \begin{subfigure}[b]{0.46\textwidth}
        \centering
        \includegraphics[width=\textwidth]{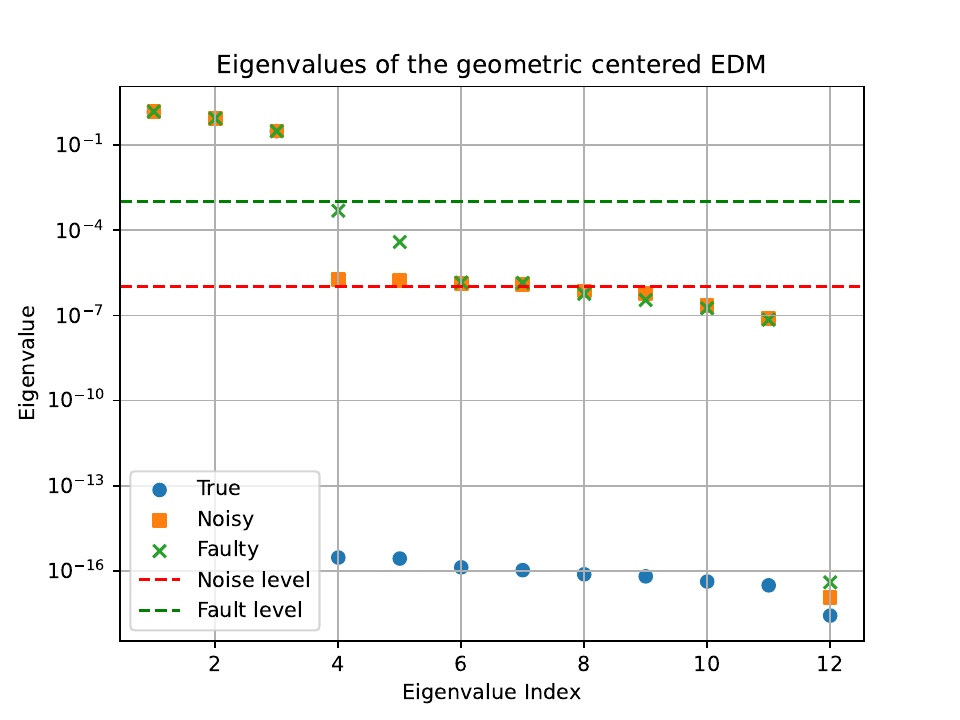}
        \subcaption{Singular values of the geometric centered EDM with one fault satellite. The fourth and fifth singular values increase compared to the non-fault case when the fault magnitude is sufficiently larger than the noise magnitude. \\
        }
    \end{subfigure}
    \hfill
    \begin{subfigure}[b]{0.46\textwidth}
        \centering
        \includegraphics[width=\textwidth]{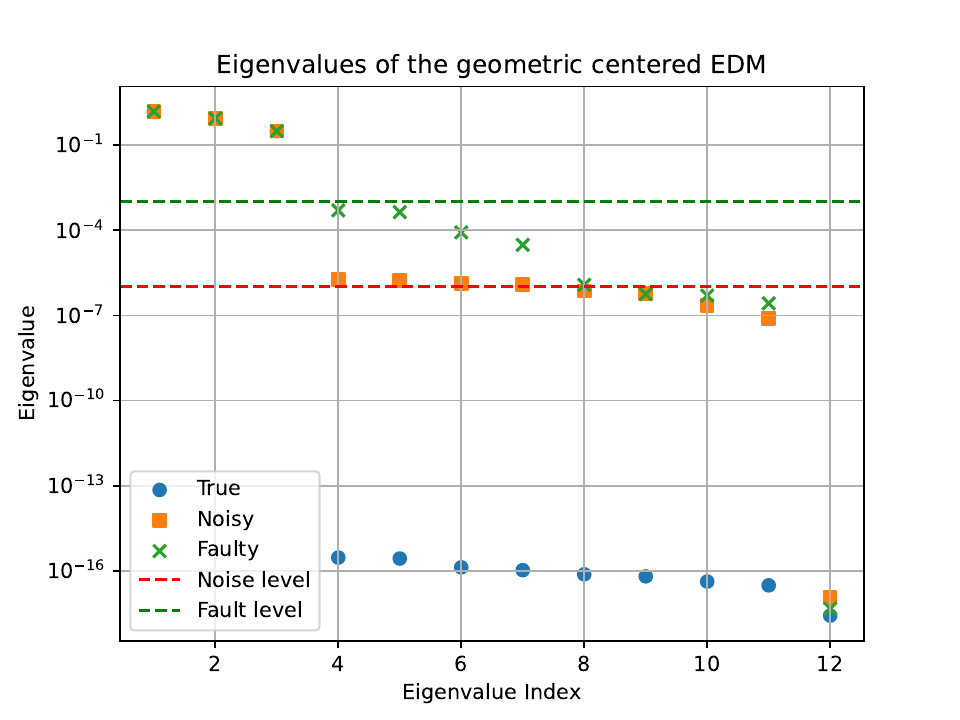}
        \subcaption{Singular values of the geometric centered EDM with two fault satellites. The fourth, fifth, sixth, and seventh singular values increase compared to the non-fault case when the fault magnitude is sufficiently larger than the noise magnitude.}
    \end{subfigure}
    \caption{The log plot of the singular values of an example geometric centered EDM for $n=12, d=3$. The singular values are ordered in descending order of their magnitude. The blue circles show the singular values when there is no fault or noise ($\GCEDM^{12, 3, 0}$), the orange squares show the singular values when there is noise (of scale $\sigma_w = 10^{-6}$) added to the range measurements ($\GCEDMnoisy^{12, 3, 0}$), and the green crosses show the singular values when both noise and fault (of scale $\bar{b} = 10^{-3}$) is added ($\GCEDMnoisy^{12, 3, 2}$). The 12 points are identical for all cases, which are sampled randomly inside a unit cube. When noise is present, the singular values increase except for the last singular value, which is close to zero (theoretically zero). When $m$ faults with magnitudes that are larger than the noise magnitude are added, the singular values up to the $d + 2m$ index increase compared to the non-fault case.}
    \label{fig:eigval_logplot}
\end{figure}

%%%%%%%%%%%%%%%%%%%
% EDM rank
%%%%%%%%%%%%%%%%%%%
\begin{proposition}
    \label{prop:edmrank}
    The rank of the EDM $\EDM^{n,d,m}$ satisfies
    \begin{equation}
        \matrank(\EDM^{n, d, m}) \leq \min(d + 2 + 2m, n) 
    \end{equation}
\end{proposition}

\begin{proof}
    See Appendix~\ref{sec:app1}.
\end{proof}

%%%%%%%%%%%%%%%%%%%
% GCEDM rank
%%%%%%%%%%%%%%%%%%%
\begin{proposition}
    \label{prop:geocenteredm}
    The rank of the GCEDM $\GCEDM^{n,d,m} = -\frac{1}{2} \centeringJ^n \EDM^{n, d, m} \centeringJ^n$  satisfies
    \begin{equation}
        \matrank(\GCEDM^{n,d,m}) \leq \min(d + 2m, n - 1) 
    \end{equation}
\end{proposition}

\begin{proof}
    See Appendix~\ref{sec:app2}.
\end{proof}

%%%%%%%%%%%%%%%%%%%
% Noisy GCEDM rank
%%%%%%%%%%%%%%%%%%%
\begin{corollary}
    \label{theorem:noisy}
        The GCEDM $\GCEDMnoisy^{n, d, m} = -\frac{1}{2} \centeringJ_n \EDMnoisy^{n, d, m} \centeringJ_n$ constructed from an EDM with the edges corrupted by Gaussian noise almost surely has rank
        \begin{equation}
            \matrank(\GCEDMnoisy^{n,d,m}) = n - 1 
        \end{equation}  
    \end{corollary}

\begin{proof}
    See Appendix~\ref{sec:app3}.
\end{proof}
%%%%%%%%%%%%%%%%%%%%
%  Distributions
%%%%%%%%%%%%%%%%%%%%
\subsubsection{Distribution of the 4th Singular Value of the Geometric Centered EDM}
\label{sec:distribution_singular}
% Outline
% \begin{enumerate}
%     \item Discussion of degenerate vs well-behaved cases (i.e., better to have multiple orbit planes)
%     \item Discussion of gamma distribution
%     \item Discussion that the gamma distribution parameters change with satellite geometry and graph topology -> Motivates regression of distribution characteristics
% \end{enumerate}

The propositions proved in the previous sections indicate that we can detect fault satellites by observing the increase in the 4th singular value (since the dimension is 3) of the GCEDMs constructed from the ISR measurements. 
We use the following test statistic $\teststatistic$ to monitor if a fault satellite exists in a given subgraph.
\begin{equation}
  \newtext{\teststatistic = \lambda_4^2}
\end{equation}
where $\lambda_i$ is the $i$th singular value of the geometric centered EDM.
If the ISR measurements are completely noiseless, we can detect if a fault satellite exists within the graph by checking if $\teststatistic$ of the GCEDMs is not 0, since GCEDMs have rank 3 when no fault exists. 
However, when the ISR measurements are noisy, the 4th singular value increases regardless of faults, 
as shown in Figure \ref{fig:eigval_logplot} and Corollary \ref{theorem:noisy}. 
Therefore, to detect faults in the presence of noise, we need to set a threshold based on the singular value distributions of the noisy (but non-fault) GCEDMs to separate the fault from the noise.

\begin{editsection}{blue}
% noise
For $n=5$, we are able to compute the distribution of the test statistic analytically.
The distribution of the test statistic $\teststatistic/\newnewtext{s^2}$ follows a centralized chi-squared distribution with degree of freedom 1, as follows:
\begin{align}
    \frac{\teststatistic}{\newnewtext{s^2}} &\sim \chi^2 \left(1, 0 \right) \quad \text{for } n=5  
    \label{eq:chi2_teststatistic}
\end{align}
where the \newnewtext{squared scale of the test statistic} $\newnewtext{s^2}$ is given by the following equation:
\begin{align}
    \newnewtext{s^2} &=  
    \newnewtext{ \sum_{a=1}^2 \sum_{b=1}^2 \sum_{i=1}^5 \sum_{j=1}^5 \left(\sigma_m \mathbf{\tilde{D}}_{ij}\right)^2 \cdot 
    \left( \hat{U}_{i,a} \hat{V}_{j,b} + \hat{U}_{j,a} \hat{V}_{i,b} \right)^2} 
    \label{eq:scale} \\
    \mathbf{\hat{U}} &= \mathbf{J}_5 \mathbf{\tilde{U}}_{:, 4:5} \in \R^{5 \times 2} \quad \quad \mathbf{\tilde{U}}_{:, 4:5}: \text{4th and 5th column of } \tilde{U} \in \R^{5 \times 5}\\
    \mathbf{\hat{V}} &= \mathbf{J}_5 \mathbf{\tilde{V}}_{:, 4:5} \in \R^{5 \times 2} \quad \quad \mathbf{\tilde{V}}_{:, 4:5}: \text{4th and 5th column of } \tilde{V} \in \R^{5 \times 5}\\
    \mathbf{\tilde{U} \tilde{\Sigma} \tilde{V}}^{\top} &= \text{SVD}(\tilde{\mathbf{G}}) \quad \quad  \text{SVD}: \text{Singular Value Decomposition}, \quad \teststatistic = {\tilde{\Sigma}_{4,4}}^2 = (\lambda_4)^2
    \label{eq:scale_formula}
\end{align}
Above, $\tilde{\mathbf{G}}, \tilde{\mathbf{D}} \in \R^{5 times 5}$ are the (noisy) GCEDM and the distance matrix constructed from the measured ranges, respectively.
The derivation is given in Appendix \ref{sec:scale_distribution}. 
The \newnewtext{squared scale} of the distribution is proportional to the noise \newnewtext{variance} $\sigma_m^2$ (Figure~\ref{fig:sigma_chi2}).
% 
% Note that e
Equation \eqref{eq:scale} can be modified to compute $\newnewtext{s^2}$ when the variance of noise is different for each measurement, by replacing $\sigma_m$ with $(\sigma_m)_{ij}$ (the noise for each range measurement).

% fault
When there is a clock jump at satellite $s_f$, the distribution of the test statistic $({\teststatistic})_{s_f}$ follows a non-central chi-squared distribution with 1 degree of freedom, 
with its \newnewtext{squared scale $s^2$} given by the same equation as above, and non-centrality parameter $\lambda_{s_f}$:
\begin{align}
  \frac{({\teststatistic})_{s_f}}{\newnewtext{s^2}} &\sim \chi^2 \left(1, \lambda_{s_f} \right) \quad \text{for } n=5
  \label{eq:fault_dist_ncx}
\end{align}
where the non-centrality parameter $\lambda_{s_f}$ is given by the following equation:
\begin{align}
   \lambda_{s_f} &= \frac{1}{\newnewtext{s^2}} \| \hat{\mathbf{U}}^{\top} \left(\mathbf{F} \circ \mathbf{D} \right) \hat{\mathbf{V}} \|_F^2 
   \label{eq:lambda} \\
   \mathbf{F} &= \begin{cases}
       f_{ij} &  i = s_f \ \text{or} \ j = s_f \\
       0             &  \text{otherwise}
   \end{cases}
   \label{eq:lambda_sf}
\end{align}
\newnewtext{where $\|\cdot\|_F$ represents the Frobenius norm.}
\newnewtext{This indicates that} when all biases added by faults have the same magnitude, the non-centrality parameter $\lambda$ is proportional to the square of the fault magnitude
as shown in Figure \ref{fig:fault_distribution} and Figure \ref{fig:fault_lambda}.
The derivation is given in Appendix \ref{sec:lambda_distribution}.

\begin{figure}[ht!]
  \centering
  \begin{subfigure}[b]{0.2\textwidth}
    \includegraphics[height=45mm, trim={20mm, 0mm, 10mm, 0mm}]{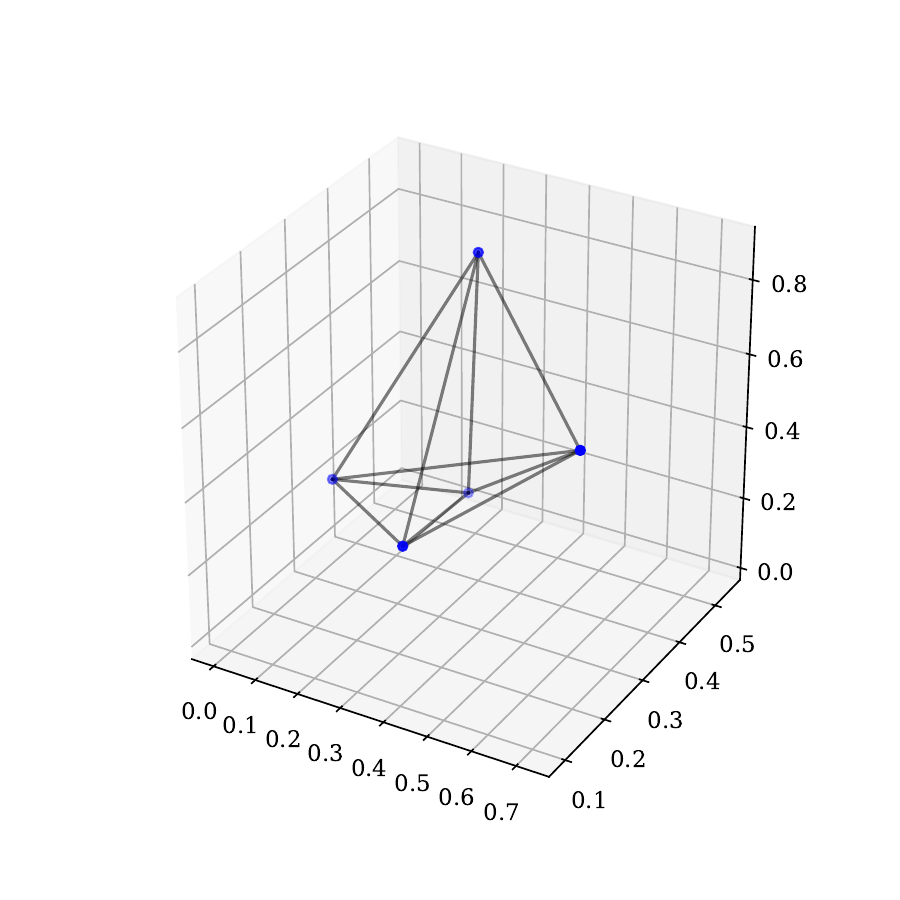}
    \newcaption{The \newnewtext{3D} points used for singular value simulation.}
  \end{subfigure}
  \hfill
  \begin{subfigure}[b]{0.38\textwidth}
    \includegraphics[height=40mm, trim={0mm, 0mm, 0mm, 0mm}]{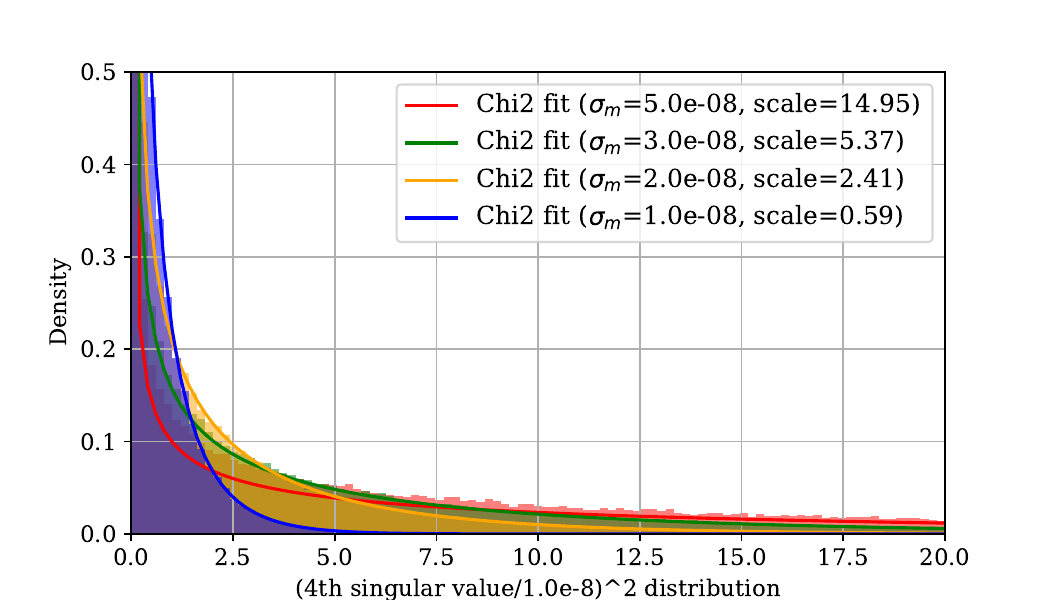}
    \newcaption{The distribution of the squared 4th singular value for different noise levels $\sigma_m$.}
  \end{subfigure}
  \hfill
  \begin{subfigure}[b]{0.38\textwidth}
    \includegraphics[height=40mm, trim={0mm, 0mm, 20mm, 0mm}]{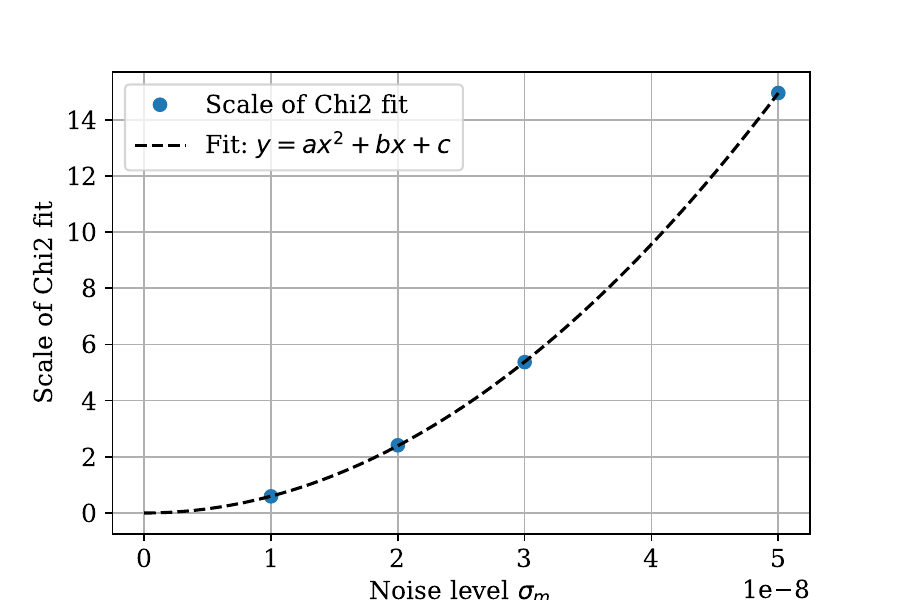}
    \newcaption{The scale of the chi squared distribution $s$ is proportional to the noise magnitude $\sigma_m^2$.}
    \label{fig:sigma_scale}
  \end{subfigure}
  \newcaption{The distribution of the test statistic $\teststatistic$ and its scale $\newnewtext{s^2}$ (squared scale of $\lambda_4$) for different noise levels $\sigma_m$. 
  The distribution of the test statistic $\teststatistic$ follows a chi-squared distribution with 1 degree of freedom.
  The scale $\newnewtext{s^2}$ is proportional to the noise magnitude $\sigma_m^2$.}
  \label{fig:sigma_chi2}
  \end{figure}

\begin{figure}[ht!]
\centering
\includegraphics[width=\textwidth, trim={0mm, 0mm, 0mm, 20mm}]{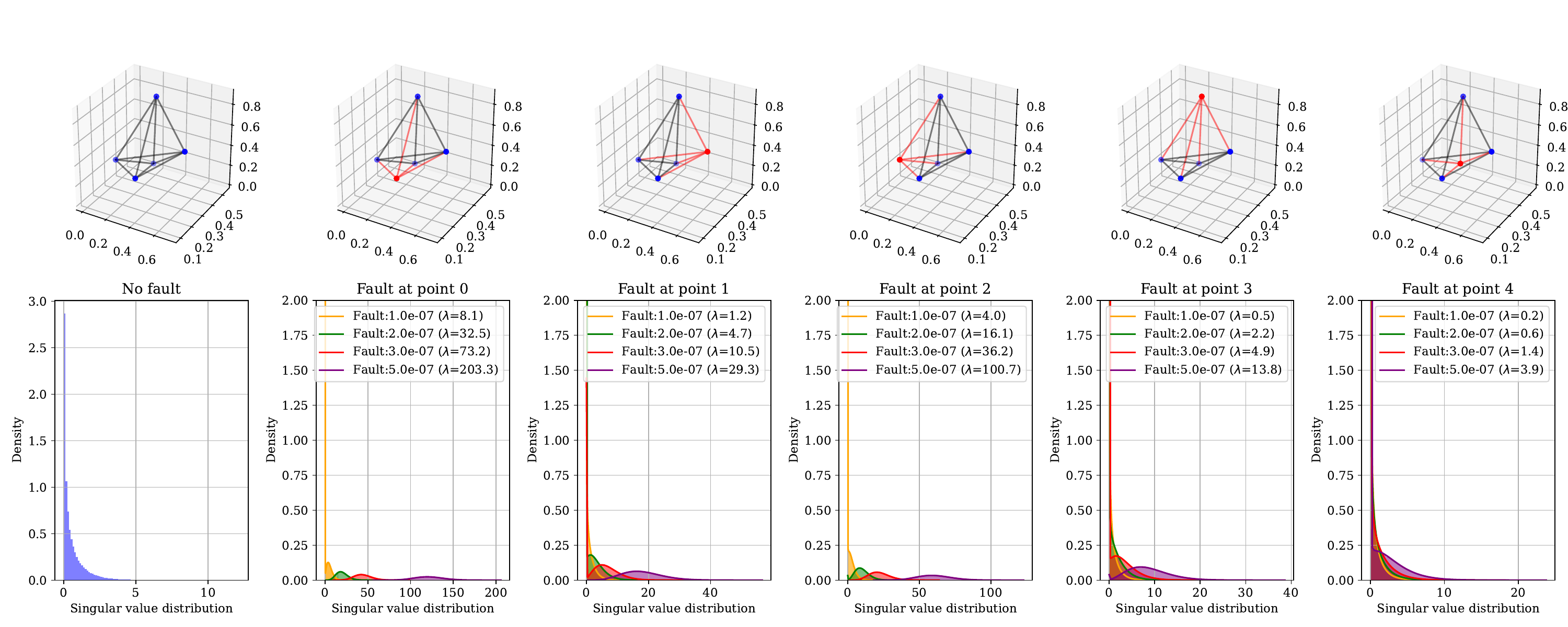}
\newcaption{The distribution of the test statistic $\teststatistic$ for different fault magnitudes and fault satellites. 
The top row shows the points in 3d space, where the red points and edges indicate the fault satellite and its connected edges, respectively.
The distribution of the test statistic $\teststatistic$ follows a non-central chi-squared distribution with 1 degree of freedom.
The non-centrality parameter $\lambda$ becomes larger as the fault magnitude increases.}
\label{fig:fault_distribution}
\end{figure}

\begin{figure}[ht!]
\centering
\includegraphics[width=0.5\textwidth]{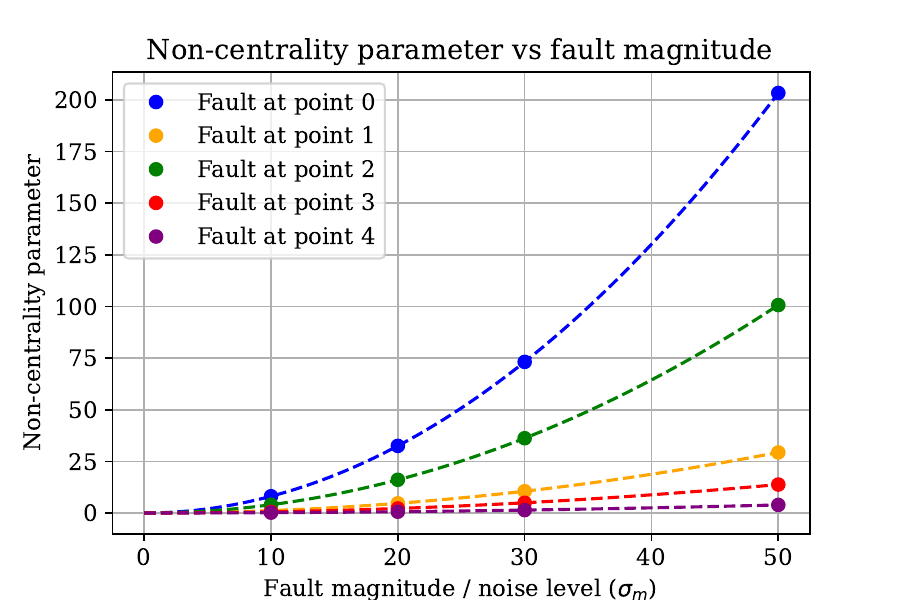}
\newcaption{The non-centrality parameter $\lambda$ of the distribution of the test statistic is proportional to the square of the fault magnitude $\bar{b}^2$.}
\label{fig:fault_lambda}
\end{figure}

Using equations \eqref{eq:chi2_teststatistic} and \eqref{eq:fault_dist_ncx}, we can compute the minimum detectable bias for satellite $i$ as follows:
\begin{equation}
    MDB(i) = \sqrt{\frac{\bar{\lambda}}{\lambda_i} \cdot \newnewtext{s^2}} \quad 
    \text{where} \quad P[\chi^2(1, \bar{\lambda}) \leq \chi^2_{\alpha}(1, 0)] = 1 - \gamma
\end{equation}
where $\alpha$ is the false alarm rate, $\gamma$ is the detection power, $\newnewtext{s^2}$ is the \newnewtext{squared scale} provided in \eqref{eq:scale}, and $\lambda_i$ is the non-centrality parameter for satellite $i$ fault case, provide in equation \eqref{eq:lambda}.

\newnewtext{
Figure \ref{fig:mdb_comparison} shows the computed minimum detectable biases for the two baseline methods (ephemeris comparison and data snooping) and the EDM-based fault detection, using the same geometry as Figures \ref{fig:sigma_chi2} and \ref{fig:fault_distribution}.
The MDB of the data snooping marks the lowest due to its optimality, while the MDB of the ephemeris comparison method largely depends on the ephemeris accuracy.
The MDB of the EDM-based detection depends on the relative geometry of each point with respect to other points, which affects the non-centrality parameter $\lambda_{sf}$. 
}

\begin{figure}[ht!]
    \centering
    \includegraphics[width=\linewidth]{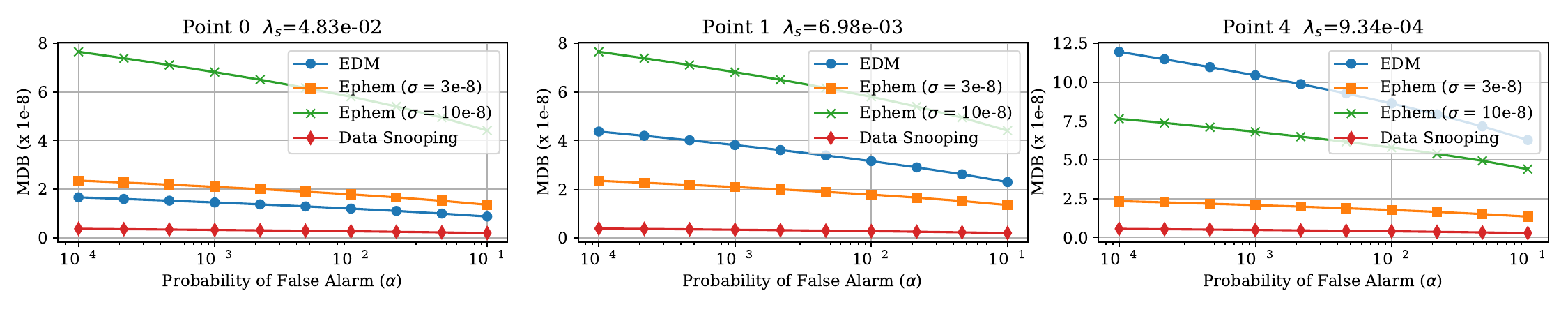}
    \caption{\newnewtext{The comparison of minimum detectable biases for different fault detection methods. The topology of the points and links is the same as Figure \ref{fig:sigma_chi2}. Detection power $\gamma$ is fixed to 0.8, and the noise level is fixed to 1e-8. $\lambda_s$ in the figure reprensents $\lambda_{sf}$ in Equation \ref{eq:lambda_sf}.}}
    \label{fig:mdb_comparison}
\end{figure}

\begin{newnewsection}

If desired, the test statistic can be extended for $n \geq 6$ as:
\begin{align}
    \lambda_{test} &= \sum_{i=4}^{n-1} \lambda_i^2
\end{align}
The readers are referred to Appendix \ref{sec:generalization_test} for the distribution of this test statistic under nominal conditions and the existence of faults.
\end{newnewsection}

\end{editsection}

\subsection{Fault Detection Algorithm Using Euclidean Distance Matrix}
\label{sec:algorithm}

\subsubsection{Clique Listing}
\label{sec:clique-finding}
In order to construct an EDM and GCEDM, we need a set of range measurements between all pairs of $n$~satellites. However, due to the occultation by the planetary body and attitude constraints, the satellites in the constellation are not fully connected with ISRs in most cases.
Therefore, the first step of the fault detection algorithm is to find a set of $k$-clique subgraphs, which are fully connected subgraphs of $k$~satellites. 
\newtext{For the graph to be 2-vertex redundantly rigid, and to use the derived distribution \eqref{eq:chi2_teststatistic}, we are interested in cliques with $k = 5$.}
An example of 5-cliques subgraphs is shown in Figure \ref{fig:cliques_example}.

Various exact algorithms to find all $k$-cliques have been proposed~\citep{Li2020OrderingHF}. In this paper, we used the Chiba-Nishizeki Algorithm (Albo) due to its simplicity of implementation~\citep{Chiba1985}. For each node $v$, Arbo expands the list of $k$-cliques by recursively creating a subgraph induced by $v$'s neighbors. The readers are referred to \citep{Li2020OrderingHF} for the pseudocode and summary of the $k$-clique listing algorithm.
Note that the clique finding algorithm does not have to be executed online. Since the availability of the ISRs can be predicted beforehand \newtext{when scheduling the ISLs} using a coarse predicted orbit, we can estimate the topology of the ISLs for future time epochs. 
We can compute the list of $k$-cliques of the predicted topologies for future time epochs (e.g., one orbit), and uplink them to satellites intermittently. 
If some of these links were actually not available in orbit for some reason, we can remove the $k$-cliques that contain the missing links from the list.

\begin{figure}[htb!]
    \centering
    \begin{subfigure}[b]{0.24\textwidth}
        \centering
        \includegraphics[width=0.99\linewidth]{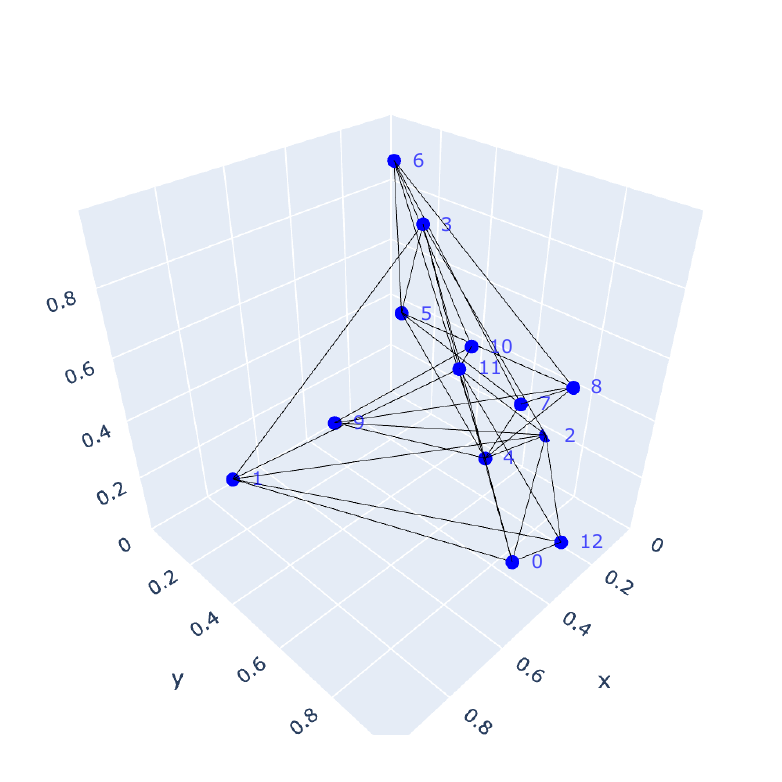}
        \subcaption{The topology of the entire graph}
    \end{subfigure}
    \hfill
    \begin{subfigure}[b]{0.24\textwidth}
        \centering
        \includegraphics[width=0.99\linewidth]{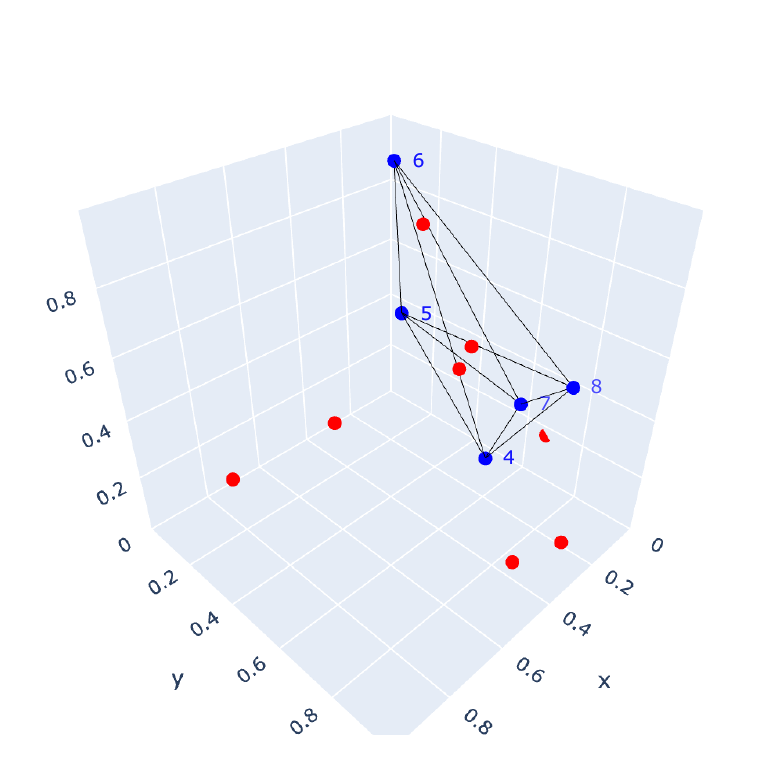}
        \subcaption{5-clique subgraph 1}
    \end{subfigure}
    \hfill
    \begin{subfigure}[b]{0.24\textwidth}
        \centering
        \includegraphics[width=0.99\linewidth]{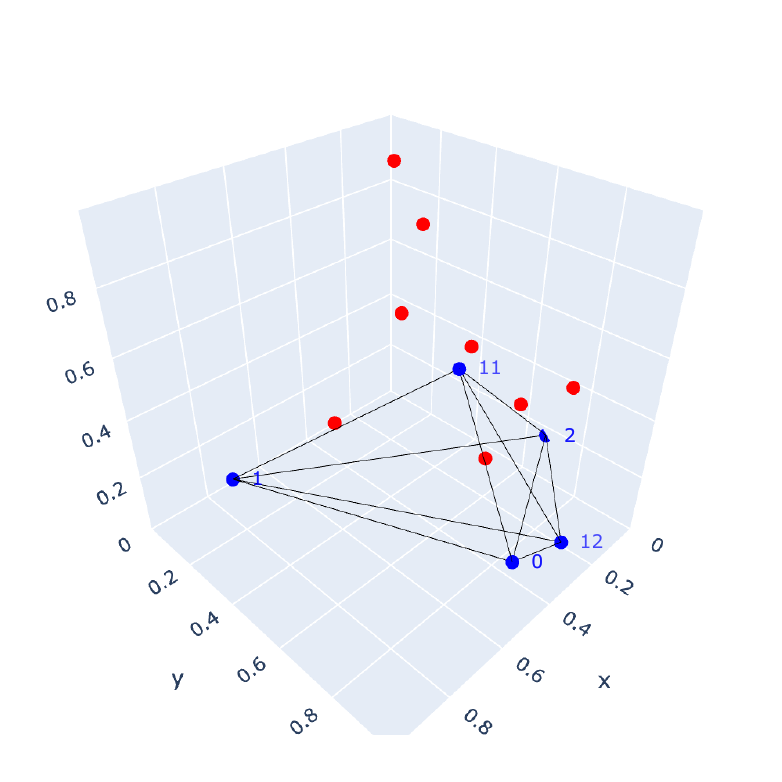}
        \subcaption{5-clique subgraph 2}
    \end{subfigure}
    \hfill
    \begin{subfigure}[b]{0.24\textwidth}
        \centering
        \includegraphics[width=0.99\linewidth]{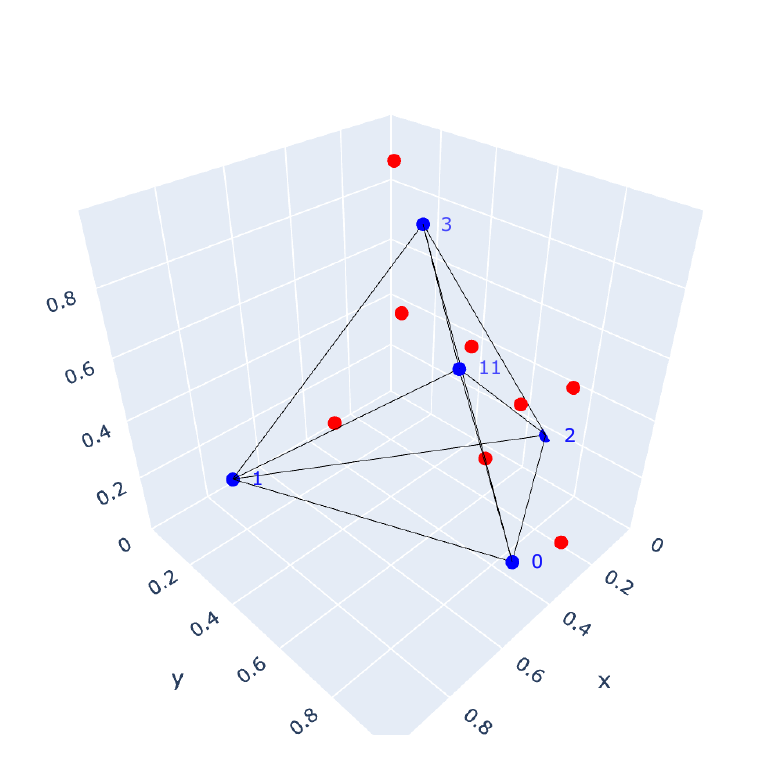}
        \subcaption{5-clique subgraph 3}
    \end{subfigure}
    \caption{Listing 5-cliques within the graph of 13 nodes \newtext{that are not fully connected.}}
    \label{fig:cliques_example}
\end{figure}

\begin{editsection}{blue}

\subsubsection{Online Fault Detection}
\label{sec:fd_algorithm}
We define the ``fault satellite'' as the satellite that has a clock jump, and the ``fault-free satellite'' as the satellite that does not have a clock jump.
We design a test where we consider the following hypotheses:
\begin{equation}
    \begin{aligned}
        H_0: \ & \text{All satellites are fault-free (null hypothesis)} \\
        H_i: \ & \text{The satellite } i \text{ is a fault satellite} \\
    \end{aligned}
\end{equation}
Let $\mathcal{L}^{G}_{t}$ be the list of all 5-cliques of the satellite network as a graph at time step $t$, and $N_s$ the total number of satellites.
We propose an online fault detection algorithm, which uses $\mathcal{L}^{G}_{t}$. 
The procedure of the detection algorithm is as follows. 

\textbf{Algorithm:}
\begin{enumerate}
% 1. compute threshold for each subgraph
\item First, we compute the scaled test statistic $\gamma_{test, G^{\prime}} = \lambda_4^2/\newnewtext{s^2}$ for each 5-clique subgraph $G^{\prime}$ in $\mathcal{L}^{G}_{t}$, using Eq.\eqref{eq:scale}.
\item For each satellite $i$ in the graph, repeat the following steps (a) to (c):
\begin{enumerate}
    \item Compute the sum of all the scaled test statistics for subgraphs that do not include satellite $i$.
        \begin{equation}
            \gamma_{\text{test}}^i = \sum_{G^{\prime} \in \mathcal{L}^{G}_{t}, i \notin G^{\prime}} \gamma_{test, G^{\prime}}.
        \end{equation}
    \item Compute the threshold $\teststatisticthreshold^i$ for the sum of the scaled test statistics $\gamma_{test}^i$ for a given false alarm rate $\alpha$ and (user-defined) threshold margin $\eta_{\alpha}$.
        \begin{equation}
            \teststatisticthreshold^i = {\eta}_{\alpha} \cdot \chi_{\alpha}^2 (N_{G, i}, 0)
        \end{equation}
        where $N_{G, i}$ is the number of subgraphs that do not include satellite $i$, $\eta_{\alpha}$ is the threshold margin, 
        and $\chi_{\alpha}^2 (N_{G, i}, 0)$ is the $\alpha$-quantile of the central chi-squared distribution with $N_{G, i}$ degrees of freedom.
    \item Normalize the test statistic by dividing it by the threshold $\teststatisticthreshold^i$.
        \begin{equation}
            \teststatistic^i = \frac{\gamma_{test}^i}{\teststatisticthreshold^i}.
        \end{equation}
\end{enumerate}
% If for all satellite $i$ we have $\teststatistic^i < 1$, we accept the null hypothesis, and otherwise we reject the null hypothesis.
\item If for \newnewtext{each} satellite $i$ \newnewtext{that has} $\teststatistic^i < 1$, we accept the null hypothesis, and otherwise we reject the null hypothesis.
We identify the fault satellite $s_f$ as the one with the smallest $\teststatistic^i$.
\begin{equation}
    s_f = \text{argmin}_{i} \teststatistic^i
\end{equation}
\end{enumerate}

% Hyperparameters
The proposed fault detection algorithm has two hyperparameters that affect its performance. 

False Alarm Rate $\alpha$: This is the probability of a false alarm, which is the probability of detecting a satellite as faulty when it is actually fault-free.

Threshold Margin ${\eta}_{\alpha}$: This inflates the threshold $\teststatisticthreshold$ to ${\eta}_{\alpha} \cdot \teststatisticthreshold$, to account for the 
increase in the variance of the sum of the test statistics, due to the correlation between the test statistics of the subgraphs.

% Discussion about threshold margin
We need to set the threshold margin $\eta_{\alpha}$ to a value larger than 1, 
to account for the increase in the variance of the distribution of $\gamma_{test}^i$,
Due to the correlation between the test statistics of the subgraphs.
The correlation between the test statistics of the subgraphs is caused by the fact that the same set of range measurements is used to compute the test statistics between different subgraphs.
We compute the test statistics for subgraphs that do not contain satellite $i$, instead of subgraphs that contain $i$, to avoid having strong correlations between the test statistics of the subgraphs, which would require a larger threshold margin $\eta_{\alpha}$.

% Multiple faults
The above test only considers a maximum of one fault satellite. 
One possible extension to the test to consider multiple fault satellites is to run the algorithm iteratively,
removing the detected fault satellite from the graph after each iteration until no fault is detected~\citep{knowles2024greedy}. 
Alternatively, we can consider $_{N_s}C_{N_f}$ combinations of fault satellites, where $N_f$ is the number of fault satellites, and compute the test for each combination.
The detection performance in the case of multiple faults will be performed in future work.

% Required precision
\newnewtext{
One practical consideration of the proposed approach is that, for satellite constellations, the fourth singular value of the GCEDM can become extremely small ($\sim 10^{-7}$) due to the large inter-satellite distances (in orders of $10^{6}$ m) compared to range measurement accuracy (in orders of sub-meters to meters).
This can lead to values that approach machine precision when operating in single precision, necessitating the use of double-precision arithmetic.
}

\end{editsection}
\begin{newnewsection}{red}
\subsubsection{Fault Detectable Graphs with Ephemeris Augmentation}
\label{sec:detectable_subgraphs}

In scenarios where the links are sparse or the number of satellites is small, the graph may lack a sufficient number of 5-cliques to perform robust fault detection using only ISR measurements. To mitigate this, we can augment the graph by ``filling in'' missing edges using range estimates derived from satellite ephemerides. This hybrid approach utilizes a combination of measured edges (ISR) and computed edges (ephemeris) to construct redundantly rigid subgraphs for fault detection.

However, a fault (specifically a clock phase jump) appears only in the physical ISR measurements and does not affect the ranges computed purely from ephemeris data. Consequently, for an augmented graph to reliably detect faults, the topology of the measured edges must satisfy a specific connectivity condition.

\begin{proposition}
    \label{prop:fault_detectable_ephemeris}
    Consider a redundantly rigid weighted graph $G = \langle V, E, W \rangle$, where the edge set is partitioned into measured edges $E_m$ (ISR measurements) and computed edges $E_e$ (derived from ephemeris), such that $E = E_m \cup E_e$. Let a fault at satellite $v \in V$ be defined as a bias that corrupts the weights of all incident edges in $E_m$. The graph $G$ is \textit{fault disprovable} (i.e., the graph becomes unrealizable under a fault) if and only if every vertex $v \in V$ is incident to at least one measured edge in $E_m$.
\end{proposition}

\begin{proof}
   See Appendix \ref{sec:fault_detectable_subgraph_proof}.
\end{proof}

We denote the $k$-node subgraphs that satisfy the connectivity condition in Proposition \ref{prop:fd_red_rigid2} (i.e., every vertex is incident to at least one measured edge) as \textit{fault-detectable $k$-node subgraphs}.

In scenarios where the graph is sparse or the number of satellites is small, the standard 5-clique list $\mathcal{L}^G_t$ described in Section \ref{sec:fd_algorithm} may be empty or insufficient for robust detection. In such cases, we replace $\mathcal{L}^G_t$ with the list of fault-detectable 5-node subgraphs, denoted as $\mathcal{L}^{FD}_t$. These subgraphs are constructed by identifying sets of 5 nodes and ``filling in'' missing measured edges with computed edges derived from the ephemeris to form fully connected 5-cliques.

To account for the difference in error characteristics between measured and computed edges, we modify the computation of the test statistic variance $s^2$. In Equation \eqref{eq:scale_formula} (and the generalized form in Equation \eqref{eq:generalized_scale}), the uniform measurement noise variance $(\sigma_m)_{ij}^2$ is replaced by an edge-dependent variance $(\sigma_{link})_{ij}^2$:

\begin{equation}
    (\sigma_{link})_{ij}^2 = 
    \begin{cases} 
        \sigma_m^2 & \text{if } (i,j) \in E_m \quad \text{(Measured Edge)} \\
        2\sigma_r^2 & \text{if } (i,j) \in E_e \quad \text{(Computed Edge)}
    \end{cases}
\end{equation}
where $\sigma_m^2$ is the variance of the ISR measurement noise and $\sigma_r^2$ is the variance of the satellite position error (assuming independent position errors, the variance of the computed range error is approximately $2\sigma_r^2$).

This substitution drastically increases the number of subgraphs available for fault detection, particularly in sparse constellations or during orbital phases with limited visibility. 
It will also enable the detection of ephemeris faults.
However, this improved availability comes at a cost: the minimum detectable bias (MDB) for these subgraphs becomes sensitive to the magnitude of ephemeris errors ($\sigma_r$), rather than being driven solely by the typically smaller measurement noise ($\sigma_m$).

\end{newnewsection}

\section{Satellite Fault Detection Simulation}
\label{sec:simulation}
% \begin{enumerate}
%     \item Metrics (i.e., TPR)
%     \item Earth / Mars Case
%     \item Moon Case
% \end{enumerate}

%%%%%%%%%%%%%%%%%%%%%%%%%%%%%%%
% Simulation Configuration
%%%%%%%%%%%%%%%%%%%%%%%%%%%%%%
\subsection{Simulation Configuration}

%%%%%%%%%%%%%%%%%%%%%%%%%%%%%%%
% Orbits and Links
%%%%%%%%%%%%%%%%%%%%%%%%%%%%%%
\begin{editsection}{blue}
\subsubsection{Orbits and Links}
To validate the proposed fault detection algorithm, we consider the following two constellations.

\begin{enumerate}
\item GPS constellation of 31 satellites:
The satellite orbits are initialized using the Two-Line Element (TLE) data from 2025/4/3 from the CelesTrak website~\citep{celestrak_gp_data_formats_2025}.
The 3-dimensional plots of the constellation and the available ISLs at $t=0$ are shown in Figure \ref{fig:earth_orbit}. 

\item A notational lunar constellation of \newnewtext{9} satellites: 
% We considered a hybrid constellation, where 9 satellites form a 9/3/1 Walker-Delta constellation and 8 satellites are deployed in elliptical lunar frozen orbits (ELFO).
% The constellation is based on the design proposed by \citet{kang2025design}, which provides a good coverage of the lunar south pole, while maintaining a good global coverage.
\newnewtext{We consider a constellation of 9 satellites in 3 orbital planes of elliptical lunar frozen orbit (ELFO) with a 30-hour orbital period, as a potential extension of the Lunar Communications Relay and Navigation Systems (LCRNS)~\citep{LCRNS2025}.}
The orbital elements of the constellation are shown in Table \ref{tab:moon_constellation}, and
the 3-dimensional plots of the constellation and the available ISLs at $t=0$ are shown in Figure \ref{fig:moon_orbit}. 
\end{enumerate}

For both constellations, the orbits are propagated in the two-body propagator without considering perturbations.
The availability of the links is calculated assuming that the links are available when it is not occulted by the central body, and the angle between the line of sight and the satellite-body center vector is less than $\phi_{max}$, as illustrated in Figure \ref{fig:visibility}.
The cutoff angle $\phi_{max}$ is set to 60 degrees for the Earth case~\citep{Zhang2024} and \newnewtext{90} degrees for the Moon case.
To compute body occultation, we set an additional mask of 1000 km for the Earth case, and 100 km for the Moon case to avoid the link being corrupted by ionospheric delays or terrain blockage\newnewtext{, respectively}.
We assume that we can obtain the range measurements from all available links at each time step. 
The scheduling of the links will be investigated in future work. 

The \newnewtext{scatter plot of the number of 5 cliques in the constellation containing each satellite} is shown in Figure \ref{fig:earth_nclique} and Figure \ref{fig:moon_nclique} for the Earth and Moon cases, respectively.
\newnewtext{For the lunar constellation, we observe that the number of available 5-cliques is significantly lower (below 30), limiting the use of EDM-based fault detection from ISR alone. Therefore, for the lunar case, we used the ephemeris augmentation described in Section \ref{sec:detectable_subgraphs} to increase the number of subgraphs that can be used for fault detection. The number of fault-detectable 5-node subgraphs are shown in \ref{fig:moon_detectable}}

% Visibility
\begin{figure}[ht!]
    \centering
    \includegraphics[width=0.4\textwidth]{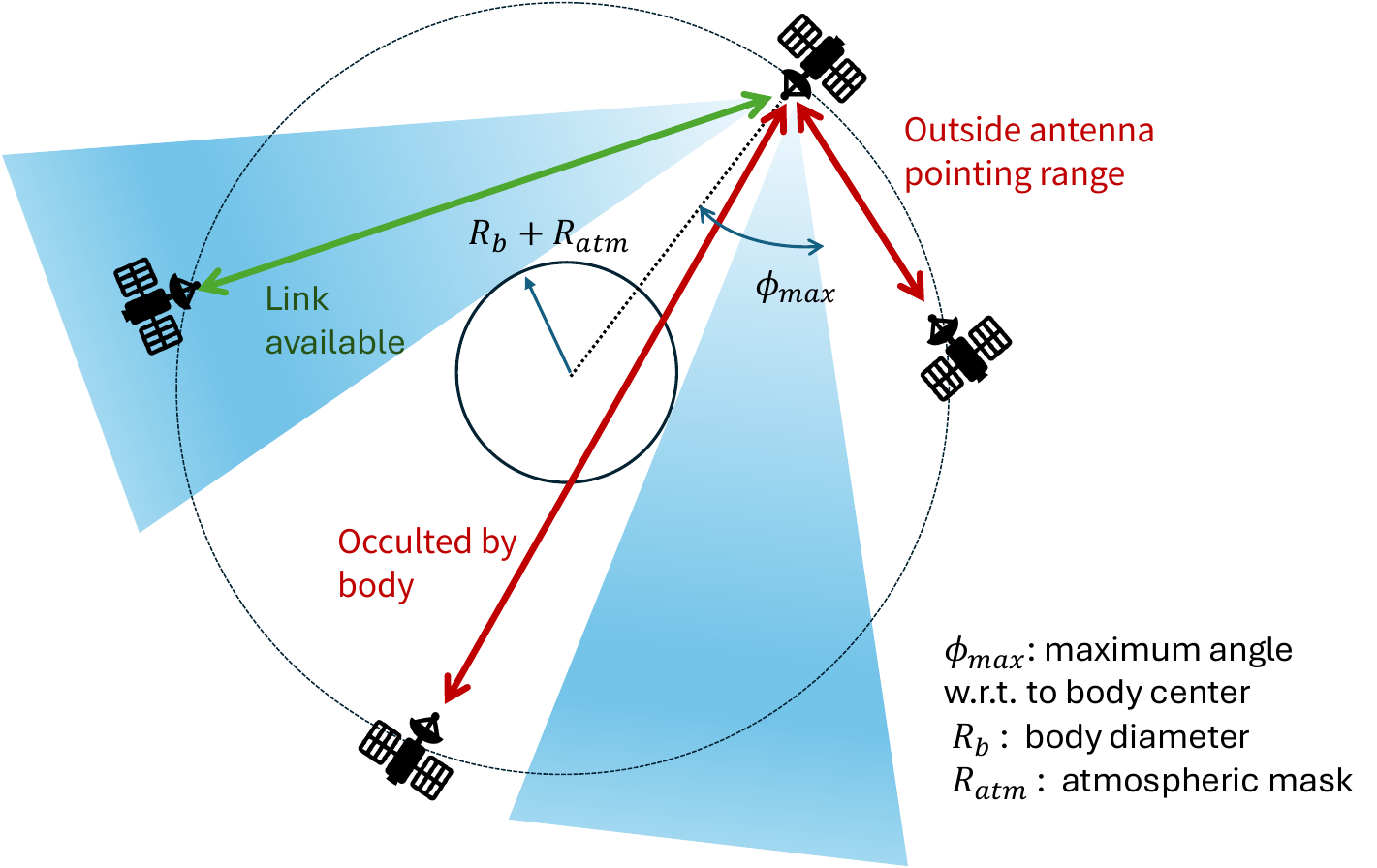}
    \newcaption{The visibility of the link between two satellites is calculated assuming that the link is visible when it is not occulted by the central body (including the atmospheric mask), 
    and the angle between the line of sight and the satellite-body center vector is less than $\phi_{max}$.}
    \label{fig:visibility}
\end{figure}

% Earth case
\begin{figure}[ht!]
    \centering
    \begin{subfigure}[b]{0.45\textwidth}
        \centering
        \includegraphics[width=\linewidth, trim={0mm, 0mm, 0mm, 10mm}]{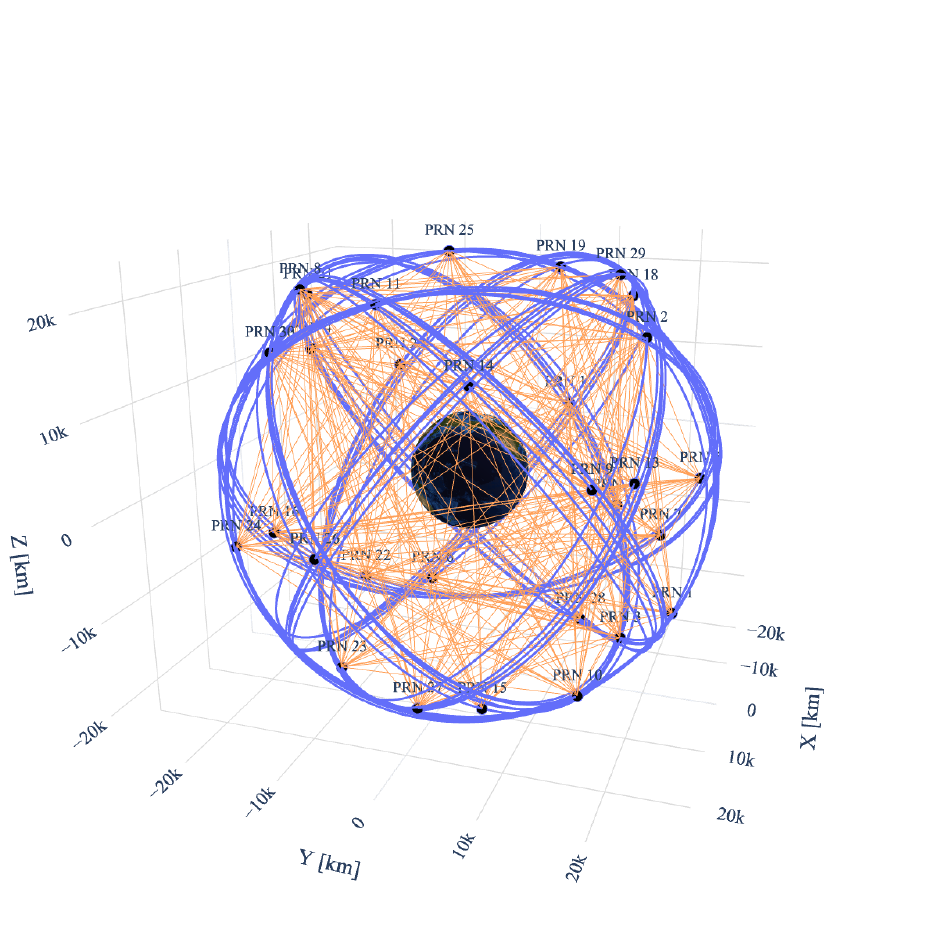}
        \subcaption{\newnewtext{Orbital planes with satellite position (black points) and links (orange line) at the initial time epoch. 
        The constellation consists of 31 satellites in six orbital planes.}} 
        \label{fig:earth_orbit}
    \end{subfigure}
    \hfill
    \begin{subfigure}[b]{0.53\textwidth}
        \centering
        \includegraphics[width=\linewidth]{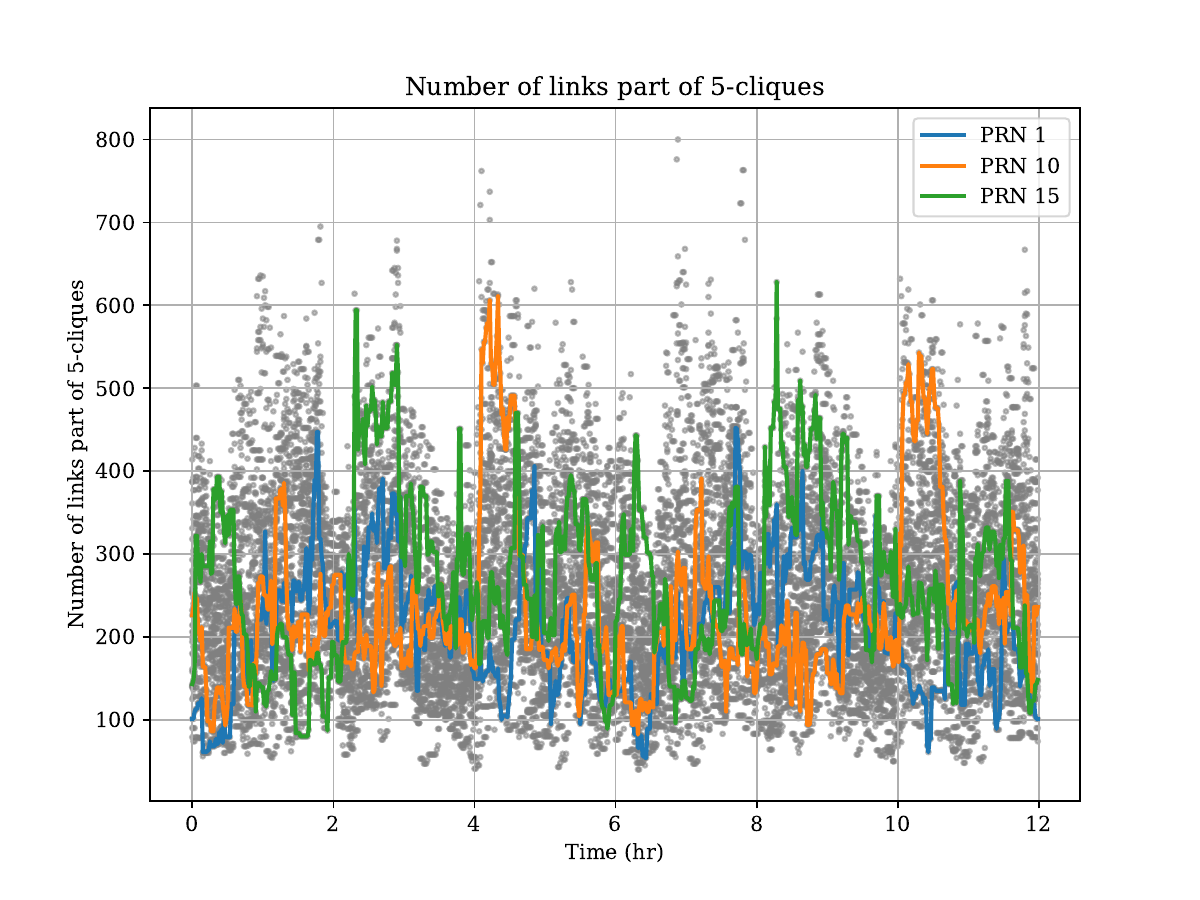}
        \subcaption{The scatter plot of the number of 5-cliques for each satellite in the constellation over full orbit. 
        The number of 5-cliques that contain PRN 1, 10, and 15 are shown in blue, orange, and green, respectively.
        Each satellite has more than 30 self-containing subgraphs throughout the orbit.}
        \label{fig:earth_nclique}
    \end{subfigure}
    \newcaption{A GPS constellation around Earth with 31 satellites. The orbit is initialized using the TLE data from CelesTrak.}
    \label{fig:orbit_earth_case}
\end{figure}

% % % % A table adhering to these guidelines is created in this example.
\begin{table}[ht!]
 \newnewcaption{Orbital elements of the Waker/ELFO hybrid lunar constellation (Case 2)}
 \label{tab:moon_constellation}
 %text alignment: l -left; c - center; r -right
\begin{tblr}{
    width=\textwidth,
    colspec={X[1,c] X[1,c] | *{6}{X[2,c]}},
    row{even} = {white,font=\small},
    row{odd} = {bg=black!10,font=\small},
    row{1} = {bg=black!20,font=\bfseries\small},
    row{2} = {bg=black!20,font=\bfseries\small},
    hline{Z} = {1pt,solid,black!60},
    rowsep=3pt
}
\SetCell[r=2]{c} Plane & \SetCell[r=2]{c} PRN  & Semi-Major Axis & Eccentricity & Inclination  & RAAN &  Argument of Periapsis & 
Mean Anomaly  \\
&  & $a$ [km] & $e$ []  & $i$ [deg] & $\Omega$ [deg] & $\omega$ [deg]  & $M$ [deg]       \\ \hline
1 & 1 - 3 & 11314.7 & 0.56 & 56.8 &  206.6 & 90 & [0, 120, 240] \\
2 & 4 - 6 & 11314.7 & 0.56 & 46.9 &  321.2 & 98.1 & [40, 160, 280] \\
3 & 7 - 9 & 11314.7 & 0.56 & 46.9 &  91.9 & 81.9 & [80, 200, 320] \\ 
\end{tblr}
\end{table}

\begin{figure}[ht!]
    \centering
    \begin{subfigure}[b]{0.5\textwidth}
        \centering
        \includegraphics[width=\linewidth, trim={0mm, 0mm, 0mm, 10mm}]{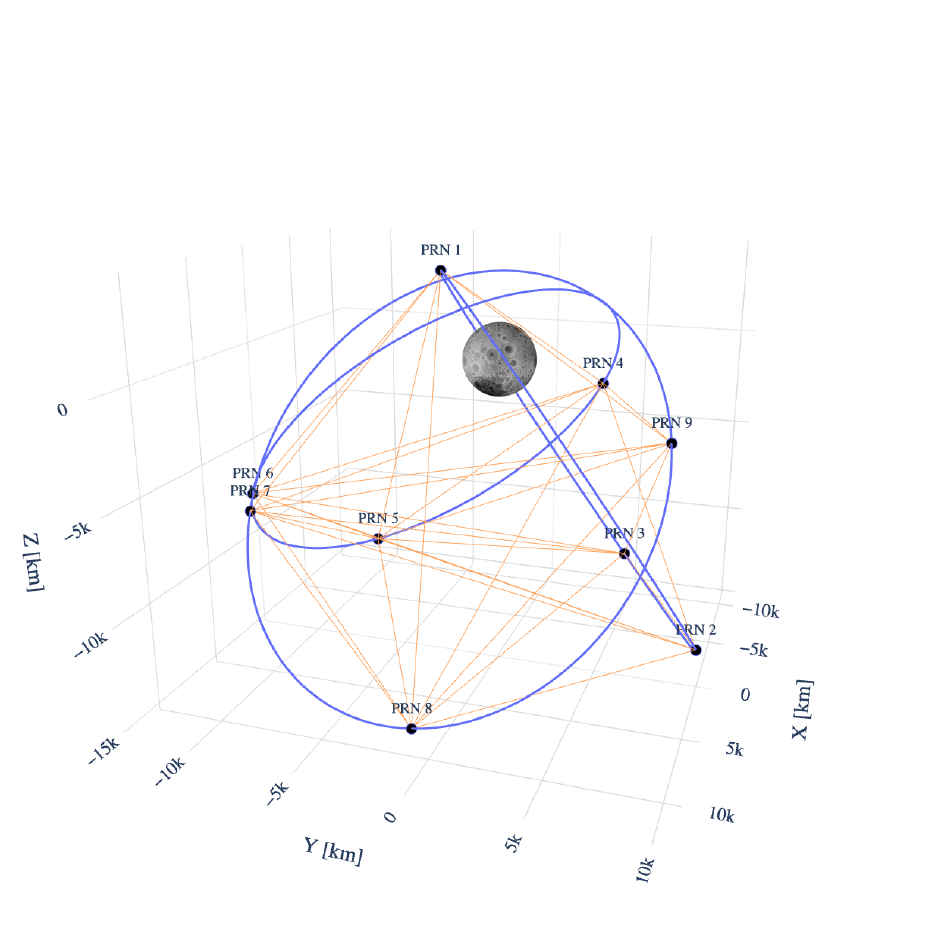}
        \subcaption{Orbital planes with satellite position (black points) and links (orange line) at the initial time epoch. 
        The constellation consists of \newnewtext{9 satellites in the Walker constellation with 3 planes}} 
        \label{fig:moon_orbit}
    \end{subfigure}
    \hfill
    \begin{subfigure}[b]{0.45\textwidth}
        \centering
        \includegraphics[width=\linewidth]{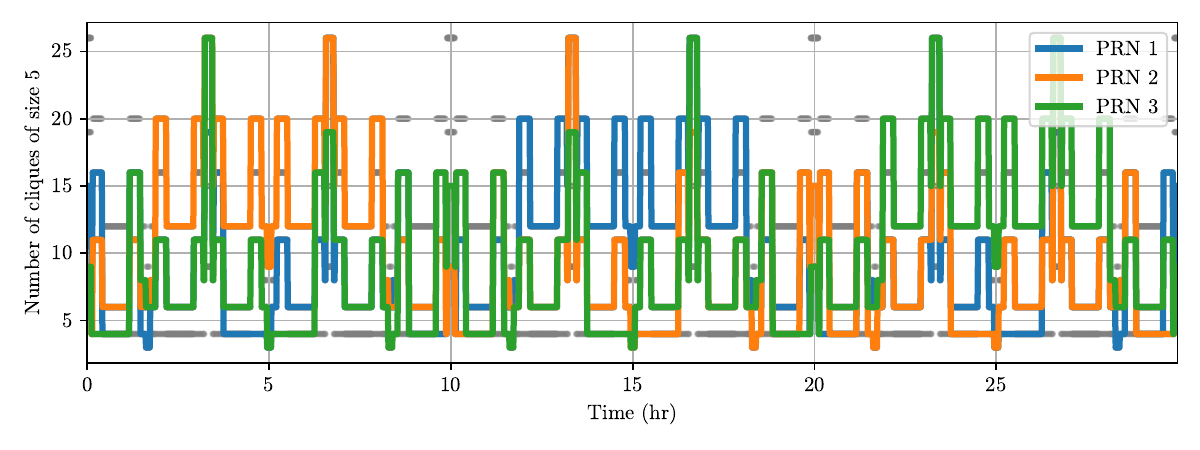}
        \subcaption{\newnewtext{The scatter plot of the number of self-containing 5-cliques for each satellite in the constellation. 
        The number of 5-cliques that contain PRN 1, 2, and 3 are shown in blue, orange, and green, respectively.}}
        \label{fig:moon_nclique}
        \vfill
        \includegraphics[width=\linewidth]{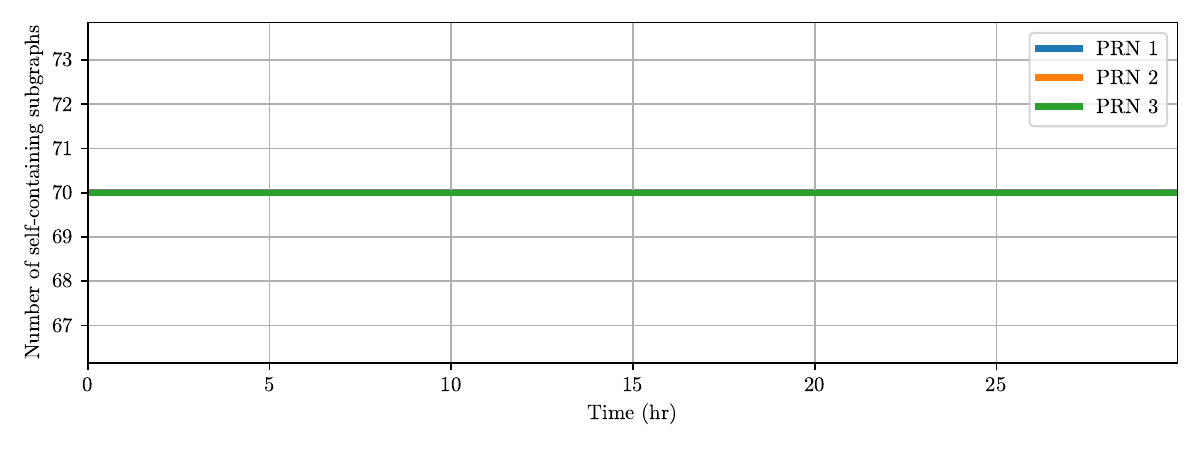}
        \subcaption{\newnewtext{The scatter plot of the number of self-containing fault-detectable 5-sat subgraphs for each satellite in the constellation. All satellites have the same number of self-containing fault-detectable subgraphs.}}
        \label{fig:moon_detectable}
    \end{subfigure}
    \newnewcaption{A lunar constellation around the Moon with 9 satellites }
    \label{fig:orbit_moon_case}
\end{figure}

\end{editsection}

%%%%%%%%%%%%%%%%%%%%%%%%%%%%%%%
% Simulation Parameters
%%%%%%%%%%%%%%%%%%%%%%%%%%%%%%
\begin{editsection}{blue}

\subsubsection{Parameters}
We tested combinations of different detection and simulation parameters: 
\begin{itemize}
\item Test methods: \newnewtext{``ephemeris comparison''}, `Data snooping', and `EDM'
\item Fault magnitudes: \newnewtext{$\bar{b}$ = 2, 4, \ldots, 20 m}
\item Ratio of fault links (the probability that the link connected to a fault satellite is biased): $r_f = 0.2, 1.0$
\item Number of fault satellites: $N_f = $ 0, 1.
\item \newnewtext{Ephemeris accuracy: $\sigma_m = \{\SI{1}{m}, \SI{2}{m}\}$ for Earth, $\sigma_m = \{\SI{2}{m}, \SI{3}{m}, \SI{4}{m} \}$ for Moon.}
\item \newnewtext{Input probability of false alarm} for the fault detection algorithms: 
    \begin{equation*}
        \alpha=0.001, 0.002, 0.003, 0.005, 0.008, 0.013, 0.022, 0.036, 0.06, 0.1
    \end{equation*}
\end{itemize}
For each of these cases, we run 5000 Monte-Carlo analyses with randomly sampled initial times (within one orbital period) and a set of fault satellites. 
For all cases, we set the standard deviation of the two-way noise measurement error to \newtext{$\sigma_m = 0.5$} m, based on the results of the Beidou constellation~\citep{Zhang2024}.
\end{editsection}
The inflation rate of $\eta_{\alpha} = 1.5$ is used for the Earth scenario and $\eta_{\alpha}=5.0$ is used for the lunar scenario.

%%%%%%%%%%%%%%%%%%%%%%%%%%%%%%%
% Evaluation Metrics
%%%%%%%%%%%%%%%%%%%%%%%%%%%%%%
\subsubsection{Evaluation Metrics}
We compute the following metrics to evaluate the performance of the fault detection algorithm
\begin{align}
    \newnewtext{\text{Probability of Missed Detection} \ P_{md}}  &= \frac{FN}{TP + FN} \\
    \newnewtext{\text{Probability of False Alarm} \ P_{fa}} &= \frac{FP}{FP + TN}
\end{align}
where $TP, FN, FP, TN$ are the total number of true positives (fault satellite classified as fault), false negatives (fault satellite classified as non-fault), false positives (non-fault satellite classified as fault), and true negatives (non-fault satellite classified as non-fault). 
A \newnewtext{lower $P_{md}$} is desired not to miss fault satellites, and a lower \newnewtext{$P_{fa}$} is desired to prevent false alarms, detecting normal satellites as a fault.
% $P_4$ value is the harmonic mean of four key conditional probabilities (precision, recall, specificity, and negative predictive value), where $P_4=1$ requires all of these values to be 1~\citep{Sitarz_2023}. $P_4$ value could be used as a metric to evaluate the ``balance" between these metrics.

%%%%%%%%%%%%%%%%%%%%%%%%%%%%%%%%%
%  Earth case
%%%%%%%%%%%%%%%%%%%%%%%%%%%%%%%%%
\begin{newnewsection}{red}

\subsection{Fault Detection Results}

\subsubsection{Case 1: GPS Constellation}
\label{sec:earth_sim}
Figure~\ref{fig:earth_fd} summarizes the fault detection performance for the GPS constellation under different parameter combinations.

\textit{a) No fault}

When no satellites are faulty, the false alarm probability $P_{fa}$ of all methods remains below the user-specified threshold $\alpha = 0.001$--$0.1$. Among the three approaches, the \newnewtext{data snooping} method and the EDM-based method achieve the lowest $P_{fa}$, indicating strong robustness in fault-free conditions.

\textit{b) Single faulty satellite with fault ratio = 1}

When the fault ratio is 1, meaning that all inter-satellite ranges (ISRs) involving the faulty satellite are corrupted, the data snooping method yields the best performance. Owing to the optimality of its statistical test under this assumption, it achieves nearly perfect fault detection ($P_{fa} = P_{md} = 0$) for fault magnitudes as small as 2~m, as shown in Figure~\ref{fig:earth_fd1_Pfa_Pmd}.  
The $P_{fa}$--$P_{md}$ trade-off curve of the EDM test lies between those of data snooping and the ephemeris comparison method, as illustrated in the same figure.

The performance of the \newnewtext{ephemeris comparison} method improves as the ephemeris accuracy increases, as shown in Figure~\ref{fig:earth_fd1_1.0_0.001}. For $\alpha = 0.001$ and ephemeris accuracy $\sigma_m = 1$~m, perfect fault detection ($P_{fa} = P_{md} = 0$) is achieved when the fault magnitude exceeds 5~m. In contrast, with lower ephemeris accuracy ($\sigma_m = 2$~m), a larger fault magnitude (approximately 8~m) is required to drive both $P_{fa}$ and $P_{md}$ to zero.

\textit{c) Single faulty satellite with fault ratio = 0.2}

When the fault ratio is reduced to 0.2, corresponding to a scenario in which only a subset of links is biased (e.g., due to link scheduling), the assumptions underlying the data snooping test are violated. Consequently, data snooping is no longer optimal. Nevertheless, for small fault magnitudes, it still provides the most favorable $P_{fa}$--$P_{md}$ trade-off among the three methods.

For larger fault magnitudes (e.g., $\geq 15$~m), the ephemeris comparison method outperforms the other two approaches. This behavior arises because, when corrupted edges are sparse, both the EDM and data snooping methods exhibit a higher probability of misattributing the fault to non-faulty satellites. As in the previous cases, the EDM method consistently demonstrates intermediate performance between ephemeris comparison and data snooping.

\end{newnewsection}

%%%%%%%%%%%%%%%%%%%%%%%%%%%%%%%%%%
% Earth case
%%%%%%%%%%%%%%%%%%%%%%%%%%%%%%%%%%
\begin{figure}[p!]
    \captionsetup[subfigure]{justification=centering}
    \centering
    \begin{subfigure}[b]{\textwidth}
        \centering
        \includegraphics[width=0.6\linewidth]{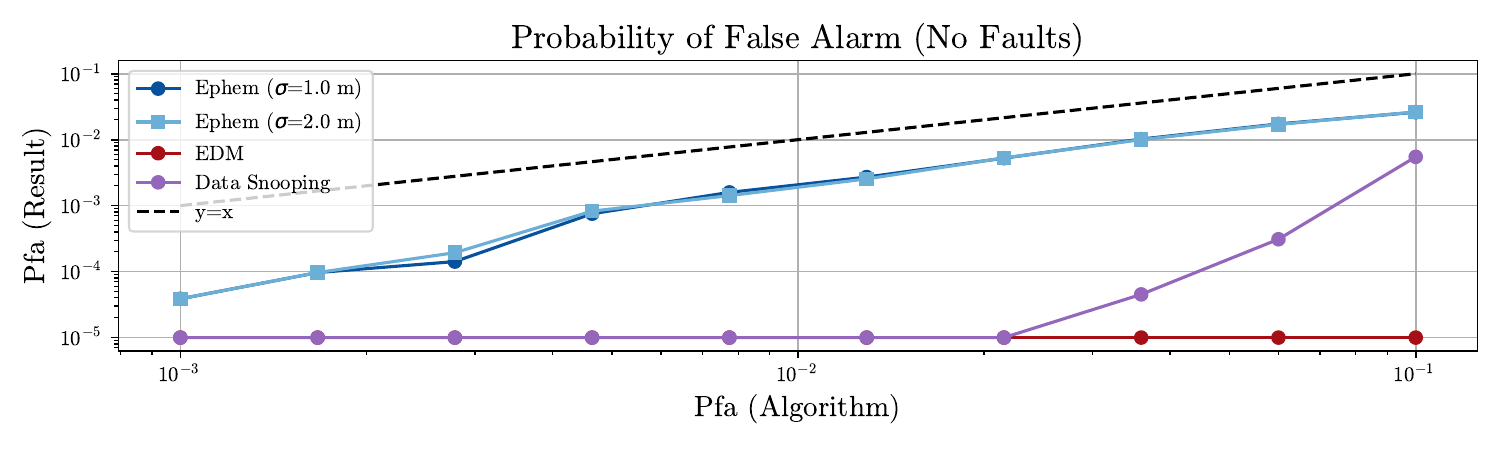}
        \subcaption{\newnewtext{[Earth] Probability of false alarm when there are no faults.}}
        \label{fig:earth_fd0}
    \end{subfigure}
    \begin{subfigure}[b]{0.33\textwidth}
        \centering
        \includegraphics[width=0.9\linewidth]{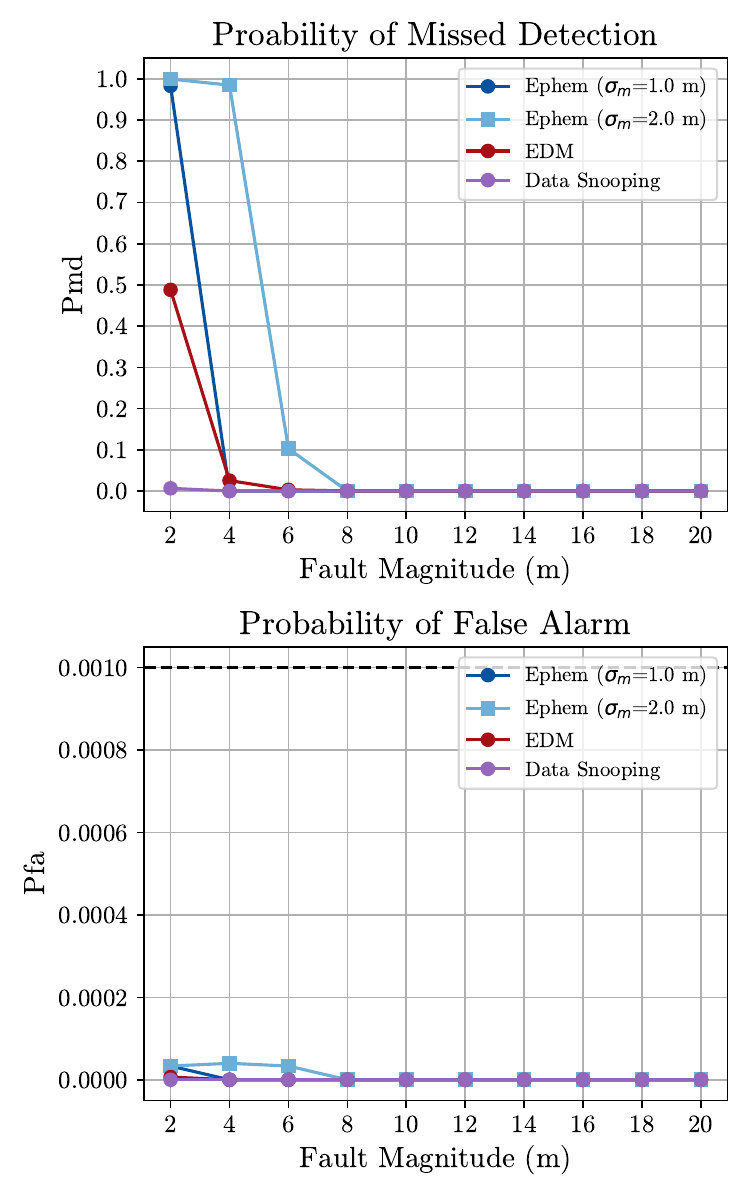}
        \subcaption{\newnewtext{[Earth] 1 fault, $\alpha$: 0.001, fault ratio: 1}}
        \label{fig:earth_fd1_1.0_0.001}
    \end{subfigure}
    \begin{subfigure}[b]{0.33\textwidth}
        \centering
        \includegraphics[width=0.9\linewidth]{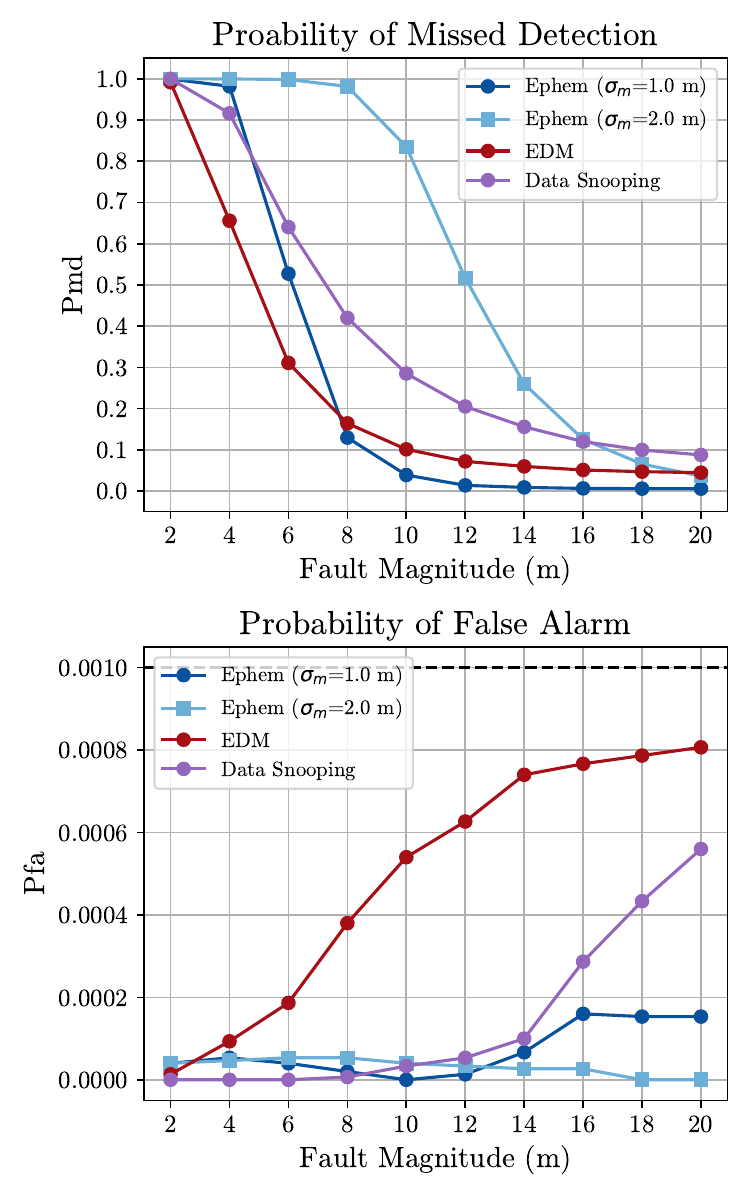}
        \subcaption{\newnewtext{[Earth] 1 fault, $\alpha$: 0.001, fault ratio: 0.2}}
        \label{fig:earth_fd1_0.2_0.001}   
    \end{subfigure}
    \begin{subfigure}[b]{0.33\textwidth}
        \centering
        \includegraphics[width=0.9\linewidth]{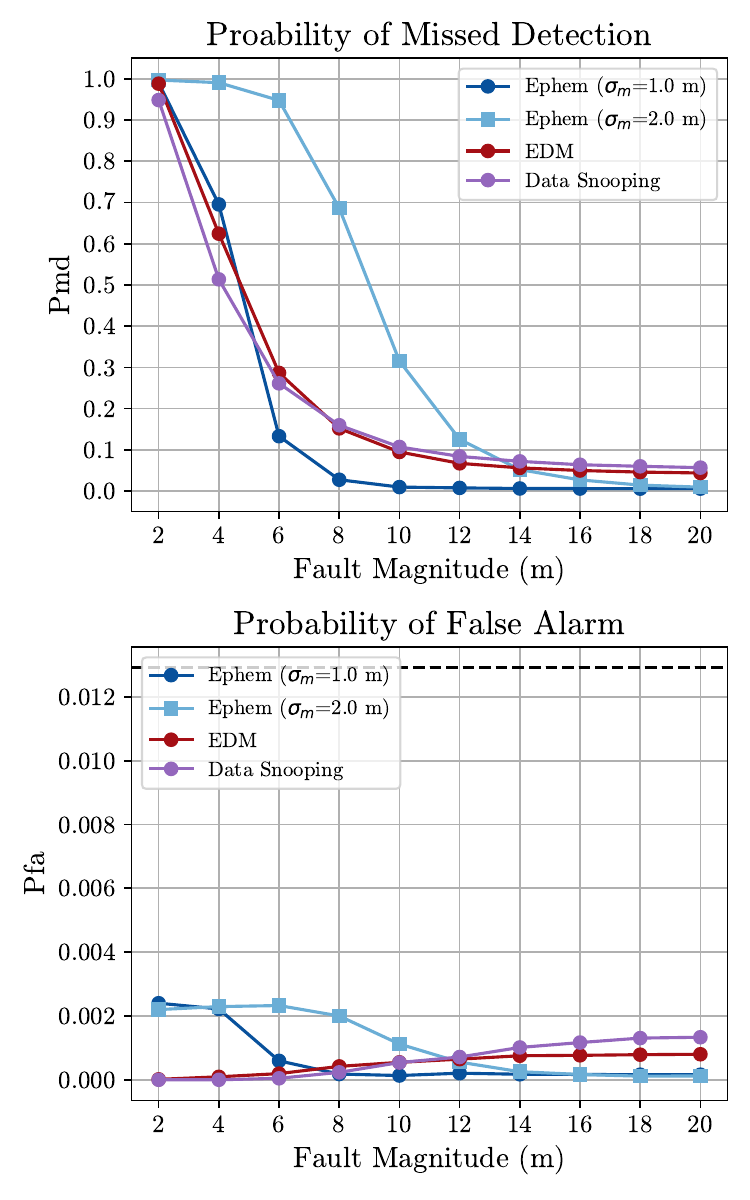}
        \subcaption{\newnewtext{Earth] 1 fault, $\alpha$: 0.013, fault ratio: 0.2}}
        \label{fig:earth_fd1_0.2_0.01}
    \end{subfigure}
    \begin{subfigure}[b]{\textwidth}
        \centering
        \includegraphics[width=0.95\linewidth]{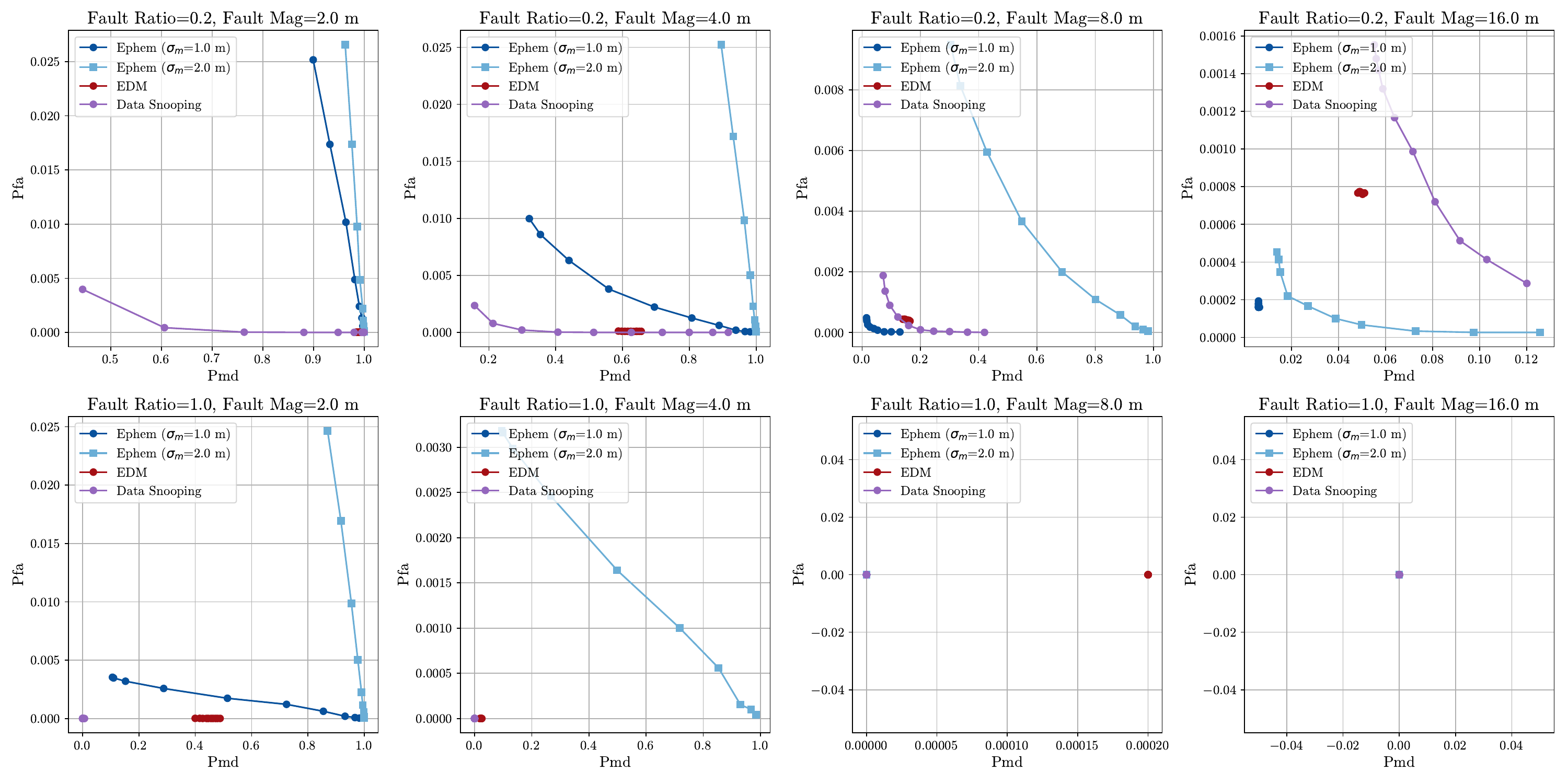}
        \subcaption{\newnewtext{[Earth] $P_{fa}$ vs $P_{md}$ for different fault magnitudes (from left to right: 2m, 4m, 8m, 16m) and input $P_{fa}$ (varied inside each subplot)}}
        \label{fig:earth_fd1_Pfa_Pmd}
    \end{subfigure}
    
    \newcaption{\newnewtext{The fault detection results for the Earth case with 0 or 1 fault satellite. 
        When there is no fault, the $P_{fa}$ of all methods went below the user-defined $P_{fa}$ $\alpha$ = 0.001 - 0.01.
        When the fault ratio is 1.0, the data snooping showed the lowest $P_{md}$ and $P_{fa}$, followed by the EDM approach and ephemeris approach.
    }}
    \label{fig:earth_fd}
\end{figure}

%%%%%%%%%%%%%%%%%%%%%%%%%%%%%%%%%%
% Moon case
%%%%%%%%%%%%%%%%%%%%%%%%%%%%%%%%%%
\begin{figure}[p!]
    \captionsetup[subfigure]{justification=centering}
    \centering
    \begin{subfigure}[b]{\textwidth}
        \centering
        \includegraphics[width=0.6\linewidth]{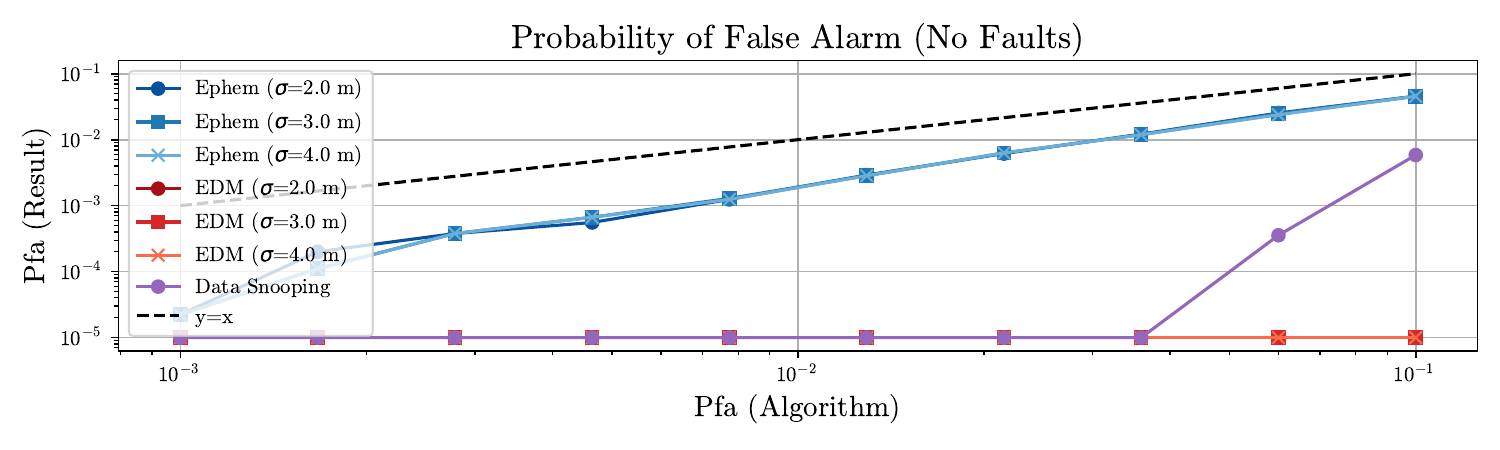}
        \subcaption{\newnewtext{[Moon] Probability of false alarm when there are no faults.}}
        \label{fig:moon_fd0}
    \end{subfigure}
    \begin{subfigure}[b]{0.33\textwidth}
        \centering
        \includegraphics[width=0.9\linewidth]{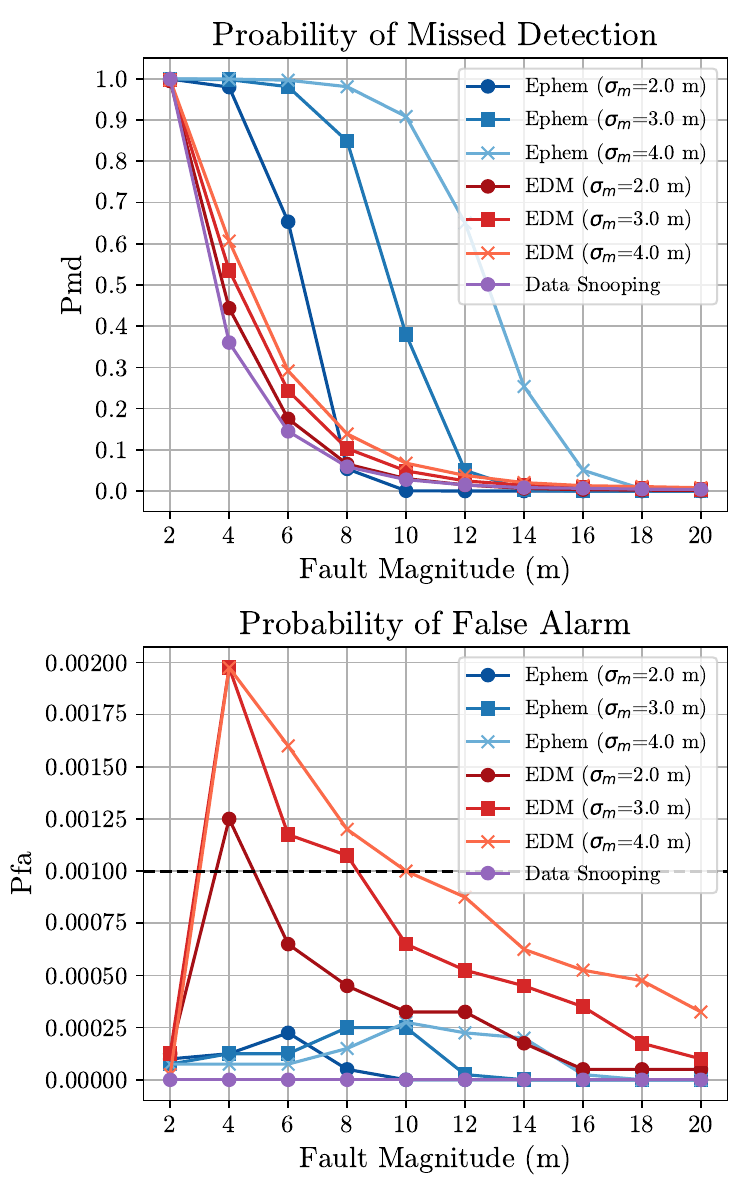}
        \subcaption{\newnewtext{[Moon] 1 fault, $\alpha$: 0.001, fault ratio: 1}}
        \label{fig:moon_fd1_1.0_0.001}
    \end{subfigure}
    \begin{subfigure}[b]{0.33\textwidth}
        \centering
        \includegraphics[width=0.9\linewidth]{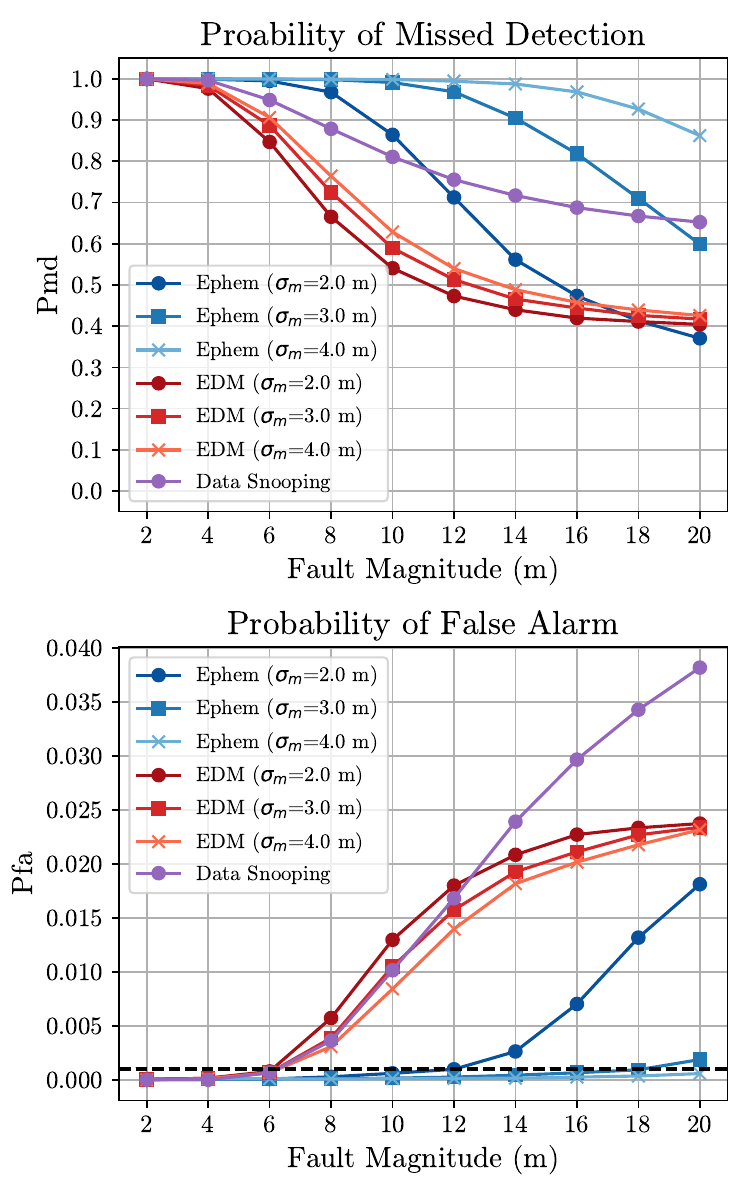}
        \subcaption{\newnewtext{[Moon] 1 fault, $\alpha$: 0.001, fault ratio: 0.2}}
        \label{fig:moon_fd1_0.2_0.001}   
    \end{subfigure}
    \begin{subfigure}[b]{0.33\textwidth}
        \centering
        \includegraphics[width=0.9\linewidth]{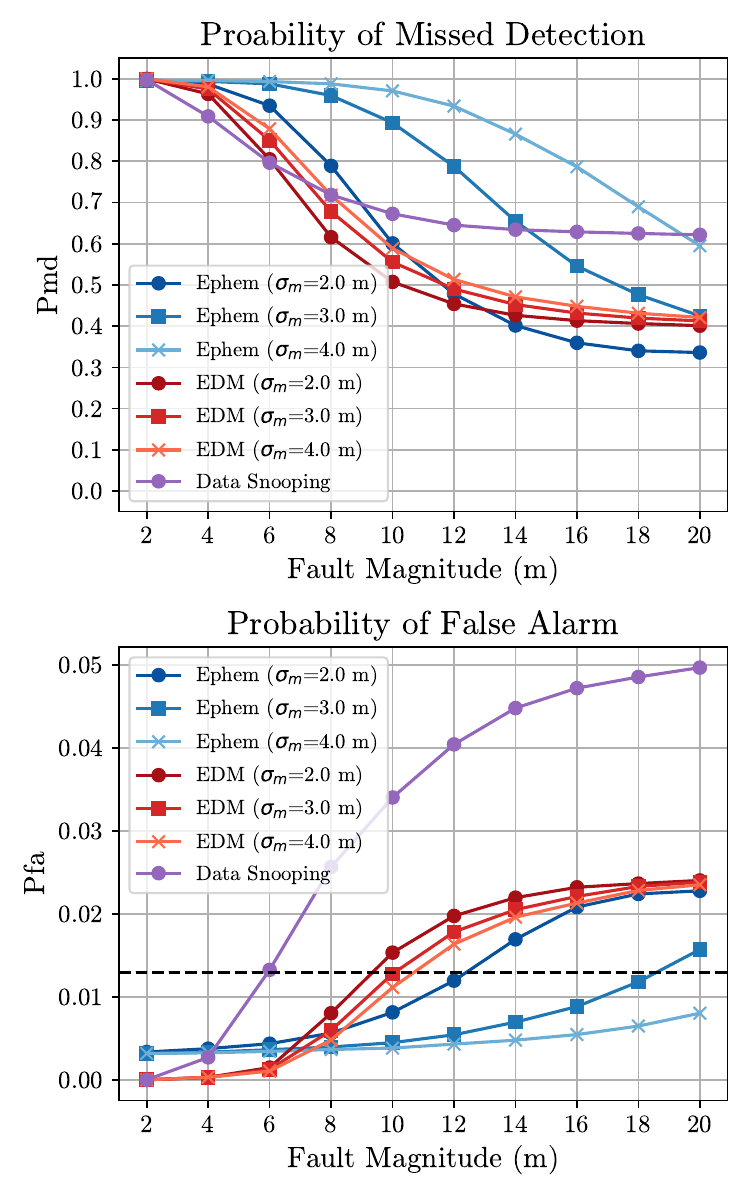}
        \subcaption{\newnewtext{[Moon] 1 fault, $\alpha$: 0.013, fault ratio: 0.2}}
        \label{fig:moon_fd1_0.2_0.01}
    \end{subfigure}
    \begin{subfigure}[b]{\textwidth}
        \centering
        \includegraphics[width=0.95\linewidth]{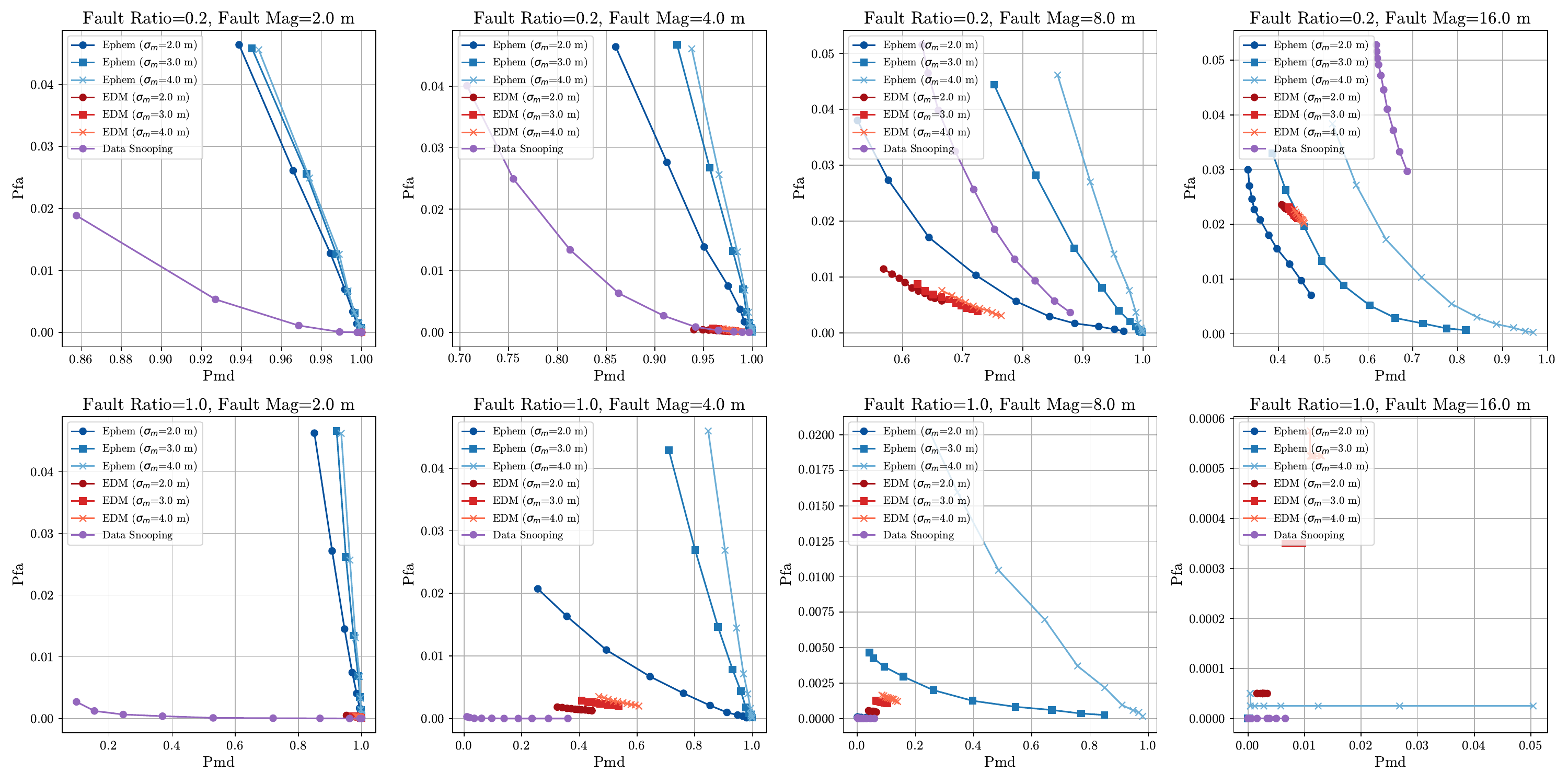}
        \subcaption{\newnewtext{[Moon] $P_{fa}$ vs $P_{md}$ for different fault magnitudes (from left to right: 2m, 4m, 8m, 16m) and input $P_{fa}$ (varied inside each subplot)}}
        \label{fig:moon_fd1_Pfa_Pmd}
    \end{subfigure}
    \newcaption{\newnewtext{The fault detection results for the lunar case with 0 or 1 fault satellite. 
        When there is no fault, the $P_{fa}$ of all methods went below the user-defined $P_{fa}$ ($\alpha$ = 0.001 - 0.1).
        For fault ratio of 1.0, the data snooping showed the lowest $P_{md}$ and $P_{fa}$, followed by the EDM approach and then the ephemeris approach.
        When the fault ratio is 0.2, the EDM and Ephemeris approach provides better $P_{fa}$ and $P_{md}$ for large fault magnitudes.
    }}
    \label{fig:moon_fd}
\end{figure}

\begin{newnewsection}{red}
\subsubsection{Case 2: Lunar Constellation}
\label{sec:moon_sim}

Figure~\ref{fig:moon_fd} summarizes the fault detection performance for the lunar constellation under different parameter combinations. Since the \newnewtext{ephemeris augmentation} strategy described in Section~\ref{sec:detectable_subgraphs} is employed for the EDM approach, its detection performance also becomes explicitly dependent on the assumed ephemeris accuracy.

\textit{a) No fault}

Consistent with the GPS case, when no satellites are faulty, the false alarm probability $P_{fa}$ of all methods remains below the user-specified threshold $\alpha = 0.001$--$0.1$. Among the three approaches, the \newnewtext{data snooping} method and the EDM-based method achieve the lowest $P_{fa}$, indicating robust behavior under fault-free lunar operating conditions.

\textit{b) Single faulty satellite with fault ratio = 1}

When the fault ratio is 1.0, the data snooping method again yields the best overall performance, owing to the optimality of its statistical test under the assumption that all links associated with the faulty satellite are corrupted. However, due to the reduced number and sparsity of inter-satellite links in the lunar constellation compared to the GPS case, a larger fault magnitude (approximately $\geq 10$~m) is required for data snooping to achieve nearly perfect fault detection ($P_{fa} \approx P_{md} \approx 0$), as shown in Figures~\ref{fig:moon_fd1_1.0_0.001} and~\ref{fig:moon_fd1_Pfa_Pmd}.

As in the terrestrial case, the $P_{fa}$--$P_{md}$ trade-off curve of the EDM approach lies between those of data snooping and the ephemeris comparison method. The performance of the \newnewtext{ephemeris comparison} method improves with increasing ephemeris accuracy, although its sensitivity to ephemeris errors is more pronounced in the lunar case due to weaker geometric redundancy.

\textit{c) Single faulty satellite with fault ratio = 0.2}

When the fault ratio is reduced to 0.2, corresponding to scenarios in which only a subset of inter-satellite links is corrupted, the assumptions underlying the data snooping test are violated. As a result, the performance of data snooping degrades significantly, with elevated $P_{md}$ observed even for moderate fault magnitudes.

In contrast, both the EDM and ephemeris comparison methods exhibit improved robustness in this sparse-fault regime. For larger fault magnitudes (e.g., $\geq 8$--10~m), these two approaches achieve lower $P_{md}$ while maintaining $P_{fa}$ near or below the prescribed threshold. The EDM method consistently demonstrates intermediate performance, outperforming data snooping for sparse faults while remaining less sensitive to ephemeris inaccuracies than the ephemeris comparison method when high-accuracy ephemerides are not available.
\end{newnewsection}

\section{Conclusion}
\begin{newnewsection}{red}  
\label{sec:conclusion}
This paper presented a novel framework for detecting satellite clock phase jumps using inter-satellite range measurements, grounded in rigid graph theory and Euclidean distance matrix analysis.
The proposed method models the constellation as a weighted graph and evaluates the realizability of vertex-redundantly rigid subgraphs to identify clock jumps that manifest as structured biases in the inter-satellite ranges.

We established the necessary graph-topological conditions for fault detectability and derived analytical expressions for the distribution of the singular values of the geometric-centered Euclidean distance matrix (GCEDM) under both nominal and faulty conditions.
These results enabled the development of statistically rigorous hypothesis tests with controllable false-alarm probabilities and analytically characterized detection performance.
To address sparse measurement geometries, we introduced an ephemeris-augmented formulation that increases fault detectability when inter-satellite connectivity is limited, at the cost of increased sensitivity to ephemeris uncertainty.
Simulation studies using both a terrestrial GPS constellation and a notional lunar constellation demonstrated that the proposed method remains effective across a wide range of fault magnitudes, link topologies, and operational scenarios.

Overall, the proposed approach offers a geometry-driven perspective on range-based fault detection through the lens of graph rigidity, complementing conventional residual-based methods.
Future work will investigate extensions that incorporate modifications to data-snooping methods, multi-fault detection, and integration with link scheduling to maximize fault observability and detection performance.

\end{newnewsection}

\section*{acknowledgements}
We acknowledge Derek Knowles for reviewing the paper.
This material is based upon work supported by The Nakajima Foundation and the National Science Foundation under Grant No. DGE-1656518. 

\nocite{*}
\printbibliography[title=References]

\appendix

\begin{editsection}{blue}

\section{Computing the Test Threshold Using Imhof's Method}
\label{sec:imhof}
In this section, we describe the procedure to obtain the critical value for the hypothesis test in section 3.1.
The correlation between the two residuals that share the same satellites can be described as~\citep{xu2011gnss}
\begin{equation}
\begin{aligned}
        \rho(g_{i, j}, g_{i, k}) &= \frac{\text{cov}(g_{i, j}, g_{i, k})}{\sqrt{\text{cov}(g_{i, j}, g_{i, j})\ \cdot \text{cov}(g_{i, k}, g_{i, k})}}
        &= \frac{\beta \sigma_r^2}{2 \sigma_r^2 + \sigma_m^2}
\end{aligned}
\end{equation}
where $\beta$ is the cosine of the angle between two line-of-sight vectors $e_{i, j}, e_{i, k}$, that satisfies 
\begin{equation}
   \cos(\phi_{max}) < \cos(\beta) < \cos(0) = 1
\end{equation}
The covariance matrix $\Sigma$ of the $l_i$ normalized residuals $\tilde{g}_{i, j}  = g_{i, j} / \sqrt{2 \sigma_r^2 + \sigma_m^2}$ obtained at satellite $i$ can be described as 
\begin{equation}
    \Sigma_{i, j} = \begin{cases}
        1   &  i = j \\
        \frac{\beta \sigma_r^2}{2 \sigma_r^2 + \sigma_m^2} & i \neq j 
    \end{cases}
\end{equation}
Using the Imhof's method, the cumulative distribution function (CDF) value of the test statistic $T_i = \sum_{j=1}^{l_i} \tilde{g}_{ij}^2$ is given as~\citep{imhof1961computing}
\begin{equation}
    F_{\Sigma}(q) = P (T_i \leq q) = \frac{1}{2} - \frac{1}{\pi}\ \int_{0}^{\infty} \Im \Bigl(e^{-itq} \cdot \phi(t) \Bigr) \frac{dt}{t}
    \label{eq:cdf_imhof}
\end{equation}
where $\Im (\cdot)$ denotes the imaginary part, and the characteristic function $\phi(t)$ is defined as
\begin{equation}
    \phi(t) = -\prod_{k=1}^{l_i} \sqrt{1 - 2 i t \lambda_k}
\end{equation}
where $\lambda_k (k=1, \ldots, l_i)$ are the eigenvalues of the covariance matrix $\Sigma$.
However, it is computationally expensive to evaluate equation \eqref{eq:cdf_imhof} at every satellite at every timestep.
Therefore, \citet{xu2011gnss} suggests to use the conservative bound for the cross-terms of $\Sigma_{i, j}$, as follows
\begin{equation}
    \tilde{\Sigma}_{i, j} = \begin{cases}
        1   &  i = j \\
        \frac{\sigma_r^2}{2 \sigma_r^2 + \sigma_m^2} & i \neq j 
    \end{cases}
\end{equation} 
Then, we can pre-compute $F_{\tilde{\Sigma}}(q)$ for $l_i=1, \ldots, N_s$, and a pre-determined false alarm rate $\alpha = 1-q$, and store them in a table.
The test procedure using this conservative bound is
\begin{equation}
    \text{Accept} \ \mathcal{H}_0 \ \text{if} \ \text{max}_{j \in \{1, \ldots, n\}} \ T_i < F_{\tilde{\Sigma}}(1 - \alpha)
\end{equation}
otherwise
\begin{equation}
    \text{Accept} \ \mathcal{H}_k \ \text{where} \ k = \text{argmax}_{j \in \{1, \ldots, n\}} \ \frac{T_j}{\sqrt{l_j}}
\end{equation}

\end{editsection}
\begin{editsection}{blue}
\section{Proof of Proposition 4.1 }
\label{sec:app21}

\begin{proof}
Let $r_{uv}$ be the observed inter-note distances.
Consider a $k$-vertex redundantly rigid graph $G$ with weight $W(e_{uv}) = r_{uv}$ that is $d$-embeddable.
Let $\mathbf{q} \colon V \in \R^3$ be a generic realization (i.e, it does not lie on a plane or line, unless forced by the ranges) of the points with true inter-node distances  $\bar{r}_{uv}$ given by
\begin{equation}
    \bar{r}_{uv} = \| \mathbf{q}_u - \mathbf{q}_v \|
\end{equation}
and $p \colon V \in \R^3$ be the generic realization of the points with the observed inter-note distances $r_{uv}$
\begin{equation}
    r_{uv} = \| \mathbf{p}_u - \mathbf{p}_v \| = \| \mathbf{q}_u - \mathbf{q}_v \| + f_{uv}
\end{equation}
where $f_{uv} \neq 0$ if $u$ or $v$ is the faulty satellite, and $f_{uv}=0$ otherwise.

Assume the following statement:
\begin{equation}
    G \ \text{contains} \ m \ (1 \leq m \leq k-1) \ \text{faulty nodes}
    \label{eq:assumption_fault}
\end{equation}
Without the loss of generality, we can assume the indices of the faulty nodes are $1, \ldots, m$ for $m \leq k-1$.
Let $G^{\prime} = G \ \backslash \ \{1, \ldots, m \}$ a graph without the $m \leq k-1$ faulty nodes. 
Since $G$ is $k$ vertex redundantly rigid, $G^{\prime}$ is rigid in $\R^d$. 
In addition, since $G^{\prime}$ contains no fault nodes, we have
\begin{equation}
    \bar{r}_{uv} = r_{uv}, \quad u, v \neq \{1, \ldots, m\}
\end{equation}
Therefore, we may choose coordinates so that 
\begin{equation}
    \mathbf{p}_j = \mathbf{q}_j \quad j \neq \{1, \ldots, m\}
\end{equation}
For each fault node $i \in \{1, \ldots, m\}$, its neighbor $j \in \mathcal{N}(i), \ \mathcal{N}(i) = \{ k | (i, k) \in G^{\prime} \}$ satisfies
\begin{equation}
    \| \mathbf{p}_u - \mathbf{p}_v \|= \| \mathbf{q}_i - \mathbf{q}_j \| + f_{ij} \quad (f_{ij} \neq 0)
\end{equation}
Hence, $\mathbf{p}_i$ must lie in the intersection 
\begin{equation}
    \bigcap_{j\in \mathcal{N}(i)} \Bigl\{\,x\in\R^3: \|x - q_j\| = \|q_i - q_j\| + f_{ij} \Bigr\}.
\end{equation}
Geometrically, this is the common intersection of $\|\mathcal{N}(i)\|$ spheres in $\R^{3}$.
It is known that every vertex in a $k$-vertex redundantly rigid graph must have degree at least $d + k - 1$~\citep{Jordan2022extremal}. 
Since $m \leq k-1$, each fault node has at maximum $k-2$ edges with other fault nodes (which are not included in $G^{\prime}$), and therefore we have $\|\mathcal{N}(i)\| \geq (d + k - 1) - (k - 2) = 3 - 1 + 2 = 4$.

Let $\{j_1, \ldots, j_4 \} \in \mathcal{N}(i)$, and 
\begin{equation}
    \mathbf{r} = \begin{bmatrix} 
    \|q_i - q_{j_1}\| + f_{i{j_1}} & \cdots & \|q_i - q_{j_4}\| + f_{i{j_4}}
    \end{bmatrix}^{\top} \in \R^4
\end{equation}
Define a smooth map
\begin{equation}
F:\R^3 \ \longrightarrow \ \R^4 \quad 
F(x) = \begin{bmatrix}
    \|x-q_{i_1}\| &
    \cdots        &
    \|x - q_{i_4}\|
\end{bmatrix}
\end{equation}
The spheres intersect if and only if
\begin{equation}
    \mathbf{r} \in  F(\R^3)
    \label{eq:r_in_F}
\end{equation}
Since the Jacobian of $F$ can have rank at most 3, its image $F(\R^3) \subset \R^4$ is at most a 3-dimensional immersed submanifold of $\R^4$. Any smooth embedding of a $k$-dimensional manifold into $\R^n$ has $n$-dimensional Lebesgue measure 0 whenever $k < n$~\citep{GuilleminPollack1974}, and since the law of random radii vector $\mathbf{r}$ is assumed continuous with respect to $4$-dimensional Lebesgue measure, the probability that equation~\eqref{eq:r_in_F} is satisfied is 0.
This contradicts the statement that the graph $G$ is $d$-embeddable with the observed ranges, at probability 1. 
Since we derived a falsehood, the assumption~\eqref{eq:assumption_fault} that the graph contains $ 1 \leq m \leq k-1$ fault nodes is false, at probability 1.
Therefore, if the graph is $k$-vertex rigid and is $d$-embeddable, and the graph contains at most $k-1$ faulty nodes, the graph does not contain any fault nodes, at probability 1.
\end{proof}

\end{editsection}
\section{Proof of Proposition 4.3 }
\label{sec:app1}

\begin{proof} 
     % Proof sketch:
     % \begin{itemize}
     %     \item EDM usually rank 1 + rank 3 + rank 1 matrices (cite the paper Keidai showed)
     %     \item For rank 1 matrices to overlap in column space ... I think we would need one satellite at the center of all the others (can verify)
     %     \item The fault matrix is rank 2 (shown with rank decomposition)
     %     \item Overlapping fault case
     %     \item Bounding above by $n-1$ (shown with a determinant of zero)
     % \end{itemize}
    As discussed in \citet{dokmanic2015euclidean}, the noiseless- and faultess-EDM, $\EDM^{n, d, 0}_{ij} = \| \singlepoint_i - \singlepoint_j\|^2$, based on a collection of points $\PointMat \in \R^{d \times n}$ can be constructed with Equation~\eqref{eq:edmnoisless}.
    \begin{equation}
        \EDM^{n, d, 0} = \onesvec \, \diag(\PointMat^\top \PointMat)^\top - 2 \PointMat^\top \PointMat + \diag(\PointMat^\top \PointMat) \onesvec^\top \label{eq:edmnoisless}
    \end{equation}
    where $\diag(\PointMat^\top \PointMat) \in \R^n$ is the vector formed from the diagonal entries of $\PointMat^\top \PointMat$.
    The first and last matrices of Equation~\eqref{eq:edmnoisless} are rank 1 matrices by construction. 
    These rank 1 matrix each contribute to the matrix rank except for degenerate cases where the points are equally spread apart ($\PointMat \onesvec = \zeromat$). 
    The middle term includes the Gram matrix $\PointMat^\top \PointMat$, which is rank $d$ except for degenerate cases where the points lie on a lower dimensional hyperplane (i.e., $\matrank(\PointMat^\top \PointMat) = 2$ if the points lie on a plane).
    So, in line with \cite{dokmanic2015euclidean}, $\matrank(\EDM^{n, d, 0}) \leq d + 2$, where the condition holds with equality outside of the aforementioned degenerate cases. 

    During a fault, we encounter cross terms with the bias following Equation~\eqref{eq:edmwithfault}.
    \begin{equation}
        \begin{aligned}
            \EDM^{n, d, m}_{ij} &= (\| \singlepoint_i - \singlepoint_j\| + \fault_{ij})^2 =  \| \singlepoint_i - \singlepoint_j\|^2 + 2\| \singlepoint_i - \singlepoint_j\| \fault_{ij} + \fault_{ij}^2 \\
            %%%%%%%%%%%%%%%%%%%%%%%
            &= \| \singlepoint_i - \singlepoint_j\|^2 + \fault_{ij} \left( 2\| \singlepoint_i - \singlepoint_j\| + \fault_{ij} \right) \\
            %%%%%%%%%%%%%%%%%%%%%%%
            &= \| \singlepoint_i - \singlepoint_j\|^2 + \fault_{ij} \left(\sqrtterm_{ij} + \fault_{ij} \right)
            \label{eq:edmwithfault}
        \end{aligned}
    \end{equation}
    Where $\sqrtterm_{ij}$ is two times the elementwise square root of the corresponding entry in the EDM. 
    As a matrix, we can write $\SqrtMatSuper_{ij} = s_{ij}$.
    % If $m = 1$, then $b_{ij}$ is only active for one satellite $k$. 
    We can write the bias term $\FaultMatSuper_{ij} = b_{ij}$ as a matrix as well.
    Both $\SqrtMatSuper$ and $\FaultMatSuper$ are real, symmetric matrices.  
    Then, we arrive at Equation~\eqref{eq:faultaddedmatrix}.
    \begin{equation}
        \EDM^{n, d, m} = \EDM^{n, d, 0} + \FaultMatSuper \hadamard (\SqrtMatSuper + \FaultMatSuper)
        \label{eq:faultaddedmatrix}
    \end{equation}
    where $\hadamard$ is the element-wise or Hadamard product. 
    For the sake of illustration, consider the case with three faults \newtext{of equal magnitude and opposite signs} realized at satellites $k=\{2, 3, 4\}$ for $n > 7$, without loss of generality.
    \begin{equation}
        \FaultMat^{n, \{2, 3, 4\}} = \faultmag \begin{bmatrix}
            0 & 1 & 1 & 1 & 0 & \cdots & 0 \\
            1 & 0 & 2 & 2 & 1 & \cdots & 1 \\
            1 & 2 & 0 & 2 & 1 & \cdots & 1 \\
            1 & 2 & 2 & 0 & 1 & \cdots & 1 \\
            0 & 1 & 1 & 1 & 0 & \cdots & 0 \\
            \vdots & \vdots & \vdots & \vdots & \vdots & \ddots & \vdots \\
            0 & 1 & 1 & 1 & 0 & \cdots & 0
        \end{bmatrix} = \faultmag \begin{bmatrix}
            0 & 1 & 1 & 1 \\
            1 & 0 & 2 & 2 \\
            1 & 2 & 0 & 2 \\
            1 & 2 & 2 & 0 \\
            0 & 1 & 1 & 1 \\
            \vdots & \vdots & \vdots & \vdots \\
            0 & 1 & 1 & 1\\
        \end{bmatrix} \begin{bmatrix}
            1 & 0 & 0 & 0 & 1 & \cdots & 1 \\
            0 & 1 & 0 & 0 & 0 & \cdots & 0 \\
            0 & 0 & 1 & 0 & 0 & \cdots & 0 \\
            0 & 0 & 0 & 1 & 0 & \cdots & 0 \\
        \end{bmatrix}
        \label{eq:faultmatThree}
    \end{equation}
    In the rank-decomposition of $\FaultMatSuper$, the column matrix has the first column has a one in each fault entry and a zero otherwise.
    This column spans all the $n - m$ fault-free columns. 
    The remaining $m$ columns correspond to the columns in $\FaultMatSuper$ with faults and have a zero at the fault entry. 
    \newtext{In the example above, the entry $2$ is illustrative. In general, the entry will depend on the fault magnitudes, which may be different.
    With probability zero, the faults across satellites exactly cancel out and lead to degeneracy in the fault detection, as discussed below.} 
    $\FaultMatSuper$ achieves full rank when the number of faults reaches $m = n-1$.  
    In general, $\matrank(\FaultMatSuper) = \min(1 + m, n)$ if $m > 0$ and $\matrank(\FaultMat^{n, 0}) = 0$.
    By similar argument, $\matrank(\FaultMatSuper \hadamard \FaultMatSuper) = \matrank(\FaultMatSuper)$ since we will have one column indicating the fault entries and $m$ columns associated with the columns of $\FaultMatSuper \hadamard \FaultMatSuper$ with faults, which are the same columns as $\FaultMatSuper$.
    This overall rank relationship often holds even if the faults are of different non-zero magnitudes.
    However, if the faults cancel out, the rank will drop.
    For example, in the case above, if $b_2 = 2 \faultmag$ and $b_3 = b_4 = - \faultmag$, then $\matrank(\FaultMatSuper) = 3$.
    Therefore, $\matrank(\FaultMatSuper) \leq \min(1 + m, n)$ for arbitrary fault sizes.
    Nevertheless, the term $\matrank(\FaultMatSuper \hadamard \FaultMatSuper)$ can still be rank $m + 1$ even if $\matrank(\FaultMatSuper) < m + 1$ since the squaring can remove the linear cancellation.
    
    However, the span of the columns of $\FaultMatSuper \hadamard \SqrtMatSuper$ will generally span the columns $\FaultMatSuper \hadamard \FaultMatSuper$. 
    Again, for the sake of illustration, we extend the example above with faults at satellites $k=\{2, 3, 4\}$ for $n > 7$, without loss of generality.
    \begin{equation}
    \begin{aligned}
        \FaultMat^{n, \{2, 3, 4\}} \hadamard \SqrtMatSuper &= \faultmag \begin{bmatrix}
            0      &   s_{12} &   s_{13} &   s_{14} & 0      & \cdots & 0 \\
            s_{12} & 0        & 2 s_{23} & 2 s_{24} & s_{25} & \cdots & s_{2n} \\
            s_{13} & 2 s_{23} & 0        & 2 s_{34} & s_{35} & \cdots & s_{3n} \\
            s_{14} & 2 s_{24} & 2 s_{34} & 0        & s_{45} & \cdots & s_{4n} \\
            0      &   s_{25} &   s_{35} &   s_{45} & 0      & \cdots & 0 \\
            \vdots & \vdots & \vdots & \vdots & \vdots & \ddots & \vdots \\
            0      &   s_{2n} &   s_{3n} &   s_{4n} & 0 & \cdots & 0
        \end{bmatrix} \\
        &= \faultmag \begin{bmatrix}
            0      &   s_{12} &   s_{13} &   s_{14} & 0      & 0 \\
            s_{12} & 0        & 2 s_{23} & 2 s_{24} & s_{25} & s_{26} \\
            s_{13} & 2 s_{23} & 0        & 2 s_{34} & s_{35} & s_{36} \\
            s_{14} & 2 s_{24} & 2 s_{34} & 0        & s_{45} & s_{46} \\
            0      &   s_{25} &   s_{35} &   s_{45} & 0      & 0 \\
            \vdots &   \vdots &   \vdots &   \vdots & \vdots & \vdots \\
            0      &   s_{2n} &   s_{3n} &   s_{4n} & 0      & 0 \\
        \end{bmatrix} \begin{bmatrix}
            1 & 0 & 0 & 0 & 0 & 0 & \rrefleftover_{17} & \cdots & \rrefleftover_{1n} \\
            0 & 1 & 0 & 0 & 0 & 0 & 0 & \cdots & 0 \\
            0 & 0 & 1 & 0 & 0 & 0 & 0 & \cdots & 0 \\
            0 & 0 & 0 & 1 & 0 & 0 & 0 & \cdots & 0 \\
            0 & 0 & 0 & 0 & 1 & 0 & \rrefleftover_{57} & \cdots & \rrefleftover_{5n} \\
            0 & 0 & 0 & 0 & 0 & 1 & \rrefleftover_{67} & \cdots & \rrefleftover_{6n} \\
        \end{bmatrix}
    \end{aligned}
        \label{eq:faultmatThreeLeftOver}
    \end{equation}
    where $\rrefleftover_{ij}$ are leftover terms from the row-reduction. 
    In the rank-decomposition of $\FaultMatSuper \hadamard \SqrtMatSuper$, $m$ columns correspond to the columns in $\FaultMatSuper$ with faults and have a zero at the fault entry, just as before.
    However, now, one column is generally not enough to span the fault-free columns since the entries in $\SqrtMatSuper$ are not linearly related.
    For example, $(s_{12}, s_{13}, s_{14})$ is generally not a linear scaling of $(s_{25}, s_{35}, s_{45})$.
    These fault-free columns constitute an $m$ dimensional subspace, from which will need $m$ columns to fully span.
    $\FaultMatSuper \hadamard \SqrtMatSuper$ achieves full rank when the number of faults reaches $m = \lceil n/2 \rceil$, where $\lceil \cdot \rceil$ is the ceiling function.
    At that point, the subspace of fault-free columns is large enough that the fault columns will contribute to the span.
    Therefore, $\matrank(\FaultMatSuper \hadamard \SqrtMatSuper) \leq \min(2m, n)$ where we do not have equality when $\SqrtMatSuper$ is degenerate (i.e., if the fault satellites are above or below a plane of fault-free satellites). 

    The columns needed to span $\FaultMatSuper \hadamard \SqrtMatSuper$ are the same as those needed for $\FaultMatSuper \hadamard \FaultMatSuper$, with the same sparsity pattern.
    Therefore, we will have
    \begin{equation}
        \matrank(\FaultMatSuper \hadamard (\SqrtMatSuper + \FaultMatSuper)) = \matrank(\FaultMatSuper \hadamard \SqrtMatSuper)  \leq \min(2m, n)
    \end{equation}
    Lastly, the subspace spanned with the fault contributions is distinct from the subspace spanned by the EDM.
    Therefore, 
    \begin{equation}
        \matrank(\EDM^{n, d, m}) \leq \min(d + 2 + 2m, n)
    \end{equation}
    where the equality holds in non-degenerate cases.

\end{proof}

\section{Proof of Proposition 4.4}
\label{sec:app2}

\begin{proof}
    First, by rank inequality $\matrank(AB) \leq \min(\matrank(A), \matrank(B))$ for arbitrary matrices $A$ and $B$. 
    The rank of the geometric centring matrix is $\matrank(\centeringJ^n) = n - 1$.
    So, $\matrank(\GCEDM^{n,d,k}) \leq n - 1$.
    However, this bound is too loose when there are few faults.
    Expanding the expression for geometric centering yields Equation~\eqref{eq:geocenteringexpanded}.
    \begin{equation}
        \begin{aligned}
            \GCEDM^{n,d,m} &= -\frac{1}{2} \centeringJ^n \EDM^{n, d, m} \centeringJ^n \\
            &= -\frac{1}{2} \centeringJ^n (\EDM^{n, d, 0} + \FaultMatSuper \hadamard (\SqrtMatSuper + \FaultMatSuper)) \centeringJ^n \\
            &= -\frac{1}{2} \centeringJ^n (\onesvec \, \diag(\PointMat^\top \PointMat)^\top - 2 \PointMat^\top \PointMat + \diag(\PointMat^\top \PointMat) \onesvec^\top) \centeringJ^n - \frac{1}{2} \centeringJ^n (\FaultMatSuper \hadamard (\SqrtMatSuper + \FaultMatSuper)) \centeringJ^n
        \end{aligned} \label{eq:geocenteringexpanded}
    \end{equation}

    First, geometric centering removes the two rank 1 matrices, as shown in Equations~\eqref{eq:geocenterterm1} and \eqref{eq:geocenterterm3} \citep{dokmanic2015euclidean}.
    \begin{align}
        -\frac{1}{2} \centeringJ^n (\onesvec \, \diag(\PointMat^\top \PointMat)^\top ) \centeringJ^n &= -\frac{1}{2} (\onesvec \, \diag(\PointMat^\top \PointMat)^\top  - \onesvec \, \diag(\PointMat^\top \PointMat)^\top ) (\eye^n - \frac{1}{n} \onesvec {\onesvec}^{\top}) = 0 \label{eq:geocenterterm1} \\
        %%%%%%%%%%%%%%
        -\frac{1}{2} \centeringJ^n (\diag(\PointMat^\top \PointMat) \onesvec^\top) \centeringJ^n &= -\frac{1}{2} (\eye^n - \frac{1}{n} \onesvec {\onesvec}^{\top}) (\diag(\PointMat^\top \PointMat) \onesvec^\top  - \diag(\PointMat^\top \PointMat) \onesvec^\top) = 0 \label{eq:geocenterterm3}
    \end{align}
    % \begin{equation}
    %     \begin{aligned}
    %         -\frac{1}{2} \centeringJ^n (\onesvec \, \diag(\PointMat^\top \PointMat)^\top ) \centeringJ^n 
    %         % &= -\frac{1}{2} (\eye^n - \frac{1}{n} \onesvec {\onesvec}^{\top}) (\onesvec \, \diag(\PointMat^\top \PointMat)^\top ) (\eye^n - \frac{1}{n} \onesvec {\onesvec}^{\top}) \\
    %         %%%%%%%%%%%%%%%%%%%%%
    %         &= -\frac{1}{2} (\onesvec \, \diag(\PointMat^\top \PointMat)^\top  - \onesvec \, \diag(\PointMat^\top \PointMat)^\top ) (\eye^n - \frac{1}{n} \onesvec {\onesvec}^{\top}) = 0 \\
    %     \end{aligned} \label{eq:geocenterterm1}
    % \end{equation}
    % \begin{equation}
    %     \begin{aligned}
    %         -\frac{1}{2} \centeringJ^n (\diag(\PointMat^\top \PointMat) \onesvec^\top) \centeringJ^n 
    %         % &= -\frac{1}{2} (\eye^n - \frac{1}{n} \onesvec {\onesvec}^{\top}) (\diag(\PointMat^\top \PointMat) \onesvec^\top) (\eye^n - \frac{1}{n} \onesvec {\onesvec}^{\top}) \\
    %         %%%%%%%%%%%%%%%%%%%%%
    %         &= -\frac{1}{2} (\eye^n - \frac{1}{n} \onesvec {\onesvec}^{\top}) (\diag(\PointMat^\top \PointMat) \onesvec^\top  - \diag(\PointMat^\top \PointMat) \onesvec^\top) = 0 \\
    %     \end{aligned} \label{eq:geocenterterm3}
    % \end{equation}
    For the $-2 \PointMat^\top \PointMat$ term, notice that the mean of the points is $\meanPoints = \frac{1}{n} \PointMat \onesvec \in \R^d$. Using this property yields Equation~\eqref{eq:geocenterterm2xTx}, which is a Gram matrix of the form $\PointMat_c^\top \PointMat_c$ for the point matrix $\PointMat_c$ centered about the origin \citep{dokmanic2015euclidean}.
    This shift will not change the rank, meaning $\matrank(\PointMat^\top \PointMat) = \matrank(\PointMat_c^\top \PointMat_c)$.
    \begin{equation}
        -\frac{1}{2} \centeringJ^n (-2 \PointMat^\top \PointMat) \centeringJ^n = (\PointMat - \meanPoints {\onesvec}^{\top})^\top (\PointMat - \meanPoints {\onesvec}^{\top}) = \PointMat_c^\top \PointMat_c \label{eq:geocenterterm2xTx}
        % \begin{aligned}
        %     -\frac{1}{2} \centeringJ^n (-2 \PointMat^\top \PointMat) \centeringJ^n &= (\eye^n - \frac{1}{n} \onesvec {\onesvec}^{\top}) (\PointMat^\top \PointMat) (\eye^n - \frac{1}{n} \onesvec {\onesvec}^{\top}) \\
        %     %%%%%%%%%%%%
        %     &= (\PointMat^\top \PointMat - \onesvec \meanPoints^\top \PointMat) (\eye^n - \frac{1}{n} \onesvec {\onesvec}^{\top}) \\
        %     %%%%%%%%%%%%
        %     &= \PointMat^\top \PointMat - \onesvec \meanPoints^\top \PointMat - \PointMat^\top \meanPoints {\onesvec}^{\top} + \onesvec \meanPoints^\top \meanPoints {\onesvec}^{\top}
        % \end{aligned} \label{eq:geocenterterm2xTx}
    \end{equation}
    The remaining term is $\frac{1}{2} \centeringJ^n (\FaultMatSuper \hadamard (\SqrtMatSuper + \FaultMatSuper)) \centeringJ^n$. 
    By rank inequality, 
    \begin{equation}
        \matrank(\centeringJ^n (\FaultMatSuper \hadamard (\SqrtMatSuper + \FaultMatSuper)) \centeringJ^n) \leq \min(\matrank(\FaultMatSuper \hadamard (\SqrtMatSuper + \FaultMatSuper)), \matrank(\centeringJ^n)) = \min(2m, n - 1)
    \end{equation}
    So, we are left with
    \begin{equation}
        \matrank(\GCEDM^{n,d,k}) \leq \min(d + 2m, n - 1) 
    \end{equation}

    % Proof sketch:
    % \begin{itemize}
    %     \item Double centering removes the rank 1 side matrices
    %     \item If the $X$ matrix is already centered, the double centering does not affect the middle matrix.
    %     \item Double centering does not remove the fault matrices
    % \end{itemize}
\end{proof}
\section{Proof of Corollary 4.4.1}
\label{sec:app3}
\begin{proof}
    In terms of rank, the noise acts as many small faults on each satellite, with the same sparsity structure as $\FaultMatSuper$ in Proposition~\ref{prop:edmrank}, almost surely with $m = n$ since there is zero probability mass that the randomly sampled noise is exactly zero or exactly cancel out.
    Using Proposition~\ref{prop:geocenteredm}, 
    \begin{equation}
        \matrank(\GCEDMnoisy^{n,d,m}) = \matrank(\GCEDM^{n,d,n}) \leq \min(d + 2n, n - 1) = n - 1
    \end{equation}
    Therefore, succinctly, $\matrank(\GCEDMnoisy^{n,d,m}) = n - 1$, almost surely.
\end{proof}

\begin{editsection}{blue}

\section{The Derivation of the Scale of Chi-squared Distribution}
\label{sec:scale_distribution}

Let $\mathbf{D} \in \R^{5 \times 5}$ the distance matrix between the nodes where $D_{ij} = \|\mathbf{x}_i - \mathbf{x}_j\|$, and $E \in \R^{5 \times 5}$ the perturbation matrix (the noise matrix) where $E_{ij} = \omega_{ij} \sim \mathcal{N}(0, \sigma_m^2)$. 
The observed noisy distance matrix is $\tilde{\mathbf{D}} = \mathbf{D} + \mathbf{E}$.
The EDM is $\mathbf{D} \circ \mathbf{D}$, where $\circ$ is the Hadamard (element-wise) product.
The observed (noisy) GCEDM is 
\begin{equation}
\begin{aligned}
    \tilde{\mathbf{G}} &= -\frac{1}{2} \centeringJ^5 (\tilde{\mathbf{D}} \circ \tilde{\mathbf{D}}) \centeringJ^5 \\
    &= -\frac{1}{2} \mathbf{J}^5 \left[(\mathbf{D} \circ \mathbf{D}) + 2 (\mathbf{D} \circ \mathbf{E}) + (\mathbf{E} \circ \mathbf{E}) \right] \mathbf{J}^5 \\
    &\approx \mathbf{G} + \mathbf{\Delta G}  \quad (\because \mathbf{E} \ll \mathbf{D})
    \label{eq:delta_G}
\end{aligned}
\end{equation}
where $\mathbf{\Delta G} = \mathbf{J^5} (\mathbf{D} \circ \mathbf{E}) \mathbf{J^5}$. 
\newnewtext{The second-order error term $(\mathbf{E} \circ \mathbf{E})$ is approximated as zero in the third line. }
Here, $\mathbf{G}$ has rank 3, and $\mathbf{\tilde{G}}$ has rank 4, as proved in Corollary 4.4.1. Therefore, $\text{rank}(\mathbf{\Delta G}) \geq 4 - 3 = 1$, by rank subadditivity.
Denote the singular value decomposition of G as 
\begin{equation}
    \mathbf{G} = \mathbf{U \Sigma V^\top} 
    = \begin{bmatrix}
    \mathbf{U_1} \ | \ \mathbf{U_0}
    \end{bmatrix}
    \begin{bmatrix}
    \mathbf{\Sigma_3}  & \mathbf{0}_{3 \times 2} \\
    \mathbf{0}_{2 \times 3} & \mathbf{0}_{2\times2}  \\
    \end{bmatrix}
    \begin{bmatrix}
    \mathbf{V}_1 \ | \  \mathbf{V}_0
    \end{bmatrix}^{\top}
    \quad (\mathbf{U_1, V_1} \in \R^{5 \times 3} \text{ and } \mathbf{U_0, V_0} \in \R^{5 \times 2})
\end{equation}
From the SVD of the perturbed matrix $\mathbf{\tilde{G} = \tilde{U} \tilde{\Sigma} \tilde{V}^\top}$, we obtain
\begin{equation}
\begin{aligned}
    \mathbf{\tilde{\Sigma}} &= \mathbf{\tilde{U}^\top \tilde{G} \tilde{V}} \\
    &= \mathbf{(U + \Delta U)^{\top} (G + \Delta G) (V + \Delta V)} \\
    &= \mathbf{\Sigma + U^{\top} (\Delta G) V  
    + U^{\top} G (\Delta V) + \Delta U^{\top} G V 
    + \left[(\Delta U)^{\top} \Delta G V + U^{\top} \Delta G \Delta V
    + (\Delta U)^{\top} \Delta G \Delta V  \right] } 
    \\ 
    &\approx \mathbf{\Sigma + U^{\top} (\Delta G) V
    + \Sigma V^{\top} (\Delta V) + \Delta U^{\top} U \Sigma} \quad (\because \text{2nd and 3rd order perturbation 
 } \approx 0) \\
    &= \begin{bmatrix}
        \mathbf{\Sigma_3 + U_1^\top \Delta G V_1 + \Delta_{11}} &
        \mathbf{U_1^\top \Delta G V_0 + \Delta_{12}} \\
        \mathbf{U_0^\top \Delta G V_1 + \Delta_{21}} &
        \mathbf{U_0^\top \Delta G V_0}
    \end{bmatrix} \quad (\mathbf{\Delta_{11}, \Delta_{12}, \Delta_{21}}: \text{ perturbation from } \mathbf{\Sigma V^{\top} (\Delta V) + \Delta U^{\top} U \Sigma)} \\
    &= \begin{bmatrix}
        \mathbf{\Sigma_3 + U_1^\top \Delta G V_1 + \Delta_{11}} &
        \mathbf{0}_{3 \times 2} \\
        \mathbf{0}_{2 \times 3} &
        \mathbf{U_0^\top \Delta G V_0}
    \end{bmatrix}
    \label{eq:sigma_pert}
\end{aligned}
\end{equation}
% Since $\mathbf{\Delta G}$ has rank 1,
$\mathbf{U_0^\top \Delta G V_0}$ has rank 1 since $\mathbf{\tilde{G}}$ has rank 4. 
Thus, the 4th singular value of $\mathbf{\tilde{\mathbf{G}}}$ satisfies
\begin{equation}
    \lambda_4^2 = \tilde{\Sigma}_{44}^2 = \| \mathbf{U_0^\top \Delta G V_0} \|_F^2
    \label{eq:sigma4_frob}
\end{equation}

% Since 
% \begin{equation}
% \begin{aligned}
%     \mathbb{E} [\lambda_4] &= \mathbb{E} \newnewtext{[} \| \mathbf{U_0^\top \Delta G V_0} \|_F \newnewtext{]}
%     = \mathbb{E}  [\| \mathbf{U_0 ^{\top} J^5 (D \circ E) J^5 V_0} \|_F]
%     = 0  \quad (\because \mathbb{E}[E_{ij}] = 0),
% \end{aligned}
% \end{equation}
% the squared 4th singular value $\lambda_4^2$ follows a scaled chi-squared distribution with degree of freedom = 1.

\begin{newnewsection}{red}
We observe that the expectation of the matrix elements inside the norm is zero:
\begin{equation}
\begin{aligned}
    \mathbb{E} [(\mathbf{U}_0^\top \Delta \mathbf{G} \mathbf{V}_0)_{ij}] 
    = \left( \mathbf{U}_0^\top \mathbf{J}^5 (\mathbf{D} \circ \mathbb{E}[\mathbf{E}]) \mathbf{J}^5 \mathbf{V}_0 \right)_{ij}
    = 0  \quad (\because \mathbb{E}[E_{ij}] = 0).
\end{aligned}
\end{equation}
Moreover, the term $(\mathbf{U}_0^\top \Delta \mathbf{G} \mathbf{V}_0)_{ij}$ is a linear combination of the elements of $\mathbf{E}$, where $E_{ij}$ follows a Gaussian distribution. Since a linear combination of Gaussian variables is also Gaussian, the term itself is normally distributed.
Combined with the zero-mean property, the squared 4th singular value $\lambda_4^2$ follows a scaled central chi-squared distribution with 1 degree of freedom ($k$), where $\lambda_4/s^2$ follows a chi-squared distribution.
\end{newnewsection}
\newnewtext{
Since the mean of a chi-squared distribution can be written as $k s^2$ (k=1), $s^2$ is given by}
\begin{equation}
\begin{aligned}
    \newnewtext{s^2} &= \newnewtext{\mathbb{E} [ \lambda_4^2 ]} \\
    &= \mathbb{E} \left[ \| \mathbf{U_0 ^{\top} J^5 (D \circ E) J^5 V_0} \|_F^2 \right] \\
    &= \mathbb{E} \left[  \|(\mathbf{\hat{U}^{\top} (D \circ E) \hat{V}} \|_F^2 \right] \quad (\hat{U} =  J^5 U_0, \hat{V} = J^5 V_0 \in \R^{5 \times 2}) \\
    &= \sum_{a=1}^2 \sum_{b=1}^2 \left( \mathbb{E} \left[ \left( \sum_{i=1}^5 \sum_{j=1}^5 \hat{U}_{i, a} D_{ij} E_{ij} \hat{V}_{j, b} \right)^2 \right] \right)
    \label{eq:scale_from_mean}
\end{aligned}
\end{equation}
The distribution of the product of the two elements in the noise matrices can be represented as
\begin{equation}
    \mathbb{E} [E_{ij} E_{kl}] = \begin{cases}
        \sigma_m^2  & \text{if} \  (i, j) = (k, l) \ \text{or} \ (i, j) = (l, k) \\
        0 & \text{otherwise}
    \end{cases}
\end{equation}
Therefore, the \newnewtext{squared scale of $\lambda_4$} is
\begin{equation}
\begin{aligned}
    \newnewtext{s^2} &= \newnewtext{\sum_{a=1}^2 \sum_{b=1}^2 \sum_{i=1}^5 \sum_{j=1}^5 \sigma_m^2 D_{ij}^2 \left( \hat{U}_{i,a} \hat{V}_{j,b} + \hat{U}_{j,a} \hat{V}_{i,b} \right)^2}
    \label{eq:scale_computed}
\end{aligned}
\end{equation}   

In addition, by comparison of the right bottom element of equation \eqref{eq:sigma_pert}, we get
\begin{equation}
\begin{aligned}
    \mathbf{\tilde{U}_0^{\top} (G + \Delta G) \tilde{V}_0} &= \mathbf{U_0^\top \Delta G V_0} 
    \label{eq:UGV_invariant}
\end{aligned}
\end{equation}
where $\tilde{U}_0, \tilde{V}_0$ are the 4th and 5th columns of $\tilde{U}, \tilde{V}$. Therefore, by plugging equation \eqref{eq:sigma4_frob} into \eqref{eq:UGV_invariant}, we obtain the same scale as equation \eqref{eq:scale_computed} by replacing $D$ and $U_0, V_0$ with the perturbed distance matrix $\tilde{D}$ and singular vectors $\tilde{U}_0, \tilde{V}_0$.
Therefore, we obtain
\begin{equation}
\begin{aligned}
    \newnewtext{s^2} &= 
    \newnewtext{ \sum_{a=1}^2 \sum_{b=1}^2 \sum_{i=1}^5 \sum_{j=1}^5 \left(\sigma_m \mathbf{\tilde{D}}_{ij}\right)^2 \cdot 
    \left( \hat{U}_{i,a} \hat{V}_{j,b} + \hat{U}_{j,a} \hat{V}_{i,b} \right)^2}  \quad (\mathbf{\hat{U} =  J^5 \tilde{U}_0, \hat{V} = J^5 \tilde{V}_0} \in \R^{5 \times 2})
    \label{eq:scale_computed_value}
\end{aligned}
\end{equation}   

\end{editsection}
\begin{editsection}{blue}

\section{The Derivation of the Non-centrality Parameter When Fault Exists}
\label{sec:lambda_distribution}

Let the fault satellite $s_f$, and all the ranges connected to the fault satellite is biased by $\bar{f}$.
Using the same notation as Appendix F, the perturbed distance matrix $\tilde{\mathbf{D_f}}$ is 
\begin{equation}
    \tilde{\mathbf{D}_f} = \mathbf{D} + \mathbf{F} + \mathbf{E}
\end{equation}
where
\begin{equation}
    \mathbf{F}_{ij} = \begin{cases}
        f_{ij} &  i = s_f \ \text{or} \ j = s_f \\
        0        &  \text{otherwise}
    \end{cases}
\end{equation}
Therefore, the perturbed GCEDM is 
\begin{equation}
\begin{aligned}
    \mathbf{\tilde{G}}^{f} &= - \frac{1}{2} \mathbf{J}^5 (\tilde{\mathbf{D}}_f \circ \tilde{\mathbf{D}}_f) \mathbf{J}^5 
    \approx \mathbf{G} + \mathbf{\Delta G^f} + \mathbf{\Delta G}
\end{aligned} 
\end{equation}
where
\begin{equation}
    \mathbf{\Delta G^f} = \mathbf{J^5} (\mathbf{F} \circ \mathbf{D}) \mathbf{J^5}
\end{equation}
From Collloary 4.4.1, $\tilde{G}^{f}$ has rank 4. 
Following similar discussion as Appendix F, the 4th singular value of $\tilde{G}^f$ satisfies
\begin{equation}
\begin{aligned}
    \lambda_4^2 &= \Sigma_{44}^2 = \| \mathbf{U_0}^T \mathbf{\Delta G  V_0} +  \mathbf{U_0}^T \mathbf{\Delta G^{f}} \mathbf{V_0} \|_F^2 
\end{aligned}
\end{equation}

% The expectation of the singular value is
% \begin{equation}
%     \begin{aligned}
%         \mathbb{E}[\lambda_4] 
%         &= \mathbb{E} [ \|\mathbf{U_0}^T \mathbf{\Delta G  V_0}  +  \mathbf{U_0}^T \mathbf{\Delta G^{f} V_0} \|_F ] \\
%         &= \mathbb{E}  [\| \mathbf{U_0}^{\top} \mathbf{J^5} (\mathbf{D} \circ \mathbf{E}) \mathbf{J^5 V_0}  + \mathbf{U_0}^{\top} \mathbf{J^5} (\mathbf{D} \circ \mathbf{F}) \mathbf{J^5 V_0} \|_F ] \\
%         &> 0 \quad \left( \because \mathbb{E}[E_{ij}] = 0, 
%         \mathbb{E}[(D \circ F)_{ij}] = \begin{cases}
%             D_{ij} f_{ij} \neq 0 & i = s_f \ \text{or} \ j = s_f \\ 
%             0 & \text{otherwise}
%         \end{cases} \right)
%     \end{aligned}
% \end{equation}
% Therefore, $\lambda_4^2$ follows the non-central chi-squared distribution. 

\begin{newnewsection}{red}
We analyze the expectation of the matrix term inside the norm:
\begin{equation}
    \begin{aligned}
        \mathbb{E} [ \mathbf{U}_0^\top (\mathbf{\Delta G} + \mathbf{\Delta G}^{f}) \mathbf{V}_0 ] 
        &= \mathbb{E} [ \mathbf{U}_0^\top \mathbf{\Delta G} \mathbf{V}_0 ] + \mathbf{U}_0^\top \mathbf{\Delta G}^{f} \mathbf{V}_0 \\
        &= \mathbf{0} + \mathbf{U}_0^\top \mathbf{J}^5 (\mathbf{D} \circ \mathbf{F}) \mathbf{J}^5 \mathbf{V}_0 \\
        &\neq \mathbf{0} \quad (\because \mathbf{D} \circ \mathbf{F} \neq \mathbf{0})
    \end{aligned}
\end{equation}
The underlying variable is the sum of a zero-mean Gaussian term ($\mathbf{U}_0^\top \mathbf{\Delta G} \mathbf{V}_0$) and a non-zero constant bias term ($\mathbf{U}_0^\top \mathbf{\Delta G}^{f} \mathbf{V}_0$).
Because the underlying variable is a Gaussian with a non-zero mean, its squared Frobenius norm $\lambda_4^2$ follows a non-central chi-squared distribution.
\end{newnewsection}

\newnewtext{The mean of this non-central chi-squared distribution is}
\begin{equation}
\begin{aligned}
        s^2(k + \lambda) = \mathbb{E}[\lambda_4^2] 
        &= \mathbb{E} \left[ \|\mathbf{U_0}^T \mathbf{\Delta G  V_0} \|_F^2 \right] + \mathbb{E} \left[\sum_{i=1}^2 \sum_{j=1}^2 (\mathbf{U_0}^T \mathbf{\Delta G  V_0})_{ij} (\mathbf{U_0}^T \mathbf{\Delta G^f  V_0})_{ij} \right] + \|\mathbf{U_0}^T \mathbf{\Delta G^f  V_0} \|_F^2  \\
        &= \newnewtext{s^2} + \|\mathbf{U_0}^T \mathbf{\Delta G^f  V_0} \|_F^2 \quad (\because \text{Eq.} \ \eqref{eq:scale_from_mean}, \ \mathbb{E} [(U_0^T \Delta G  V_0)_{ij}] = 0) \\
        &= \newnewtext{s^2} \left( 1 + \frac{1}{\newnewtext{s^2}} \|\mathbf{U_0} ^{\top} \mathbf{J} (\mathbf{D} \circ \mathbf{F}) \mathbf{J V_0} \|_F^2 \right)
\end{aligned} 
\end{equation}
\newnewtext{where $\newnewtext{s^2}$ is the squared scale given in \eqref{eq:scale_computed_value}}

Therefore, the non-centrality parameter is given by
\begin{equation}
    \lambda_{s_f} = \frac{1}{\newnewtext{s^2}} \|\mathbf{U_0} ^{\top} \mathbf{J^5} (\mathbf{D} \circ \mathbf{F}) \mathbf{J^5 V_0}\|_F^2
    \label{eq:non_central_formula}
\end{equation}

\end{editsection}
\begin{newnewsection}{red}

\section{Generalization of the Test Statistic Distribution for \texorpdfstring{$n \ge 5$}{n >= 5}}
\label{sec:generalization_test}

We generalize the derivation from Appendix F to an arbitrary number of satellites $n \ge 5$. We define the generalized test statistic $\gamma_{test}$ as the total energy of the noise subspace singular values:
\begin{equation}
    \gamma_{test} = \sum_{k=4}^{n-1} \tilde{\lambda}_k^2
\end{equation}
where $\tilde{\lambda}_k$ are the singular values of the observed (noisy) GCEDM $\tilde{G}$. This statistic is equivalent to the squared Frobenius norm of the matrix projected onto the active noise subspace.

Since the true noiseless matrix $G$ is unknown, we approximate the active noise subspace using the singular vectors of the measured matrix $\tilde{G}$. Let the SVD of the observed matrix be $\tilde{G} = \tilde{U} \tilde{\Sigma} \tilde{U}^\top$. We partition the singular vectors $\tilde{U}$ as follows:
\begin{equation}
    \tilde{U} = \begin{bmatrix}
        \underbrace{\tilde{u}_1 \quad \tilde{u}_2 \quad \tilde{u}_3}_{\tilde{U}_S} & \underbrace{\tilde{u}_4 \quad \ldots \quad \tilde{u}_{n-1}}_{\tilde{U}_{noise}} & \underbrace{\tilde{u}_n}_{\tilde{u}_{\mathbf{1}}}
    \end{bmatrix}
    \label{eq:u_tilde_vector}
\end{equation}
Here, $\tilde{U}_S$ spans the signal subspace. The vectors $\tilde{u}_4, \ldots, \tilde{u}_{n-1}$ span the active noise subspace. The final vector $\tilde{u}_n$ corresponds to the structural null direction parallel to $\mathbf{1}$. Due to geometric centering ($\tilde{G} \mathbf{1} = \mathbf{0}$), the singular value associated with $\tilde{u}_n$ remains strictly zero.

We approximate the projection of the noise perturbation $\Delta G$ using this observed basis $\tilde{U}_{noise}$. Let $M_{noise} \approx \tilde{U}_{noise}^\top \Delta G \tilde{U}_{noise}$. Since $\Delta G$ is linear in the Gaussian measurement noise $E$, $M_{noise}$ is a symmetric matrix of zero-mean Gaussian variables. The sum of squares of its independent elements follows a Chi-squared distribution with degrees of freedom $k$:
\begin{equation}
    k = \frac{(n-4)(n-3)}{2}
\end{equation}
(e.g., $k=1$ for $n=5$). The scaled test statistic follows $\frac{\gamma_{test}}{s^2} \sim \chi^2(k, 0)$, where the squared scale parameter $s^2$ is computed using the observed singular vectors $\tilde{u}_a$ (columns of $\tilde{U}_{noise}$) and the observed distances $\tilde{D}_{ij}$:
\begin{equation}
    s^2 \approx \frac{1}{k} \sum_{a=4}^{n-1} \sum_{b=4}^{n-1} \left( \sum_{i=1}^n \sum_{j=1}^n (\sigma_m \tilde{D}_{ij})^2 (\tilde{u}_{i, a} \tilde{u}_{j, b} + \tilde{u}_{j, a} \tilde{u}_{i, b})^2 \right)
    \label{eq:generalized_scale}
\end{equation}
This approximation holds because the perturbation of the singular vectors corresponding to the noise-subspace is small (Equation \eqref{eq:sigma_pert}), allowing $\tilde{U}_{noise}$ (obtained from $\mathbf{\tilde{G}}$) to serve as a valid approximation for the true noise subspace $\hat{U}_N$ (of $\mathbf{G}$) in the variance computation.
Note that for $n=5$, equation \eqref{eq:generalized_scale} reduces to the result derived in Equation    \eqref{eq:scale_computed_value}.

When a fault occurs, the GCEDM includes a deterministic bias term $\Delta G^f = J^n (D \circ F) J^n$. The projection onto the noise subspace becomes $M_{total} = M_{noise} + M_{bias}$, where $M_{bias} = \hat{U}_N^\top \Delta G^f \hat{U}_N$. The test statistic $\gamma_{test} \approx \| M_{total} \|_F^2$ then follows a non-central Chi-squared distribution:
\begin{equation}
    \frac{\gamma_{test}}{s^2} \sim \chi^2(k, \lambda_{nc})
\end{equation}
The non-centrality parameter $\lambda_{nc}$ can be computed as:
\begin{equation}
    \lambda_{nc} = \frac{1}{s^2} \| M_{bias} \|_F^2 = \frac{1}{s^2} \sum_{a=1}^{n-4} \sum_{b=1}^{n-4} (\hat{u}_a^\top \Delta G^f \hat{u}_b)^2
\end{equation}
For $n=5$, this reduces to the result $\lambda_{s_f} = \frac{1}{s^2} (u_0^\top \Delta G^f u_0)^2$ derived in Equation \eqref{eq:non_central_formula}.

\end{newnewsection}
\begin{newnewsection}{red}

\section{Proof of Proposition 4.5}
\label{sec:fault_detectable_subgraph_proof}
\begin{proof}
We assume the graph $G$ is redundantly rigid. A graph is fault disprovable if the injection of a fault bias makes the set of edge constraints inconsistent (i.e., the graph cannot be embedded in Euclidean space).

Let $v_f \in V$ be a satellite experiencing a clock fault. This fault introduces a bias $b \neq 0$ to the measured ranges. The weights of the edges incident to $v_f$ are given by:
\begin{equation}
    W(u, v_f) = 
    \begin{cases} 
        \| \mathbf{x}_u - \mathbf{x}_{v_f} \| + b & \text{if } (u, v_f) \in E_m \\
        \| \mathbf{x}_u - \mathbf{x}_{v_f} \| & \text{if } (u, v_f) \in E_e
    \end{cases}
\end{equation}
where the measurement noise and ephemeris error are not considered in this proof.

\textit{Case 1: Sufficient Condition.} Assume $v_f$ has at least one incident edge $e_{meas} = (u, v_f) \in E_m$. Since $G$ is redundantly rigid, the distance between $u$ and $v_f$ is structurally constrained by the remaining edges in $E \setminus \{e_{meas}\}$. The faulty weight $W(e_{meas})$ includes the bias $b$, creating a geometric contradiction with the constraints imposed by the rest of the graph. Thus, the graph becomes unrealizable, and the fault is detectable.

\textit{Case 2: Necessary Condition.} Assume $v_f$ has no incident edges in $E_m$ (i.e., $\text{deg}_{E_m}(v_f) = 0$). In this case, all edges connected to $v_f$ are from the set $E_e$. The weights of these edges are derived solely from the ephemeris, which is independent of the instantaneous clock jump fault. Consequently, the geometry of $v_f$ in the graph remains consistent with the ephemeris model, regardless of the actual clock fault. The graph remains realizable, and the fault cannot be detected by analyzing the rigidity or embeddability of $G$.

Therefore, for the graph to be fault-disprovable for any single satellite fault, it is necessary that every node $v \in V$ has at least one incident edge in $E_m$.
\end{proof}

\end{newnewsection}

\end{document}